\documentclass{article}

\usepackage{microtype}
\usepackage{graphicx}
\usepackage{subfigure}
\usepackage{booktabs} 
\usepackage{lipsum}

\usepackage{amsmath}
\usepackage{amssymb}
\usepackage{mathtools}
\usepackage{amsthm}
\usepackage{subfiles}
\usepackage{dsfont}
\usepackage{style2026/algorithm}
\usepackage{style2026/algorithmic}
\usepackage{bbm}
\usepackage{enumitem}

\usepackage{hyperref}

\DeclareMathOperator{\Tr}{Tr} 
\DeclareMathOperator{\arsinh}{arsinh}
\DeclareMathOperator{\Var}{Var}
\DeclareMathOperator{\sgn}{sgn}
\DeclareMathOperator{\rank}{rank}

\usepackage[accepted]{style2026/icml2026}

\newcommand{\norm}[1]{\left\lVert #1 \right\rVert}

\allowdisplaybreaks

\usepackage[capitalize,noabbrev]{cleveref}

\newtheoremstyle{assumpstyle}
  {10pt}
  {6pt}
  {\itshape}
  {}
  {\bfseries}
  {\newline\par}
  {6pt}
  {}

\theoremstyle{assumpstyle}
\newtheorem{theorem}{Theorem}[section]

\newtheorem{lemma}[theorem]{Lemma}

\newtheorem{remark}[theorem]{Remark}
\newtheorem{assumption}[theorem]{Assumption}

\usepackage[textsize=tiny]{todonotes}

\icmltitlerunning{Rank-Accuracy Trade-off for LoRA: A Gradient-Flow Analysis}

\begin{document}

\twocolumn[
\icmltitle{Rank-Accuracy Trade-off for LoRA: A Gradient-Flow Analysis}



\icmlsetsymbol{equal}{*}

\begin{icmlauthorlist}
\icmlauthor{Michael Rushka}{yyy}
\icmlauthor{Diego Klabjan}{zzz}
\end{icmlauthorlist}

\icmlaffiliation{yyy}{Department of Engineering Sciences and Applied Mathematics, Northwestern University, Evanston, USA}
\icmlaffiliation{zzz}{Department of Industrial Engineering and Management Sciences, Northwestern University, Evanston, USA}

\icmlcorrespondingauthor{Michael Rushka}{michaelrushka2028@u.northwestern.edu}

\icmlkeywords{Machine Learning, ICML}
]



\printAffiliationsAndNotice{~} 

\begin{abstract}
Previous empirical studies have shown that LoRA achieves accuracy comparable to full-parameter methods on downstream fine-tuning tasks, even for rank-1 updates. By contrast, the theoretical underpinnings of the dependence of LoRA’s accuracy on update rank remain relatively unexplored. In this work, we compare the accuracy of rank-$r$ LoRA updates against full-parameter updates for fine-tuning tasks from a dynamical systems perspective. We perform gradient flow analysis in both full-rank and low-rank regimes to establish explicit relationships between rank and accuracy for two loss functions under LoRA. While gradient flow equations for LoRA are presented in prior work, we rigorously derive their form and show that they are identical for simultaneous and sequential LoRA parameter updates. We then use the resulting dynamical system equations to obtain closed-form relationships between LoRA rank and accuracy for trace-squared and Frobenius-norm low-rank approximation loss functions.
\end{abstract}

\section{Introduction}

Parameter-efficient finetuning (PEFT) has become the norm for training deep learning models on downstream tasks. Low-Rank Adaptation (LoRA) \cite{Hu2022LoRA} in particular has gained wide popularity among such methods. Consider $f: \mathbb{R}^{n \times m} \to \mathbb{R}$ which has been optimized by parameters $W_0 \in \mathbb{R}^{n \times m}$. Keeping $W_0$ frozen, finetuning consists of finding an update $\Delta W \in \mathbb{R}^{n \times m}$ which approximates $ \min\limits_{\Delta W}f(W_0 + \Delta W)$ for smaller downstream datasets introduced post-training. While full-parameter finetuning (FPFT) methods search all of $\mathbb{R}^{n \times m}$ for optimal $\Delta W$, LoRA splits the classical gradient descent (GD) algorithm into separate updates which solve $\min_{A,B}\limits f(W_0+BA)$, where $B \in \mathbb{R}^{n \times r}$ and $A \in \mathbb{R}^{r \times m}$ for $r << \min(n,m)$. By lowering the number of finetuning parameters from $nm$ to $r(n+m)$, LoRA enables parameter efficiency crucial for models with hundreds of billions of trainable parameters.

While several empirical studies have verified the effectiveness of LoRA for finetuning models in both natural language processing \cite{Li2024SecureCodeGen, Lazauskas2025EmpiricalLoRA, HanindhitoPatelJohn2025LoRA, Mao2025SurveyLoRA} and computer vision tasks \cite{Zanella2024LowRankFA, Agiza2024MTLoRA}, the breadth of theoretical literature on LoRA remains sparse. Existing work in this area focuses on convergence \cite{Malinovsky2024RACLoRA, mu2025convergence}, expressiveness bounds \cite{zeng2024the}, and initialization \cite{Hayou}. Moreover, several such studies \cite{Malladi2023KernelView, Jang2024} are conducted entirely within the neural-network setting, relying on specific architectural assumptions which do not directly generalize to the loss functions studied in our analysis. In particular, we know of only one prior 
work~\cite{xu2025understanding} which studies the learning dynamics of LoRA using gradient flow (GF), an area we seek to build upon. While prior studies \cite{xu2025understanding, zeng2024the} have derived theoretical \emph{bounds} on the loss and approximation error attainable by LoRA for certain \emph{rank thresholds}, to our knowledge none have studied explicit, closed-form relationships between either loss or approximation error and rank parameter $r$. By \emph{approximation error}, we mean the Frobenius-norm discrepancy between the full-rank fine-tuning update and its rank-$r$ approximation. We establish here these explicit rank-accuracy tradeoffs for LoRA via a gradient-flow analysis.

\subsection{Related Work}

\subsubsection{Expressive Power of LoRA}
Previous work by Zeng and Lee~\cite{zeng2024the} shows that, in the context of Fully Connected Neural Networks (FNNs) and Transformer Networks (TFNs), there exists a low-rank update which perfectly adapts a prefrozen network to a smaller target network. Let $W_0 \in \mathbb{R}^{n \times m}$ represent a prefrozen model and $\bar{W} \in \mathbb{R}^{n \times m}$ a smaller target model. Zeng and Lee demonstrate, under mild assumptions on network architecture, that \emph{there exists} a rank-$R$ \big(where $R < \min(n,m)$\big) adapter $\Delta W$ such that $W_0 + \Delta W$ exactly represents the target model, i.e. $W_0 + \Delta W = \bar{W}$. They note that $R$ must exceed an explicit threshold which they derive. For FNNs where $\rank(\Delta W)$ is lower than the required threshold for perfect adaptation, they provide an upper bound on the resulting approximation error. However, they only demonstrate a rank-independent error bound for any $\rank(\Delta W)$ below the exact-adaptation threshold without analyzing any particular finetuning algorithm. In contrast, our analysis operates outside any specific neural network architecture and provides a rank-dependent final loss and approximation error from an optimization perspective, demonstrating the accuracy achieved by LoRA-adapted GD for any given rank $r < \min(n,m)$.

\subsubsection{LoRA Gradient Flow}
To our knowledge, only Xu et al. have previously analyzed LoRA from a GF perspective~\cite{xu2025understanding}. While their work studies the convergence of LoRA within a matrix factorization (MF) setting, they focus on how convergence \emph{rates} and \emph{behavior} depend on  \emph{initialization schema}. In particular, Xu et al. consider fine-tuning a perfectly factorized matrix $Y_{\text{pre}} = W_2W_1$ toward an updated target matrix $Y_{\text{ft}}$ via low-rank updates $B_2A_2$ and $B_1A_1$. They find that the low-rank updates converge to the target matrix with arbitrary precision, i.e.
{\small
\begin{align}
    \lim_{t \to \infty} \frac{1}{2}\norm{Y_{\text{ft}} - \Big(W_2+B_2(t)A_2(t)\Big)\Big(W_1+B_1(t)A_1(t)\Big)}^2 = 0. \nonumber
\end{align}
}

By contrast, our work focuses on the distinct classical low-rank approximation problem. In particular, we analyze how well LoRA can approximate any given matrix $W_0 \in \mathbb{R}^{n \times m}$ \emph{as a function of rank}. In this sense, our work connects more closely to foundational MF work by Eckhart, Young, and Mirsky~\cite{EckartYoung1936, mirsky1960symmetric}, providing a continuous-time extension of their theoretical findings from the perspective of LoRA optimization dynamics. 

Finally, while Xu et al. state the gradient-flow equations governing LoRA, we provide their rigorous derivation and show that these equations are invariant to whether LoRA updates are performed sequentially or simultaneously, resolving an ambiguity between theoretical formulations and common implementation practice.

\subsubsection{Other Related Theoretical Work}
Further theoretical work on LoRA focuses on convergence rates~\cite{Malinovsky2024RACLoRA, mu2025convergence}, initialization~\cite{Hayou}, and learning-rate selection~\cite{HayouLearningRate}. Related work on MF~\cite{ward2023convergence} studies convergence behavior of factorized optimization problems but does not consider low-rank approximation. Other studies propose spectral initialization methods for LoRA~\cite{Balazy2024LoRAXS, Meng2024, wang-etal-2025-milora, Lin2025denselora} which we adopt but do not novelly contribute to in our analysis.

\subsection{Contributions}
We build on prior theoretical studies of LoRA by applying GF to analyze the loss and approximation error attained by LoRA as a \emph{function of rank}. Our contributions are:

\begin{enumerate}

\item While the ODEs governing LoRA are presented in previous work \cite{xu2025understanding}, we provide the first (to our knowledge) rigorous derivation of these GF dynamics as the continuous-time limit of deterministic gradient descent with fixed step size under a low-rank parameterization. In particular, for updates of the form $BA$, we show that the resulting GF equations are identical under simultaneous updates of $A$ and $B$ (using data from the previous iterate) and sequential updates (with $A$ updated after $B$ using data from the current iterate).

\item We use the resulting GF equations to analyze the learning dynamics for the trace-squared objective under LoRA:
\begin{align}
    \underset{\substack{B \in \mathbb{R}^{n \times r} \\ A \in \mathbb{R}^{r \times n}}}{\min} \, \frac{1}{2}\Tr^2(W_0-BA). \label{eq tr^2 intro}
\end{align}
where $r < \min(n,m)$. The trace-squared loss acts as a smooth spectral regularizer closely related to nuclear-norm minimization, arising in applications such as kernel alignment \cite{cristianini2006kerneltarget} and covariance shrinkage \cite{lancewicki2019kernel}. We show that LoRA GF dynamics admit a closed-form solution to \eqref{eq tr^2 intro} for near-arbitrary initial conditions. We solve for this solution and compute the resulting asymptotic loss and approximation error relative to the full-rank optimum as explicit functions of $r$.

\item We use GF to analyze the learning dynamics of the low-rank approximation problem under LoRA, namely
\begin{align}
    \underset{\substack{B \in \mathbb{R}^{n \times r} \\ A \in \mathbb{R}^{r \times m}}}{\min} \, \frac{1}{2}\norm{W_0-BA}^2. \label{eq SSE intro}
\end{align}
where $r < \min(n,m)$. Low-rank approximation is ubiquitous in applied mathematics and ML, with applications in areas such as recommender systems~\cite{8844978} and image processing~\cite{10.1145/3360488}. We show that LoRA GF dynamics for \eqref{eq SSE intro} admit a closed-form solution under a spectral initialization scheme described in previous works~\cite{Meng2024,Balazy2024LoRAXS,xu2025understanding, wang-etal-2025-milora, Lin2025denselora}. Using spectral initialization, we show that the nonzero singular values of the rank-$r$ finetuning matrix $BA$ converge to the top $r$ singular values of $W_0$, demonstrating that LoRA achieves the optimal low-rank solution characterized by the Eckart–Young–Mirsky (EYM) theorem \cite{EckartYoung1936}.
    
\end{enumerate}

\subsection{Notation}
Throughout this paper, capital letters denote matrices, and lowercase letters denote scalars. Boldface letters denote vectors. For a matrix $A$, we write $A^T$ for its transpose and $a_{ij}$ for its $(i,j)$-th entry. We use $\norm{A}$ to denote the Frobenius norm of $A$ and $\norm{\mathbf{a}}_2$ to denote the Euclidean norm of a vector $\mathbf{a}$. We use $\lfloor \cdot \rfloor$ to denote the integer floor function. We denote $\Tr^2(W) := \Big(\Tr(W)\Big)^2$.

For a time-dependent function $f(t)$ we write \\ $f'(t) := \frac{d}{dt} f(t)$ for its time derivative. Uppercase $A(t)$ refers to a matrix-valued function of time. When emphasizing time dependence, we write $a_{ij}(t)$ for the $(i,j)$-th entry of $A(t)$, with $a_{ij}(0)$ denoting its value at initialization.

We refer to the set of all natural numbers using $\mathbb{N}$ and to the set of non-negative real numbers using $\mathbb{R}_0^+$.

Let $B \in \mathbb{R}^{n \times r}$. For a scalar-valued function
$g(B,A)$ of \emph{two matrices}, we define the partial gradient
$\nabla_B g(B,A) := \frac{\partial g}{\partial B}$ to be the matrix in
$\mathbb{R}^{n \times r}$ whose $(i,j)$-th entry is
$\frac{\partial g}{\partial b_{ij}}$.

Similarly, let $W \in \mathbb{R}^{n \times m}$. For a scalar-valued
function $f(W)$ of a \emph{single matrix}, we define the gradient
${\nabla f(W) := \frac{\partial f}{\partial W}}$ to be the matrix in
$\mathbb{R}^{n \times m}$ whose $(i,j)$-th entry is
$\frac{\partial f}{\partial w_{ij}}$.

The $n \times n$ identity matrix is denoted by $I_n$. The $n \times m$ matrix of all zeroes is denoted by $\mathbf{0}_{n \times m}$.

\section{Results}

\subsection{LoRA Gradient Flow Equations} \label{sec LoRA GF main body}
Denote the Cartesian product space $\Theta := \mathbb{R}^{n \times r} \times \mathbb{R}^{r \times m}$. Consider a function $g: \Theta \to \mathbb{R}$ which we seek to minimize over $\Theta$:
\begin{align}
    \underset{\substack{B \in \mathbb{R}^{n \times r} \\ A \in \mathbb{R}^{r \times m}}}{\min} \, g\left( B , A \right) \label{eq general optimization main body}
\end{align}
In the context of LoRA, we will have ${g(B,A) = f(W_0 + BA)}$ for $f: \mathbb{R}^{n \times m} \to \mathbb{R}$, where $W_0 \in \mathbb{R}^{n \times m}$ gives the weights from a pre-trained task. We solve \eqref{eq general optimization main body} via LoRA-adapted GD given in Algorithm~\ref{alg: lora-gradient-descent} of Appendix~\ref{app LoRA GF Proof}. Algorithm~\ref{alg: lora-gradient-descent} applies alternating updates to $A$ and $B$ and is agnostic to how these updates are performed. For $\lambda = 1$, the updates of $A$ and $B$ are performed simultaneously, with both parameters evaluated using data from the previous iteration. In contrast, $\lambda = 0$ corresponds to sequential updates, where $B$ is first updated using its prior iterate, after which $A$ is updated using the newly computed value of $B$ from the current iteration. For $0 < \lambda < 1$, the update of $A$ uses a convex combination of the values of $B$ from the current and previous iterations. Algorithm~\ref{alg: lora-gradient-descent} further generalizes the update scheme by allowing $A$ and $B$ to be updated a total of $k$ times within each iteration. While $\lambda = 1$ is most efficient in practice for parallelism, we allow arbitrary $\lambda \in [0,1]$ and $k \in \mathbb{N}$ to show that the resulting GF equations for LoRA are invariant to the choice of update scheme.

The update equations at any given iteration $i$ of Algorithm~\ref{alg: lora-gradient-descent} are given by
{\footnotesize
\begin{align}
     B_{ik+j+1}^{( \alpha )} &= B_{ik+j}^{( \alpha )} - \alpha \nabla_B g \left(B_{ik+j}^{(\alpha)},A_{ik}^{(\alpha)} \right) \\
    A_{ik+j+1}^{( \alpha )} &= A_{ik+j}^{( \alpha )} - \alpha \nabla_A g \left(\left[ \lambda B_{ik}^{(\alpha)} + (1-\lambda)B_{(i+1)k}^{(\alpha)} \right],A_{ik+j}^{(\alpha)} \right).
\end{align}
}

The discrete sets of updates for $A$ and $B$ motivate us to define their continuous affine interpolations:
{\footnotesize
\begin{align}
    Y_{\alpha}(t) &= B_{\lfloor \frac{t}{\alpha} \rfloor}^{(\alpha)} - \left( t - \left\lfloor \frac{t}{\alpha} \right\rfloor \alpha \right) \nabla_B g\left( B_{\lfloor \frac{t}{\alpha} \rfloor}^{(\alpha)}, A_{h(t)}^{(\alpha)} \right) \label{eq Y interpolation main body} \\
    X_{\alpha}(t) &= A_{\lfloor \frac{t}{\alpha} \rfloor}^{(\alpha)} - \left( t - \left\lfloor \frac{t}{\alpha} \right\rfloor \alpha \right) \nabla_A g\left( \left[ \lambda B_{h(t)}^{(\alpha)} + (1-\lambda)B_{h(t)+k}^{(\alpha)} \right], A_{\lfloor \frac{t}{\alpha} \rfloor}^{(\alpha)} \right), \label{eq X interpolation main body}
\end{align}
}
where
\begin{align}
    h(t) &= \left\lfloor \frac{t}{\alpha} \right\rfloor - \left( \left\lfloor \frac{t}{\alpha} \right\rfloor \mod k \right).
\end{align}
We proceed by defining the candidate GF dynamics \emph{a priori}. 
Specifically, let $X : \mathbb{R}_0^+ \to \mathbb{R}^{r \times m}$ and 
$Y : \mathbb{R}_0^+ \to \mathbb{R}^{n \times r}$ be matrix-valued functions 
satisfying the initial-value problem (IVP)
\begin{align}
    \frac{dY(t)}{dt} &= -\nabla_Y g\Big(Y(t),X(t) \Big) \label{eq dY/dt = -f main body}  \\
    \frac{dX(t)}{dt} &= -\nabla_X g\Big(Y(t),X(t) \Big) \label{eq dX/dt = -f main body} \\
    Y(0) &= B_0 \\
    X(0) &= A_0.
\end{align}
We justify the existence of solutions $Y(t)$ and $X(t)$ under the following assumptions:
\begin{assumption}[Uniform Boundedness of Iterates] \label{ass uniform boundedness of iterates main body}
    We assume that all iterates $\norm{B_{ik+j}^{(\alpha)}}$ and $\norm{A_{ik+j}^{(\alpha)}}$ remain uniformly bounded during finetuning.
\end{assumption}
\begin{assumption}[Uniform Boundedness of Gradient]
We assume that the gradient $\norm{\nabla g(B,A)}$ is uniformly bounded on bounded subsets of $\Theta$. 
That is, for every bounded set $S \subset \Theta$, there exists a constant $M_S < \infty$ such that
\[
\|\nabla g(B,A)\| \le M_S \quad \text{for all } (B,A) \in S.
\]
Definitions of norms and gradients on $\Theta$ are provided in Remark~\ref{rem product norm} of Appendix~\ref{app LoRA GF Proof}.
\end{assumption}

In subsequent portions of our derivation, we assume that $g$ is Lipschitz smooth:
\begin{assumption}[Lipschitz Smoothness of $g$] \label{ass lipschitz smoothness main body}
The function $g$ is Lipschitz smooth on bounded subsets of $\Theta$. That is, for any
$(B_1,A_1), (B_2,A_2) \in \Theta$ satisfying
\begin{align}
    \max\!\Big( \norm{B_1}, \norm{A_1}, \norm{B_2}, \norm{A_2} \Big) \le R,
\end{align}
there exists a constant $L_R > 0$ such that
{\small
\begin{align}
    \norm{\nabla g(B_2,A_2) - \nabla g(B_1,A_1)}
    \le L_R \norm{(B_2,A_2) - (B_1,A_1)}.
\end{align}
}
\end{assumption}
Under these assumptions, we show that
\begin{align}
\lim_{\alpha \to 0} ||Y_\alpha(t) - Y(t)|| &= 0 \label{eq Y convergence main body} \\
\lim_{\alpha \to 0} ||X_\alpha(t) - X(t)|| &= 0 \label{eq X convergence main body}
\end{align}
for all $t \geq 0$. The uniform convergence in
(\ref{eq Y convergence main body}--\ref{eq X convergence main body}) implies that the
continuous affine interpolations defined in
(\ref{eq Y interpolation main body}--\ref{eq X interpolation main body}), and thus the
underlying discrete iterates themselves, evolve according to the
dynamical system in
(\ref{eq dY/dt = -f main body}--\ref{eq dX/dt = -f main body}) in the limit where
$\alpha \to 0$. This brings us to our final result:

\begin{theorem} \label{thm LoRA GF equations main body}
Consider an objective function $g: \Theta \to \mathbb{R}$ satisfying Assumptions~\ref{ass uniform boundedness of iterates main body}--\ref{ass lipschitz smoothness main body} which is minimized via Algorithm~\ref{alg: lora-gradient-descent}. During finetuning, the iterates produced by Algorithm~\ref{alg: lora-gradient-descent}
evolve according to the dynamical system in
(\ref{eq dY/dt = -f main body}--\ref{eq dX/dt = -f main body}) for arbitrary
$\lambda \in [0,1]$ and $k \in \mathbb{N}$.
\end{theorem}

The full proof of Theorem~\ref{thm LoRA GF equations main body} is provided in
Appendix~\ref{app LoRA GF Proof}. We now present two problems to demonstrate use of this result to analyze how the accuracy of LoRA depends on update rank.

\subsection{Trace-Squared Objective} \label{sec tr^2 main body}
Consider the finetuning problem
\begin{align}
    \underset{\substack{W \in \mathbb{R}^{n \times n}}}{\min} \, \frac{1}{2}\Tr^2(W_0-W), \label{eq trace squared minimization problem full rank main body}
\end{align}
as well as its rank-$r$ constrained equivalent:
\begin{align}
    \underset{\substack{B \in \mathbb{R}^{n \times r} \\ A \in \mathbb{R}^{r \times n}}}{\min} \, \frac{1}{2}\Tr^2(W_0-BA), \label{eq trace squared minimization problem low rank main body}
\end{align}
where $r < n$. We seek to compare the final losses obtained by full-rank GD on
\eqref{eq trace squared minimization problem full rank main body} and by
LoRA-adapted GD on
\eqref{eq trace squared minimization problem low rank main body} as a function of
$r$. Additionally, we quantify the rank-dependent approximation error between the
respective minimizers of
\eqref{eq trace squared minimization problem full rank main body} and
\eqref{eq trace squared minimization problem low rank main body}.

From the GF equation for full-rank GD, we know that the learning dynamics of \eqref{eq trace squared minimization problem full rank main body} during finetuning are governed by
{\small
\begin{align}
    \frac{dU(t)}{dt} &= -\nabla \left(\frac{1}{2}\Tr^2\Big(W_0-U(t)\Big)\right)
    = \Tr\big(W_0-U(t)\big)\,I_n. \label{eq full rank trace squared ODE main body}
\end{align}
}

Here, $U : \mathbb{R}_0^+ \to \mathbb{R}^{n \times n}$ denotes the full-rank
iterate produced by full-rank GF applied to \eqref{eq trace squared minimization problem full rank main body} at time $t \ge 0$. We show that the unique closed-form solution to \eqref{eq full rank trace squared ODE main body} is given by
\begin{align}
    U(t) &= \frac{1-e^{-nt}}{n}\Tr \left( W_0-Y_0X_0 \right)I_n+Y_0X_0,
\end{align}
where we initialize $U(0) = Y_0X_0$. The full-rank iterates converge to the minimizer of \eqref{eq trace squared minimization problem full rank main body}:
\begin{align}
    \lim_{t\to\infty} U(t)
    = \frac{1}{n}\Tr\!\left( W_0 - Y_0X_0 \right)I_n+Y_0X_0.
\end{align}
For consistency in our comparison of the full-rank and low-rank minimizers,
we employ the same initialization $Y_0X_0$ in the LoRA setting. Following
LoRA initialization~\cite{Hu2022LoRA}, we initialize
$Y_0 = \mathbf{0}_{n \times r}$ and draw the entries of
$X_0 \in \mathbb{R}^{r \times n}$ i.i.d.\ from a centered Gaussian distribution:
\begin{align}
    x_{ij}(0) \sim \mathcal{N}\!\left(0,\sigma^2\right). \label{eq distribution main body}
\end{align}
where $\sigma^2 > 0$. With this initialization
scheme in mind, the full-rank iterates converge to
\begin{align}
    \lim_{t\to\infty} U(t)
    = \frac{1}{n}\Tr\!\left( W_0 \right) I_n, \label{eq final full rank minimizer for trace squared main body}
\end{align}
which yields zero final loss in
\eqref{eq trace squared minimization problem full rank main body}. Full details for our calculation of \eqref{eq final full rank minimizer for trace squared main body} are provided in Appendix~\ref{app full rank learning dynamics for trace squared loss}.

For the rank-$r$ problem in \eqref{eq trace squared minimization problem low rank main body}, we assume uniform boundedness of the iterates as stated in
Assumption~\ref{ass uniform boundedness of iterates main body}. Under this
assumption, we show that the gradient of the trace-squared objective in \eqref{eq trace squared minimization problem low rank main body} is uniformly bounded and Lipschitz continuous on bounded subsets of $\Theta$. Furthermore, we assume
\begin{assumption}[Nonzero Initialization for $\norm{X_0}$] \label{ass X_0 nonzero main body}
    We assume $\norm{X_0} \neq 0$, which holds almost surely for the Gaussian initialization in \eqref{eq distribution main body}.
\end{assumption}
\begin{assumption}[Nonzero Trace for Prefrozen Weights] \label{ass W_0 nonzero main body}
    We assume $\Tr(W_0) \neq 0$.
\end{assumption}
Observe the ODEs in (\ref{eq Tr^2 GF ODE main body Y}--\ref{eq Tr^2 GF ODE main body X}), which govern the learning dynamics of LoRA GD for \eqref{eq trace squared minimization problem low rank main body}. In the case where either of Assumptions~\ref{ass X_0 nonzero main body} or \ref{ass W_0 nonzero main body} is violated, the learning dynamics of \eqref{eq trace squared minimization problem low rank main body} are trivially given by
\begin{align}
    Y(t) &= \mathbf{0}_{n \times r} \\
    X(t) &= \mathbf{0}_{r \times n}
\end{align}
for all $t \geq 0$ due to stationary point initialization.

With all assumptions from Section~\ref{sec LoRA GF main body} satisfied, we know that the learning dynamics of \eqref{eq trace squared minimization problem low rank main body} during finetuning are governed by
{
\begin{align}
    \frac{dY(t)}{dt} &= -\nabla_Y \left( \frac{1}{2}\Tr^2\Big(W_0-Y(t)X(t)\Big) \right) \nonumber \\ &= \Tr(W_0-YX)X^T \label{eq Tr^2 GF ODE main body Y} \\
    \frac{dX(t)}{dt} &= -\nabla_X \left( \frac{1}{2}\Tr^2\Big(W_0-Y(t)X(t)\Big) \right) \nonumber \\ &= \Tr(W_0-YX)Y^T \label{eq Tr^2 GF ODE main body X}
\end{align}
}
Here, $Y: \mathbb{R}_0^+ \to \mathbb{R}^{n \times r}$ and $X: \mathbb{R}_0^+ \to \mathbb{R}^{r \times n}$ denote the low-rank iterates produced by LoRA-adapted GF applied to \eqref{eq trace squared minimization problem low rank main body} at time $t \geq 0$. We find the following closed-form solution to (\ref{eq Tr^2 GF ODE main body Y}--\ref{eq Tr^2 GF ODE main body X}):
\begin{align}
    Y(t) &= q(t)\Tr\!\left(W_0\right)X_0^T \label{eq closed form solution low rank tr^2 main body Y} \\
    X(t) &= p(t)X_0 \label{eq closed form solution low rank tr^2 main body X}
\end{align}
where $p(t)$ and $q(t)$ are smooth scalar functions on $[0,\infty)$ given in Appendix~\ref{app learning dynamics for trace squared loss}. Note that, while we draw the entries of $X_0$ from the distribution in \eqref{eq distribution main body}, the solution in (\ref{eq closed form solution low rank tr^2 main body Y}--\ref{eq closed form solution low rank tr^2 main body X}) holds for arbitrary $X_0 \in \mathbb{R}^{r \times n}$ with nonzero norm. 

Right multiply \eqref{eq closed form solution low rank tr^2 main body Y} by \eqref{eq closed form solution low rank tr^2 main body X} to find the rank-$r$ iterate produced by LoRA-adapted GF applied to \eqref{eq trace squared minimization problem low rank main body} at time $t \geq 0$:
\begin{align}
    Y(t)X(t) &= p(t)q(t)\Tr(W_0)X_0^TX_0.
\end{align}
We find that the rank-$r$ iterates converge to
\begin{align}
    \lim_{t \to \infty}Y(t)X(t) &= \lim_{t \to \infty} p(t)q(t)\Tr\left( W_0 \right)X_0^TX_0 \\
    &= \frac{\Tr(W_0)}{\norm{X_0}^2}X_0^TX_0, \label{eq final rank r minimizer main body}
\end{align}
which yields zero final loss in \eqref{eq trace squared minimization problem low rank main body}. Full details for our calculation of \eqref{eq final rank r minimizer main body} are provided in Appendix~\ref{app learning dynamics for trace squared loss}.

Our analysis thus shows that LoRA-adapted GF converges to a global minimizer of \eqref{eq trace squared minimization problem low rank main body} for arbitrary $r < n$, bringing us to our next result:
\begin{theorem}
Consider the trace-squared optimization problems in 
\eqref{eq trace squared minimization problem full rank main body} and
\eqref{eq trace squared minimization problem low rank main body}. For \eqref{eq trace squared minimization problem low rank main body}, assume
uniform boundedness of the iterates as in
Assumption~\ref{ass uniform boundedness of iterates main body} as well as the nonzero initialization in Assumptions~\ref{ass X_0 nonzero main body}--\ref{ass W_0 nonzero main body}.
Under GF, both the full-rank and LoRA rank-$r$ dynamics
converge to global minimizers attaining zero final loss. In particular, for all
$r < n$, the final losses obtained in the low-rank and full-rank cases are
zero.
\end{theorem}

In Appendix~\ref{app approximation error}, we show that the relative approximation error between $U(t)$ and $Y(t)X(t)$ converges to
\begin{align}
    \lim_{t \to \infty} \frac{\norm{Y(t)X(t)-U(t)}}{\norm{U(t)}}
    &= \frac{\sqrt{n\norm{X_0^TX_0}^2-\norm{X_0}^4}}{\norm{X_0}^2}, \label{eq trace squared lora convergence main body}
\end{align}
for arbitrary $X_0 \in \mathbb{R}^{r \times n}$ with nonzero norm. If we draw the independent entries of $X_0$ from \eqref{eq distribution main body}, we can calculate the expectation of the \emph{square} of our relative approximation error with respect to the initial conditons as a function of $n$ and $r$:
\begin{align}
    \mathbb{E}\left[\lim_{t \to \infty} \frac{\norm{Y(t)X(t)-U(t)}^2}{\norm{U(t)}^2}\right]
    &= \frac{n^2+n-2}{nr+2}. \label{eq expectation of square main body}
\end{align}
By Jensen's Inequality~\cite{dekking2005modern}, we obtain an upper bound on the expectation of the relative approximation error itself:
\begin{align}
    \mathbb{E}\left[\lim_{t \to \infty} \frac{\norm{Y(t)X(t)-U(t)}}{\norm{U(t)}}\right]
    \leq \sqrt{\frac{n^2+n-2}{nr+2}}, \label{eq expectation main body}
\end{align}
which shows that, for fixed $n$, the expected relative approximation error decays at least on the order of $r^{-1/2}$ as $r$ increases. We thus have our next result:
\begin{theorem}
Consider the trace-squared optimization problems in 
\eqref{eq trace squared minimization problem full rank main body} and
\eqref{eq trace squared minimization problem low rank main body}. For \eqref{eq trace squared minimization problem low rank main body}, assume
uniform boundedness of the iterates as in
Assumption~\ref{ass uniform boundedness of iterates main body} as well as the nonzero initialization in Assumptions~\ref{ass X_0 nonzero main body}--\ref{ass W_0 nonzero main body}. Under GF, the approximation error between the converged full-rank and low-rank minimizers is given exactly by \eqref{eq trace squared lora convergence main body}. 

When the entries of $X_0$ are drawn independently according to
\eqref{eq distribution main body}, the expectation of the squared
approximation error at convergence with respect to the initial conditions is given by
\eqref{eq expectation of square main body}. An upper bound on the expectation
of the converged approximation error is given by
\eqref{eq expectation main body}.
\end{theorem}
A complete derivation of these results is provided in Appendix~\ref{app approximation error}. We have thus shown that, using the LoRA GF equations derived in Section~\ref{sec LoRA GF main body}, one can calculate both the expected approximation error and loss at convergence for the rank-$r$ constrained finetuning problem in \eqref{eq trace squared minimization problem low rank main body} as a function of $r$.

\subsection{Low-Rank Approximation}
Consider the low-rank approximation problem
\begin{align}
    \underset{\substack{B \in \mathbb{R}^{n \times r} \\ A \in \mathbb{R}^{r \times m}}}{\min} \, \frac{1}{2}\norm{W_0 - BA}^2, \label{eq squared frobenius minimization problem low rank main body}
\end{align}
where $r < \min(n,m)$. Using the GF equations derived in Section~\ref{sec LoRA GF main body}, we show that LoRA-adapted GD converges to the optimal rank-$r$ solution to \eqref{eq squared frobenius minimization problem low rank main body} given by the EYM Theorem~\cite{EckartYoung1936}.

We again assume uniform boundedness for the iterates as stated in Assumption~\ref{ass uniform boundedness of iterates main body}. Under this
assumption, we show that the gradient of the objective in \eqref{eq squared frobenius minimization problem low rank main body} is uniformly bounded and Lipschitz continuous on bounded subsets of $\Theta$. Furthermore, we assume
\begin{assumption}[Rank of $W_0$] \label{ass rank W_0 main body}
    Let $k := \rank(W_0)$. We assume that
    \begin{align}
    r < k \leq \min(n,m).
    \end{align}
\end{assumption}
Violation of Assumption~\ref{ass rank W_0 main body} renders
\eqref{eq squared frobenius minimization problem low rank main body}
an over-parameterized matrix factorization problem rather than a
low-rank approximation problem.

With all assumptions from Section~\ref{sec LoRA GF main body} satisfied, the learning dynamics for \eqref{eq squared frobenius minimization problem low rank main body} under LoRA are governed by
\begin{align}
   \frac{dY(t)}{dt} &= -\nabla_Y \left( \frac{1}{2}\norm{W_0-Y(t)X(t)}^2 \right) \nonumber \\ &= (W_0-YX)X^T \label{eq MF GF ODE main body Y} \\
    \frac{dX(t)}{dt} &= -\nabla_X \left( \frac{1}{2}\norm{W_0-Y(t)X(t)}^2 \right) \nonumber \\ &= Y^T(W_0-YX) \label{eq MF GF ODE main body X}
\end{align}
Under the standard initialization scheme for LoRA described in Section~\ref{sec tr^2 main body}, we are unaware of any closed-form solution to the ODE system above. However, under a spectral initialization scheme similar to that described
for LoRA in previous work~\cite{Meng2024,Balazy2024LoRAXS,xu2025understanding,
wang-etal-2025-milora, Lin2025denselora}, the dynamics in
(\ref{eq MF GF ODE main body Y}--\ref{eq MF GF ODE main body X})
decouple into scalar ODEs governing the evolution of the singular values
of $Y(t)X(t)$ during training. We then solve the resulting scalar equations to demonstrate the convergence of the nonzero singular values of $Y(t)X(t)$ to the top $r$ singular values of $W_0$.

We begin by calculating the singular value decomposition (SVD) of $W_0$:
\begin{align}
    W_0 &= U \Sigma_0 V^T.
\end{align}
where the singular values in $\Sigma_0$ are in non-increasing order along the diagonal. Using the matrices of left and right singular vectors for $W_0$, we define the following transformations of $X$ and $Y$:
\begin{align}
    \tilde{Y} &= U^TY \\
    \tilde{X} &= XV.
\end{align}
Under this change of variables,
(\ref{eq MF GF ODE main body Y}--\ref{eq MF GF ODE main body X}) become
\begin{align}
    \frac{d\tilde{Y}(t)}{dt} &= \left(\Sigma_0 - \tilde{Y}\tilde{X}\right)\tilde{X}^T \label{eq tilde Y SSE main body} \\
    \frac{d\tilde{X}(t)}{dt} &=\tilde{Y}^T \left(\Sigma_0 - \tilde{Y}\tilde{X}\right). \label{eq tilde X SSE main body}
\end{align}
We initialize $\tilde{Y}_0 = \mathbf{0}_{n \times r}$. The initial $\tilde{X}_0$ has zero off-diagonal entries, while its diagonal entries
are drawn independently from a centered Gaussian distribution with variance
$\sigma^2 > 0$, i.e.,
\begin{align}
    \tilde{x}_{ij}(0) &= \begin{cases}
        0 & i \neq j, \\
        \mathcal{N}(0, \sigma^2) & i = j. \label{eq initialization SSE main body}
    \end{cases}
\end{align}
To prevent stationary point initialization in (\ref{eq scalar y ode main body}--\ref{eq scalar x ode main body}), we assume the diagonal entries of $\tilde{X}$ are initially nonzero:
\begin{assumption}[Nonzero Diagonal Initialization] \label{ass nonzero diagonal init main body}
    We assume $\tilde{x}_{ii}(0) \neq 0$ for all $i \in \{1, \ldots ,r\}$, which holds almost surely for the Gaussian initialization in \eqref{eq initialization SSE main body}. 
\end{assumption}
This initialization scheme preserves diagonality during training, which
decouples (\ref{eq MF GF ODE main body Y}--\ref{eq MF GF ODE main body X}) into
$r$ independent scalar dynamical systems:
\begin{align}
    \frac{d{y}}{dt} &= \left( s_0-{y}{x} \right){x} \label{eq scalar y ode main body} \\
    \frac{d{x}}{dt} &= \left( s_0-{y}{x} \right){y} \label{eq scalar x ode main body}
\end{align}
Here, $s_0$ denotes an arbitrary singular value of $W_0$, while $y$ and $x$
represent the corresponding diagonal entries of $\tilde{Y}$ and $\tilde{X}$.

For $s_0 = 0$ (when $k < \min(n,m)$), we initialize at a stationary point of (\ref{eq scalar y ode main body}--\ref{eq scalar x ode main body}), and $y(t)x(t) = 0$ for all $t \geq 0$. As a result, the zero singular values of $W_0$ and $\tilde{Y}(t)\tilde{X}(t)$ are aligned throughout training.

When $s_0 \neq 0$, the ODE system admits the closed-form solution
\begin{align}
    y(t) &= q_s(t) s_0x_0 \label{eq y solution main body} \\
    x(t) &= p_s(t) x_0 \label{eq x solution main body}
\end{align}
where $x_0 := x_{ii}(0)$, and $p_s(t)$ and $q_s(t)$ are smooth scalar functions on $[0,\infty)$ given
in Appendix~\ref{app learning dynamics for SSE}. Taking the product of
\eqref{eq y solution main body} and \eqref{eq x solution main body} yields the
corresponding diagonal entry of $\tilde{Y}(t)\tilde{X}(t)$ associated with $s_0$:
\begin{align}
    y(t)x(t) &= p_s(t) q_s(t) s_0 x_0^2. \label{eq scalar product main body}
\end{align}
Then \eqref{eq scalar product main body} converges to
\begin{align}
    \lim_{t \to \infty} y(t)x(t) &= s_0.
\end{align}
We thus find that the nonzero diagonal entries of $\tilde{Y}(t)\tilde{X}(t)$
align with the top $r$ singular values of $W_0$ as ${t \to \infty}$. The remaining
diagonal entries are zero since
${\rank\!\left(\tilde{Y}(t)\tilde{X}(t)\right) = r}$. Hence, after convergence,
any discrepancy between $\Sigma_0$ and $\tilde{Y}(t)\tilde{X}(t)$ lies in the
extra $k - r$ nonzero singular values of $W_0$. Note that 
\begin{align}
    YX &= U\tilde{Y}\tilde{X}V^T.
\end{align}
Transforming back to $Y$ and $X$, any discrepancy between $Y(t)X(t)$ and $W_0$ after convergence can be attributed to the bottom $\min(n,m)-r$ singular values of $W_0$, i.e.
\begin{align}
    \lim_{t \to \infty} \norm{W_0-Y(t)X(t)}^2 &= \sum\limits_{i = r+1}^{\min(n,m)} s_{0,i}^2 \label{eq final loss MF main body}
\end{align}
where $s_{0,i}$ denotes the $i$th largest singular value of $W_0$. The result in \eqref{eq final loss MF main body} matches that given by EYM for the loss between $W_0$ and its optimal rank-$r$ minimizer~\cite{EckartYoung1936}. Since the bottom $\min(n,m)-r$ singular values of $Y(t)X(t)$ are zero after convergence, we can write the final rank-$r$ minimizer for \eqref{eq squared frobenius minimization problem low rank main body} obtained under LoRA as
\begin{align}
    \lim_{t \to \infty}Y(t)X(t) &= U_r \Sigma_{0,r} V_r^T \label{eq MF minimizer main body}
\end{align}
where $U_r$ and $V_r$ (truncations of $U$ and $V$) are the matrices of left and right singular vectors for the top $r$ singular values of $W_0$, and $\Sigma_{0,r} \in \mathbb{R}^{n \times m}$ is a diagonal matrix containing the top $r$ singular values of $W_0$, with all remaining diagonal entries equal to zero. The minimization problem equivalent to
\eqref{eq squared frobenius minimization problem low rank main body} with
unconstrained rank is trivially solved by $W_0$ itself. Consequently, the
converged approximation error between the full-rank and low-rank solutions
coincides with the final loss in this setting. Our GF analysis
therefore shows that, for low-rank approximation, the accuracy of LoRA scales with $r$ at a rate given by the tail of the singular-value spectrum of $W_0$, with residual error determined by the $(r+1)$th singular value. We thus arrive at our final result:
\begin{theorem}
Consider the optimization problem in 
\eqref{eq squared frobenius minimization problem low rank main body}. Assume
uniform boundedness of the iterates as in
Assumption~\ref{ass uniform boundedness of iterates main body} and nonzero initialization as in Assumption~\ref{ass nonzero diagonal init main body}. Assume $r < \rank(W_0)$ as in Assumption~\ref{ass rank W_0 main body}.

Under LoRA GF, the converged rank-$r$ minimizer of \eqref{eq squared frobenius minimization problem low rank main body} is given by \eqref{eq MF minimizer main body}. Furthermore, the final loss and approximation
error between the low-rank and full-rank solutions are given by
\eqref{eq final loss MF main body}. These results coincide with the optimal
rank-$r$ minimizer and loss characterized by the Eckart--Young--Mirsky theorem.
\end{theorem}

A detailed calculation for this section is provided in Appendix~\ref{app learning dynamics for SSE}.

\section{Conclusion}

In this work, we presented a generalized GD update scheme for
LoRA-adapted finetuning which is agnostic to whether the low-rank factors
$A$ and $B$ are updated simultaneously or sequentially. We provided a rigorous derivation of the GF equations governing this
algorithm under standard boundedness and Lipschitz-smoothness assumptions and showed that the resulting LoRA GF is invariant to choice of update scheme.

Using these dynamical systems governing the behavior of LoRA, we calculated the accuracy of LoRA as a function of rank $r$. For the trace-squared objective, we found that LoRA attains zero final loss for
arbitrary rank. Moreover, under standard LoRA initialization~\cite{Hu2022LoRA}, the expected relative approximation error between the converged low-rank and full-rank
solutions decreases at least as fast as $r^{-1/2}$. For low-rank approximation, we found that the improvement in approximation
accuracy achieved by LoRA depends on the tail values of
the singular-value spectrum of $W_0$ as the rank $r$ increases, consistent with the optimal rank--accuracy
trade-off characterized by the EYM
theorem~\cite{EckartYoung1936}.

Several limitations of our study point to future work. First,
our spectral initialization scheme may prove unfeasible in settings where
computing the full SVD of the prefrozen weights $W_0$ is impractical. Second, the rank--accuracy dependence derived here is specific to our chosen optimization objectives. Extension of this analysis involves repeating our calculation for each new loss function, while only two elementary objectives are studied in this work. 

Future work could address these limitations by developing a generalized
update scheme for SGD adapted via LoRA and deriving its
associated GF equations. Additional directions include
application of our analysis to more complex or task-specific loss functions as well
as design of LoRA-based approaches to low-rank approximation which do not rely
on access to the full SVD of $W_0$.

\bibliography{paper_detailed_refs}

\newpage
\appendix
\onecolumn


\onecolumn

\section{Derivation of LoRA Gradient Flow ODEs} \label{app LoRA GF Proof}

We wish to analyze the learning dynamics of LoRA by deriving a set of so-called gradient flow ODEs which describe how our parameters $B \in \mathbb{R}^{n \times r}$ and $A \in \mathbb{R}^{r \times m}$ evolve during training. When adapted via LoRA, the classical gradient descent algorithm for continuously differentiable objective $g: \mathbb{R}^{n \times r} \times \mathbb{R}^{r \times m} \to \mathbb{R}$ is given below:


\begin{algorithm}[H]
   \caption{Deterministic Gradient Descent adapted via LoRA (Fixed Stepsize)}
   \label{alg: lora-gradient-descent}
   \begin{algorithmic}[1]

      \vspace{4pt}
      \STATE \textbf{Initialize:}
      \STATE \hspace{1em} Choose an initial $B_0 \in \mathbb{R}^{n \times r}$ and $A_0 \in \mathbb{R}^{r \times m}$
      \STATE \hspace{1em} Set $B^* = B_0$, $A^* = A_0$, and $i = 0$
      \STATE \hspace{1em} Choose $\alpha > 0$ and $\lambda \in \left[ 0,1 \right]$

      \vspace{10pt}
      \WHILE{not converged} 

         \vspace{6pt}
         \STATE \textbf{Stage 1: Update $B$ while keeping $A^*$ fixed}

         \FOR{$j = 0, \dots, k-1$}
            \vspace{2pt}
            \STATE \hspace{2em} $B_{j+1} = B_{j} - \alpha \nabla_B g(B_j,A^*)$
            \vspace{2pt}
         \ENDFOR

         \STATE Set $B^* \longleftarrow \lambda B^* + (1-\lambda)B_{k}$
         \STATE Set $B_0 \longleftarrow B_{k}$

         \vspace{10pt}
         \STATE \textbf{Stage 2: Update $A$ while keeping $B^*$ fixed}
         \FOR{$j = 1, \dots, k$}
            \vspace{2pt}
            \STATE \hspace{2em} $A_{j+1} = A_j - \alpha \nabla_A g(B^*, A_j)$
            \vspace{2pt}
         \ENDFOR

         \STATE Set $A^* \longleftarrow A_{k}$
         \STATE Set $A_0 \longleftarrow A_{k}$

         \vspace{10pt}
         \STATE \textbf{Stage 3: Convergence Check}

         \STATE $i \leftarrow i + 1$

         \IF{stopping criteria met}
            \STATE \hspace{1em} \textbf{break}
         \ENDIF

         \vspace{6pt}
      \ENDWHILE

      \vspace{10pt}
      \STATE \textbf{Output:} $B^* A^*
         = \underset{\substack{B \in \mathbb{R}^{n \times r} \\ A \in \mathbb{R}^{r \times m}}}{\arg\min} f(W_0 + BA)$

   \end{algorithmic}
\end{algorithm}

Note that the updates of $A$ in stage 2 of the algorithm above are performed using the convex combination given on line 10. This definition for $B^*$ incorporates information about $B$ from both the current and previous iterations and makes our analysis agnostic to whether updates between $B$ and $A$ are performed sequentially ($\lambda = 0$) or simultaneously ($\lambda = 1$) using data from the previous iteration. Although updates are applied simultaneously in practice for parallelism \cite{hu2021loracode}, we derive the gradient-flow dynamics from the algorithm above to show that our analysis does not depend on whether the updates are performed sequentially or simultaneously.


We wish to show that, in the limit where $\alpha$ approaches zero, the learning dynamics of Algorithm \ref{alg: lora-gradient-descent} are described by the solution $\Big( Y(t), X(t) \Big)$ of the ODE system
\begin{align}
    \frac{dY(t)}{dt} &= -\nabla_Y g\Big( Y(t), X(t) \Big) \label{eq dY/dt final paper} \\
    \frac{dX(t)}{dt} &= -\nabla_X g\Big( Y(t), X(t) \Big) \\
    Y(0) &= B_0, \\
    X(0) &= A_0, \label{eq X(0) final paper}
\end{align}
defined for $t \in [0, T)$, where $T > 0$ is arbitrary.

To begin, define the Cartesian product $\Theta$ as
\begin{align}
    \Theta := \mathbb{R}^{n \times r} \times \mathbb{R}^{r \times m}
\end{align}
and denote the elements $\theta \in \Theta$ as ordered pairs
\begin{align}
    \theta = \left( B , A  \right)
\end{align}
for $B \in \mathbb{R}^{n \times r}$ and $A \in \mathbb{R}^{r \times m}$. For the Frobenius norm on $\mathbb{R}^{n \times r}$ and $\mathbb{R}^{r \times m}$, define the function $\norm{\cdot}: \Theta \to \mathbb{R}_0^+$ as
\begin{align}
    || \theta || &:= \left({||B||^2+||A||^2}\right)^{1/2}
\end{align}
Remark \ref{rem product norm} below shows that $\norm{\cdot}$ is a product norm on $\Theta$.

\begin{remark}[Product Norm on $\Theta$] \label{rem product norm}
Define the space 
\[
\Theta := \mathbb{R}^{n\times r} \times \mathbb{R}^{r\times m},
\]
with elements $\theta \in \Theta$ given as ordered pairs of matrices:
\begin{align}
    \theta &= (B,A), \\
    B &\in \mathbb{R}^{n \times r} \\
    A &\in \mathbb{R}^{r \times m}
\end{align}
equipped with the norm
\[
\|\theta\| := \big( \|B\|^2 + \|A\|^2 \big)^{1/2}.
\]
Then for any $\theta \in \Theta$, we have:
\vspace{-6pt}
\begin{enumerate}
    \item $\|\theta\| \ge 0$
    \item $\|\theta\| = 0$ iff $\theta=(\mathbf{0}_{n \times r},\mathbf{0}_{r \times m})$
    \item Define scalar multiplication on $\Theta$ as $\eta \theta = \left( \eta B , \eta A \right)$. Then we have $\|\eta \theta\| = |\eta|\,\|\theta\|$ for all $\eta \in \mathbb{R}$
    \item For any two elements $\theta_1 = (B_1,A_1)$ and $\theta_2 = (B_2, A_2)$ of $\Theta$, define the binary addition operation $\Theta \times \Theta \to \Theta$ as
    \begin{align}
        \theta_1 + \theta_2 &= (B_1 + B_2, \, A_1 + A_2)
    \end{align}
    Then $\|\theta_1 + \theta_2\| \le \|\theta_1\| + \|\theta_2\|$
\end{enumerate}
\end{remark}
While previous works \cite{folland1999real,AliprantisBorder2006} note that Cartesian products of normed spaces are also normed spaces, we were unable to find a complete proof that the particular function defined above acts as a norm on $\Theta$. For completeness, we present the proof of Remark~\ref{rem product norm} in Appendix~\ref{app product norm definition}, as our derivation depends on this structure.

Before beginning our derivation of (\ref{eq dY/dt final paper}-\ref{eq X(0) final paper}), note also that for any $\theta \in \Theta$, we define the gradient
\begin{align}
    \nabla g( \theta ) &= \nabla g(B , A) = \Big( \nabla_B g(B, A), \, \nabla_A g(B, A)  \Big)
\end{align}
For a time-dependent path $\theta(t) = \Big( Y(t), X(t) \Big)$, we define
\begin{align}
    \frac{d}{dt} \left[ \theta(t) \right] &=  \frac{d}{dt} \left[ \Big( Y(t), X(t) \Big) \right] = \left( \frac{dY(t)}{dt}, \frac{dX(t)}{dt} \right)
\end{align}
as well as
\begin{align}
    \int_{t_1}^{t_2} \theta (s) \, ds &= \left( \int_{t_1}^{t_2} Y(s)ds \, , \,  \int_{t_1}^{t_2} X(s) ds \right).
\end{align}
We are now ready to begin our derivation.

Consider a continuously differentiable function $g: \Theta \to \mathbb{R}$, and suppose we wish to minimize $g$ on $\Theta$. This is equivalent to solving
\begin{align}
    \underset{\substack{B \in \mathbb{R}^{n \times r} \\ A \in \mathbb{R}^{r \times m}}}{\min} g(B , A)
\end{align}
To better illustrate the finetuning aim of LoRA, we consider also a continuously differentiable function $f: \mathbb{R}^{n \times m} \to \mathbb{R}$. In the context of LoRA, we define $g(B,A) := f(W_0+BA)$ for constant $W_0 \in \mathbb{R}^{n \times m}$. Our optimization problem then becomes
\begin{align}
    \underset{\substack{B \in \mathbb{R}^{n \times r} \\ A \in \mathbb{R}^{r \times m}}}{\min} f(W_0+BA) \label{eq LoRA problem final paper}
\end{align}

Consider the iterates for $A$ and $B$ output by Algorithm \ref{alg: lora-gradient-descent} applied to (\ref{eq LoRA problem final paper}). At any given iteration $i$ of Algorithm \ref{alg: lora-gradient-descent}, stages 1 and 2 of the algorithm produce the sequences of iterates as
\begin{align}
    B_{ik+j+1}^{( \alpha )} &= B_{ik+j}^{( \alpha )} - \alpha \nabla_B f \left( W_0+B_{ik+j}^{(\alpha)}A_{ik}^{(\alpha)} \right) \label{eq gradient descent B} \\
    A_{ik+j+1}^{( \alpha )} &= A_{ik+j}^{( \alpha )} - \alpha \nabla_A f \left( W_0+\left[ \lambda B_{ik}^{(\alpha)} + (1-\lambda)B_{(i+1)k}^{(\alpha)} \right]A_{ik+j}^{(\alpha)} \right) \label{eq gradient descent A}
\end{align}
We assume all iterates produced by Algorithm \ref{alg: lora-gradient-descent} lie within some compact domain, as illustrated below:
\begin{assumption}[Uniform Boundedness of Iterates] \label{ass boundedness of iterates}
Define the domain $\mathcal{D}_R \subseteq \Theta$ by
\[
    \mathcal{D}_{R} := \{ (B,A) \in \Theta : \|B\| \le R,\ \|A\| \le R \}.
\]
Then there exists $R' > 0$ such that, for any $ik+j \in \mathbb{N}_0$, we have
\[
    \left( B_{ik+j}^{(\alpha)},\, A_{ik+j}^{(\alpha)} \right) \in \mathcal{D}_{R'}.
\]
\end{assumption}
\begin{remark}[Uniform Boundedness of Products] \label{rem boundedness of products}
    Define the bounded subset of $\mathbb{R}^{n \times m}$:
    \begin{align}
        \mathcal{B}_R := \{ W \in \mathbb{R}^{n \times m}: ||W-W_0|| \leq R \}
    \end{align}
    As a consequence of Assumption \ref{ass boundedness of iterates}, the shifted products $W_0 + B_{ik+j}^{(\alpha)}A_{ik+j}^{(\alpha)}$ of all iterates lie within $\mathcal{B}_{R'^2}$. That is, for any $\left( B_{ik+j}^{(\alpha)},\, A_{ik+j}^{(\alpha)} \right) \in \mathcal{D}_{R'}$, we will have
    \begin{align}
    \norm{W_0 + B_{ik+j}^{(\alpha)}A_{ik+j}^{(\alpha)} - W_0} \leq \norm{B_{ik+j}^{(\alpha)}}\norm{A_{ik+j}^{(\alpha)}} \leq R'^2
    \end{align}
\end{remark}
The assumption and remark above ensure all iterates and their products remain within compact subsets of $\Theta$ and $\mathbb{R}^{n \times m}$ on which all subsequent regularity assumptions for our analysis are imposed. In practical implementations of Algorithm \ref{alg: lora-gradient-descent}, Assumption \ref{ass boundedness of iterates} is automatically satisfied by limitations of floating-point arithmetic.

To illustrate how the iterates evolve during training, construct their analogous continuous affine interpolations:  
\begin{align}
    Y_{\alpha}(t) &= B_{\lfloor \frac{t}{\alpha} \rfloor}^{(\alpha)} - \left( t - \left\lfloor \frac{t}{\alpha} \right\rfloor \alpha \right) \nabla_B f\left( W_0 + B_{\lfloor \frac{t}{\alpha} \rfloor}^{(\alpha)} A_{h(t)}^{(\alpha)} \right) \label{eq Y_alpha} \\
    X_{\alpha}(t) &= A_{\lfloor \frac{t}{\alpha} \rfloor}^{(\alpha)} - \left( t - \left\lfloor \frac{t}{\alpha} \right\rfloor \alpha \right) \nabla_A f\left( W_0 + \left[ \lambda B_{h(t)}^{(\alpha)} + (1-\lambda)B_{h(t)+k}^{(\alpha)} \right] A_{\lfloor \frac{t}{\alpha} \rfloor}^{(\alpha)} \right) \label{eq X_alpha}
\end{align}
where
\begin{align}
    h(t) &= \left\lfloor \frac{t}{\alpha} \right\rfloor - \left( \left\lfloor \frac{t}{\alpha} \right\rfloor \mod k \right)
\end{align}
Under the following boundedness assumption on the objective gradient, the continuous-time interpolations capture the behavior of Algorithm \ref{alg: lora-gradient-descent}.
\begin{assumption}[Uniform Boundedness of Gradient] \label{ass boundedness of gradient}
    Recall that $g(B,A) := f(W_0+BA)$. For every finite $R > 0$, there exists $M_R > 0$ such that, for every $(B,A) \in \mathcal{D}_R$,
    \begin{align}
        \| \nabla g(B, A) \|^2
        = \| \nabla_B g(B,A) \|^2 
        + \| \nabla_A g(B,A) \|^2 
        \le M_R^2,
    \end{align}
    or, equivalently,
    \begin{align}
        \norm{\nabla g(B,A)}^2 &= \norm{\nabla_B f(W_0+BA)}^2 + \norm{\nabla_A f(W_0+BA)}^2 \leq M_R^2
    \end{align}
\end{assumption}
The assumption above enables us to illustrate how $Y_\alpha(t)$ and $X_\alpha(t)$ capture the behavior of the iterates during training. Under Assumption \ref{ass boundedness of gradient}, we will have, for any $ik+j \in \mathbb{N}_0$,
\begin{align}
    Y_\alpha \Big( t= (ik+j)\alpha \Big) &= B_{\lfloor \frac{(ik+j)\alpha}{\alpha} \rfloor}^{(\alpha)} - \left( (ik+j)\alpha - \left\lfloor \frac{(ik+j)\alpha}{\alpha} \right\rfloor \alpha \right) \nabla_B f\left( W_0 + B_{\lfloor \frac{(ik+j)\alpha}{\alpha} \rfloor}^{(\alpha)} A_{h((ik+j)\alpha)}^{(\alpha)} \right) \\
    &= B_{ik+j}^{(\alpha)} - \Big( (ik+j)\alpha - (ik+j) \alpha \Big) \nabla_B f\left( W_0 + B_{ik+j}^{(\alpha)} A_{ik}^{(\alpha)} \right) \\
    &= B_{ik+j}^{(\alpha)}
\end{align}
Similarly,
\begin{align}
    X_\alpha \Big( t= (ik+j)\alpha \Big) &= A_{\lfloor \frac{(ik+j)\alpha}{\alpha} \rfloor}^{(\alpha)} \nonumber \\
    &- \left( (ik+j)\alpha - \left\lfloor \frac{(ik+j)\alpha}{\alpha} \right\rfloor \alpha \right) \nabla_A f\left( W_0 + \left[ \lambda B_{h((ik+j)\alpha)}^{(\alpha)} + (1-\lambda)B_{h((ik+j)\alpha)+k}^{(\alpha)} \right] A_{\lfloor \frac{(ik+j)\alpha}{\alpha} \rfloor}^{(\alpha)} \right) \\
    &= A_{ik+j}^{(\alpha)} - \Big( (ik+j)\alpha - (ik+j)\alpha \Big)\nabla_A f \left( W_0+ \left[ \lambda B_{ik}^{(\alpha)} + (1-\lambda)B_{(i+1)k} \right] A_{ik+j}^{(\alpha)} \right) \\
    &= A_{ik+j}^{(\alpha)}
\end{align}
Thus, the iterates output by Algorithm \ref{alg: lora-gradient-descent} for $B$ and $A$ correspond exactly to the values of the continuous affine interpolations at the discretization points. Note also that $Y_\alpha(t)$ and $X_\alpha(t)$ lie within $\mathcal{D}_{R'}$ for all finite $t$.

Let $T > 0$ be any finite number. To show that the ODEs in (\ref{eq dY/dt final paper}--\ref{eq X(0) final paper}) describe the dynamics of $Y_\alpha(t)$ and $X_\alpha(t)$ in the limit that $\alpha$ approaches zero for any $t \in [0 , T]$, define $Y: [0,T] \to \mathbb{R}^{n \times r}$ and $X: [0,T] \to \mathbb{R}^{r \times m}$ as matrix-valued functions which satisfy (\ref{eq dY/dt final paper}--\ref{eq X(0) final paper}). That is, $Y(t)$ and $X(t)$ are solutions on $[0,T]$ to
\begin{align}
    \frac{dY(t)}{dt} &= -\nabla_Y f\Big(W_0+Y(t)X(t) \Big) \label{eq dY/dt = -f}  \\
    \frac{dX(t)}{dt} &= -\nabla_X f\Big(W_0+Y(t)X(t) \Big) \label{eq dX/dt = -f} \\
    Y(0) &= B_0 \\
    X(0) &= A_0 \label{eq X(0) = A_0}
\end{align}
which we can rewrite using our notation as
\begin{align}
    \frac{d\theta(t)}{dt} &= -\nabla g \Big(\theta(t)\Big) \label{eq d theta/dt} \\
    \theta(0) &= \left( B_0^{(\alpha)} , A_0^{(\alpha)} \right) \label{eq theta(0)}
\end{align}
where $\theta(t) = \Big( Y(t), X(t) \Big)$. Remark~\ref{rem existence of ODE solutions} explains how Assumption~\ref{ass boundedness of gradient} guarantees existence of the solution on $[0,T]$.
\begin{remark}[Existence of ODE Solution] \label{rem existence of ODE solutions}
Since $\nabla g$ is continuous and bounded on any $\mathcal{D}_R$, the initial-value problem in (\ref{eq d theta/dt}--\ref{eq theta(0)}) has a solution for all finite $t \geq 0$. \cite{Walter1998_ODE}. Moreover, any such solution remains bounded on finite time intervals. That is, for any finite $T > 0$, there exists some $R_T > 0$ such that
\begin{align}
    \Big( Y(t) , X(t) \Big) \in \mathcal{D}_{R_T} \hfill \text{ for all } t \in [0, T].
\end{align}
\end{remark}
The boundedness of $\nabla g$ on any $\mathcal{D}_R$, together with the existence and boundedness of solutions to (\ref{eq d theta/dt}--\ref{eq theta(0)}), implies Remark~\ref{rem uniform boundedness of gradient on time}:
\begin{remark}[Uniform Boundedness of Gradient for Finite Time] \label{rem uniform boundedness of gradient on time}
Let finite $T > 0$ be given. For any $t \in [0,T]$, we will have that 
\begin{align}
    \Big( Y(t) , X(t) \Big) \in \mathcal{D}_{R_t}
\end{align}
where $R_t = \max \Big( \norm{Y(t)}, \norm{X(t)} \Big)$. By Assumption~\ref{ass boundedness of gradient}, there exists $M_{R_t} > 0$ such that
\begin{align}
    \norm{\nabla g \Big( \theta(t) \Big)} \leq M_{R_t}
\end{align}
Denote
\begin{align}
    M_T &= \sup\limits_{t \in [0,T]} \{ M_{R_t} \}
\end{align}
Then we will have
\begin{align}
    \norm{\nabla g \Big( \theta(t) \Big)} \leq M_T \hfill \text{ for all } t \in [0,T].
\end{align}
\end{remark}

To demonstrate that the gradient flow ODEs in (\ref{eq d theta/dt}--\ref{eq theta(0)}) give the behavior of $Y_\alpha(t)$ and $X_\alpha(t)$ as $\alpha$ approaches zero, we will show that
\begin{align}
\lim_{\alpha \to 0} ||Y_\alpha(t) - Y(t)|| &= 0 \\
\lim_{\alpha \to 0} ||X_\alpha(t) - X(t)|| &= 0
\end{align}
for all $t \in [0, T]$. Or, equivalently, that
\begin{align}
    \lim_{\alpha \to 0} ||\theta_\alpha (t) - \theta (t)|| &= 0
\end{align}
where $\theta_\alpha (t) = \Big( Y_\alpha(t), X_\alpha(t) \Big)$. Begin by letting $q = ik + j$ be any non-negative integer satisfying $0 \le q+1 \le \frac{T}{\alpha}$, and denote $t_q = q\alpha$. Note that $t_{q+1} - t_{q} = \alpha$. Integrating both sides of (\ref{eq d theta/dt}) with respect to $t$, we have (for any $\alpha > 0$)
\begin{align}
    \int\limits_{t_q}^{t_{q+1}} \frac{d\theta (s)}{ds} \, ds &= -\int\limits_{t_q}^{t_{q+1}} \nabla g\Big( \theta(s) \Big)ds \\
    \theta(t_{q+1}) - \theta(t_q) &= -\int\limits_{t_q}^{t_{q+1}} \nabla g\Big( \theta(s) \Big)ds
\end{align}
Add and subtract $\alpha \nabla g \Big( \theta(t_q) \Big)$ from the right-hand side to get
\begin{align}
    &\theta(t_{q+1}) - \theta(t_q) = -\alpha \nabla g \Big( \theta(t_q) \Big) - \int\limits_{t_q}^{t_{q+1}} \nabla g\Big( \theta(s) \Big) - \nabla g \Big( \theta(t_q) \Big) \, ds \\
    &\theta(t_{q+1}) - \left[\theta(t_q) - \alpha \nabla g \Big( \theta(t_q) \Big) \right] = - \int\limits_{t_q}^{t_{q+1}} \nabla g\Big( \theta(s) \Big) - \nabla g \Big( \theta(t_q) \Big) \, ds \label{eq tau q}
\end{align}
Define $\tau_q \in \Theta$, where
\begin{align}
    \tau_q &:= \theta(t_{q+1}) - \left[\theta(t_q) - \alpha \nabla g \Big( \theta(t_q) \Big) \right] \label{eq tau definition}
\end{align}
and note from (\ref{eq tau q}) that
\begin{align}
    \tau_q &= - \int\limits_{t_q}^{t_{q+1}} \nabla g\Big( \theta(s) \Big) - \nabla g \Big( \theta(t_q) \Big) \, ds \label{eq tau q 2}
\end{align}
To continue with our derivation, we need Assumption \ref{ass Lipschitz continuity} and Lemma \ref{lem normed differences are bounded} below:
\begin{assumption}[Lipschitz Smoothness of $g$] \label{ass Lipschitz continuity}
    Assume $g$ is Lipschitz smooth on any $\mathcal{D}_R$. That is, for any finite $R > 0$, there exists $L_R > 0$ such that, for all \\ $\theta_1, \theta_2~\in~\mathcal{D}_R$, we have
    \begin{align}
        \norm{\nabla g(\theta_2) - \nabla g(\theta_1)} &\leq L_R || \theta_2 - \theta_1||
    \end{align}
\end{assumption}
Lipschitz smoothness of $\nabla g$ on any $\mathcal{D}_R$ gives
\begin{align}
    \norm{ \nabla g\Big( \theta(s) \Big) - \nabla g\Big(\theta (t_q) \Big) } &\leq L_{R_T} \norm{\theta(s)-\theta(t_q)}
\end{align}
The following lemma allows us to bound $\norm{\theta(s)-\theta(t_q)}$:
\begin{lemma}[Changes in $\theta(t)$ in Time are Bounded] \label{lem normed differences are bounded}
    Let $T > 0$ be an arbitrary finite number. For the solution to (\ref{eq d theta/dt}--\ref{eq theta(0)}) which exists on $t \in [0,T]$, we will have
    \begin{align}
        \norm{\frac{d\theta(t)}{dt}} = \norm{\nabla g\Big( \theta(t) \Big)} \hfill \text{ for all } t \in [0,T].
    \end{align}
    By Remark \ref{rem uniform boundedness of gradient on time}, we have that
    \begin{align}
       \norm{\frac{d\theta(t)}{dt}} = \norm{\nabla g\Big( \theta(t) \Big)} \leq M_T \hfill \text{ for all } t \in [0,T].
    \end{align}
\end{lemma}
Using Lemma \ref{lem normed differences are bounded}, the Fundamental Theorem of Calculus gives
\begin{align}
    \norm{\theta(s)-\theta(t_q)} &= \norm{\int\limits_{t_q}^s \frac{d}{ds'}\theta(s') ds'} \\
    &\leq \int\limits_{t_q}^s \norm{\frac{d}{ds'}\theta(s')} ds' \\
    &\leq \int\limits_{t_q}^{t_{q+1}} \norm{\frac{d}{ds'}\theta(s')} ds' \\
    &\leq \int\limits_{t_q}^{t_{q+1}} M_T \, ds' \\
    &= M_T\alpha
\end{align}
and we have $\norm{\theta(s)-\theta(t_q)} \leq M_T\alpha$.

Return to (\ref{eq tau q 2}). Taking the norm of both sides, we have
\begin{align}
    \norm{\tau_q} &= \norm{ \int\limits_{t_q}^{t_{q+1}} \nabla g\Big( \theta(s) \Big) - \nabla g \Big( \theta(t_q) \Big) \, ds } \\
    &\leq \int\limits_{t_q}^{t_{q+1}} \norm{ \nabla g\Big( \theta(s) \Big) - \nabla g \Big( \theta(t_q) \Big) } \, ds \\
    &\leq L_{R_T} \int\limits_{t_q}^{t_{q+1}} \norm{\theta(s)-\theta(t_q)} \, ds \\
    &\leq  L_{R_T}\int\limits_{t_q}^{t_{q+1}} M_T \alpha \, ds \\
    &= L_{R_T}M_T \alpha^2 \label{eq bound on tau_q}
\end{align}
So we have $\norm{\tau_q} \leq L_{R_T}M_T \alpha^2$, where
\begin{enumerate}
    \item $R_T$ bounds $\norm{Y(t)}$ and $\norm{X(t)}$ for all $t \in [0,T]$. \\
    \item $L_{R_T}$ is the Lipschitz constant of $\nabla g( \theta )$ on $\mathcal{D}_{R_T}$. \\
    \item $M_T$ bounds $\norm{\nabla g \Big( \theta (t) \Big)}$ for all $t \in [0,T]$.
\end{enumerate}
Denote $C_T := L_{R_T}M_T$.

Now define
\begin{align}
    e_q &:= \theta(t_q) - \left( B_q^{(\alpha)} \, , \, A_q^{(\alpha)} \right) \\
\end{align}
and rearrange (\ref{eq tau q}--\ref{eq tau q 2}) to get
\begin{align}
    \theta(t_{q+1}) &= \theta(t_q) - \alpha \nabla g \Big( \theta(t_q) \Big) + \tau_q 
\end{align}
The gradient descent iterations in (\ref{eq gradient descent B}--\ref{eq gradient descent A}) give
\begin{align}
    B_{q+1}^{( \alpha )} &= B_{q}^{( \alpha )} - \alpha \nabla_B f \left( W_0+B_{q}^{(\alpha)}A_{q-j}^{(\alpha)} \right) \\
    A_{q+1}^{( \alpha )} &= A_{q}^{( \alpha )} - \alpha \nabla_A f \left( W_0+\left[ \lambda B_{q-j}^{(\alpha)} + (1-\lambda)B_{q-j+k}^{(\alpha)} \right]A_{q}^{(\alpha)} \right)
\end{align}
We can then calculate $e_{q+1}:$
\begin{align}
    e_{q+1} :&= \theta(t_{q+1}) - \Big( B_{q+1}^{(\alpha)} , A_{q+1}^{(\alpha)} \Big) \\
    &= \theta(t_q) - \alpha \nabla g \Big( \theta(t_q) \Big) + \tau_q \nonumber \\
    &\hspace{10pt} - \Bigg( B_{q}^{( \alpha )} - \alpha \nabla_B f \left( W_0+B_{q}^{(\alpha)}A_{q-j}^{(\alpha)} \right) \,  , \, A_{q}^{( \alpha )} - \alpha \nabla_A f \left( W_0+\left[ \lambda B_{q-j}^{(\alpha)} + (1-\lambda)B_{q-j+k}^{(\alpha)} \right]A_{q}^{(\alpha)} \right) \Bigg) \label{eq huge ordered pair}
\end{align}
We can rewrite the ordered pair in (\ref{eq huge ordered pair}) using the definitions of binary addition and scalar multiplication on $\Theta$ (see Remark \ref{rem product norm}) as
\begin{align}
    &\Bigg( B_{q}^{( \alpha )} - \alpha \nabla_B f \left( W_0+B_{q}^{(\alpha)}A_{q-j}^{(\alpha)} \right) \, , \, A_{q}^{( \alpha )} - \alpha \nabla_A f \left( W_0+\left[ \lambda B_{q-j}^{(\alpha)} + (1-\lambda)B_{q-j+k}^{(\alpha)} \right]A_{q}^{(\alpha)} \right) \Bigg) \\
    = &\Big( B_{q}^{( \alpha )} , A_{q}^{( \alpha )} \Big) - \alpha \Bigg( \nabla_B f \left( W_0+B_{q}^{(\alpha)}A_{q-j}^{(\alpha)} \right) , \nabla_A f \left( W_0+\left[ \lambda B_{q-j}^{(\alpha)} + (1-\lambda)B_{q-j+k}^{(\alpha)} \right]A_{q}^{(\alpha)} \right) \Bigg)
\end{align}
Then (\ref{eq huge ordered pair}) becomes
\begin{align}
    e_{q+1} &= \theta(t_q) - \alpha \nabla g \Big( \theta(t_q) \Big) + \tau_q \nonumber \\
    &\hspace{10pt} -\Big( B_{q}^{( \alpha )} , A_{q}^{( \alpha )} \Big) + \alpha \Bigg( \nabla_B f \left( W_0+B_{q}^{(\alpha)}A_{q-j}^{(\alpha)} \right) , \nabla_A f \left( W_0+\left[ \lambda B_{q-j}^{(\alpha)} + (1-\lambda)B_{q-j+k}^{(\alpha)} \right]A_{q}^{(\alpha)} \right) \Bigg) \\
    &= {\theta(t_q) - \Big( B_{q}^{( \alpha )} , A_{q}^{( \alpha )} \Big)} + \tau_q \nonumber \\
    &\hspace{10pt} +\alpha\left[\Bigg( \nabla_B f \left( W_0+B_{q}^{(\alpha)}A_{q-j}^{(\alpha)} \right) , \nabla_A f \left( W_0+\left[ \lambda B_{q-j}^{(\alpha)} + (1-\lambda)B_{q-j+k}^{(\alpha)} \right]A_{q}^{(\alpha)} \right) \Bigg) - \nabla g \Big( \theta(t_q) \Big)\right] \\
    &= e_q + \tau_q \nonumber \\
    &\hspace{10pt}+\alpha\left[\Bigg( \nabla_B f \left( W_0+B_{q}^{(\alpha)}A_{q-j}^{(\alpha)} \right) , \nabla_A f \left( W_0+\left[ \lambda B_{q-j}^{(\alpha)} + (1-\lambda)B_{q-j+k}^{(\alpha)} \right]A_{q}^{(\alpha)} \right) \Bigg) - \nabla g \Big( \theta(t_q) \Big)\right]
\end{align}
We can then bound the norm of $e_{q+1}$:
\begin{align}
    \norm{e_{q+1}} &\leq \norm{e_q} + \norm{\tau_q} \nonumber \\
    &\hspace{10pt}+ \alpha \norm{\Bigg( \nabla_B f \left( W_0+B_{q}^{(\alpha)}A_{q-j}^{(\alpha)} \right) , \nabla_A f \left( W_0+\left[ \lambda B_{q-j}^{(\alpha)} + (1-\lambda)B_{q-j+k}^{(\alpha)} \right]A_{q}^{(\alpha)} \right) \Bigg) - \nabla g \Big( \theta(t_q) \Big)} \label{eq giant difference of norms}
\end{align}
For brevity, denote
\begin{align}
    B_1 &:= B_{q}^{(\alpha)} \\
    A_1 &:= A_{q-j}^{(\alpha)} \\
    B_2 &:=  \lambda B_{q-j}^{(\alpha)} + (1-\lambda)B_{q-j+k}^{(\alpha)} \\
    A_2 &:= A_{q}^{(\alpha)}
\end{align}
We bound the norm of the third term in (\ref{eq giant difference of norms}) above:
\begin{align}
    &\norm{\Bigg( \nabla_B f \left( W_0 + B_1A_1 \right) , \nabla_A f \left( W_0 + B_2A_2 \right) \Bigg) - \nabla g \Big( \theta(t_q) \Big)} \\
    = &\norm{\Bigg( \nabla_B f \left( W_0 + B_1A_1 \right) - \nabla_Y f \Big( W_0 + Y(t_q)X(t_q) \Big) \, , \, \nabla_A f \left( W_0 + B_2A_2 \right) - \nabla_X f \Big( W_0 + Y(t_q)X(t_q)  \Big)  \Bigg)} \\
    = &\norm{\Bigg( \nabla_B f \left( W_0 + B_1A_1 \right) - \nabla_Y f \Big( W_0 + Y(t_q)X(t_q) \Big) \, , \, \mathbf{0}_{r \times m} \Bigg) + \Bigg( \mathbf{0}_{n \times r} \, , \, \nabla_A f \left( W_0 + B_2A_2 \right) - \nabla_X f \Big( W_0 + Y(t_q)X(t_q)  \Big)  \Bigg)} \\
    \leq &\norm{\Bigg( \nabla_B f \left( W_0 + B_1A_1 \right) - \nabla_Y f \Big( W_0 + Y(t_q)X(t_q) \Big) \, , \, \mathbf{0}_{r \times m} \Bigg)} + \norm{\Bigg( \mathbf{0}_{n \times r} \, , \, \nabla_A f \left( W_0 + B_2A_2 \right) - \nabla_X f \Big( W_0 + Y(t_q)X(t_q)  \Big)  \Bigg)} \\
    \leq &\norm{\nabla g(B_1,A_1) - \nabla g \Big( \theta(t_q) \Big)} + \norm{\nabla g(B_2,A_2) - \nabla g \Big( \theta(t_q) \Big)} \label{eq bounding sums of gradient differences}
\end{align}
By Assumption \ref{ass boundedness of iterates} and Remark \ref{rem existence of ODE solutions}, we know that
\begin{align}
    \norm{B_{q}^{(\alpha)}} &\leq R' \label{eq first bound} \\
    \norm{A_{q}^{(\alpha)}} &\leq R' \\
    \norm{A_{q-j}^{(\alpha)}} &\leq R' \\
    \norm{B_{q-j}^{(\alpha)}} &\leq R' \\
    \norm{B_{q-j+k}^{(\alpha)}} &\leq R' \\
    \norm{Y(t_q)} &\leq R_T \\
    \norm{X(t_q)} &\leq R_T \label{eq last bound}
\end{align}
We also have
\begin{align}
    \norm{\lambda B_{q-j}^{(\alpha)} + (1-\lambda)B_{q-j+k}^{(\alpha)}} \leq \lambda \norm{B_{q-j}^{(\alpha)}} + (1-\lambda)\norm{B_{q-j+k}^{(\alpha)}} \leq \lambda R' + (1-\lambda)R' = R' \label{eq convex combination is bounded}
\end{align}
Let $R_0 = \max(R', R_T)$. Then the ordered pairs above satisfy
\begin{align}
    \Big( B_{q}^{(\alpha)} , A_{q-j}^{(\alpha)} \Big) &\in \mathcal{D}_{R_0} \\
    \Big( \lambda B_{q-j}^{(\alpha)} + (1-\lambda)B_{q-j+k}^{(\alpha)} \, , \, A_{q}^{(\alpha)} \Big) &\in \mathcal{D}_{R_0} \\
    \Big( Y(t_q) , X(t_q) \Big) &\in \mathcal{D}_{R_0}
\end{align}
By the Lipschitz smoothness assumption in \ref{ass Lipschitz continuity}, this gives the existence of finite $L_{R_0} > 0$ such that
\begin{align}
    \norm{\nabla g(B_1, A_1) - \nabla g \Big( \theta(t_q) \Big)} &\leq L_{R_0}\norm{(B_1,A_1) - \theta(t_q)} \\
    \norm{\nabla g(B_2, A_2) - \nabla g \Big( \theta(t_q) \Big)} &\leq L_{R_0}\norm{(B_2,A_2) - \theta(t_q)}
\end{align}
Return to (\ref{eq bounding sums of gradient differences}). We have
\begin{align}
    &\norm{\Bigg( \nabla_B f \left( W_0 + B_1A_1 \right) , \nabla_A f \left( W_0 + B_2A_2 \right) \Bigg) - \nabla g \Big( \theta(t_q) \Big)} \\
    \leq &\norm{\nabla g(B_1, A_1) - \nabla g \Big( \theta(t_q) \Big)} + \norm{\nabla g(B_2, A_2) - \nabla g \Big( \theta(t_q) \Big)} \\
    \leq &L_{R_0}\norm{(B_1,A_1) - \theta(t_q)} +  L_{R_0}\norm{(B_2,A_2) - \theta(t_q)} \\
    = &L_{R_0} \Bigg( \norm{(B_1,A_1) - \theta(t_q)} +  \norm{(B_2,A_2) - \theta(t_q)} \Bigg) \label{eq sum of normed differences}
\end{align}
To bound the normed differences above, we can calculate
\begin{align}
    \norm{(B_1,A_1) - \theta(t_q)} &= \norm{(B_1,A_1) - \Big( B_q^{(\alpha)} , A_q^{(\alpha)} \Big)+\Big( B_q^{(\alpha)} , A_q^{(\alpha)} \Big) - \theta(t_q)} \\
    &\leq \norm{(B_1,A_1) - \Big( B_q^{(\alpha)} , A_q^{(\alpha)} \Big)} + \norm{\Big( B_q^{(\alpha)} , A_q^{(\alpha)} \Big) - \theta(t_q)} \\
    &= \norm{(B_{q}^{(\alpha)},A_{q-j}^{(\alpha)}) - \Big( B_q^{(\alpha)} , A_q^{(\alpha)} \Big)} + \norm{e_q} \\
    &= \norm{A_{q-j}^{(\alpha)} - A_q^{(\alpha)}} + \norm{e_q}
\end{align}
Similarly, we have
\begin{align}
    \norm{(B_2,A_2) - \theta(t_q)} &= \norm{(B_2,A_2) - \Big( B_q^{(\alpha)} , A_q^{(\alpha)} \Big)+\Big( B_q^{(\alpha)} , A_q^{(\alpha)} \Big) - \theta(t_q)} \\
    &\leq \norm{(B_2,A_2) - \Big( B_q^{(\alpha)} , A_q^{(\alpha)} \Big)} + \norm{\Big( B_q^{(\alpha)} , A_q^{(\alpha)} \Big) - \theta(t_q)} \\
    &= \norm{\Big(\lambda B_{q-j}^{(\alpha)} + (1-\lambda)B_{q-j+k}^{(\alpha)} \, , \, A_{q}^{(\alpha)}\Big) - \Big( B_q^{(\alpha)} , A_q^{(\alpha)} \Big)} + \norm{e_q} \\
    &= \norm{\lambda B_{q-j}^{(\alpha)} + (1-\lambda)B_{q-j+k}^{(\alpha)}-B_q^{(\alpha)}} + \norm{e_q}
\end{align}
So we get
\begin{align}
    \norm{(B_1,A_1) - \theta(t_q)} &\leq \norm{A_{q-j}^{(\alpha)} - A_q^{(\alpha)}} + \norm{e_q} \label{eq normed difference 1} \\
    \norm{(B_2,A_2) - \theta(t_q)} &\leq \norm{\lambda B_{q-j}^{(\alpha)} + (1-\lambda)B_{q-j+k}^{(\alpha)}-B_q^{(\alpha)}} + \norm{e_q} \label{eq normed difference 2}
\end{align}
Recall our update equations for the gradient descent iterates:
\begin{align}
    B_{q+1}^{( \alpha )} &= B_{q}^{( \alpha )} - \alpha \nabla_B f \left( W_0+B_{q}^{(\alpha)}A_{q-j}^{(\alpha)} \right) \\
    A_{q+1}^{( \alpha )} &= A_{q}^{( \alpha )} - \alpha \nabla_A f \left( W_0+\left[ \lambda B_{q-j}^{(\alpha)} + (1-\lambda)B_{q-j+k}^{(\alpha)} \right]A_{q}^{(\alpha)} \right)
\end{align}
By the bounded gradient assumption from \ref{ass boundedness of gradient}, there exists $M_{R_0}$ such that
\begin{align}
    \norm{B_{q+1}^{( \alpha )}-B_{q}^{( \alpha )}} \leq \alpha M_{R_0} \\
    \norm{A_{q+1}^{( \alpha )}-A_{q}^{( \alpha )}} \leq \alpha M_{R_0}
\end{align}
We can then calculate
\begin{align}
    \norm{A_q^{(\alpha)} - A_{q-j}^{(\alpha)}} &= \norm{A_q^{(\alpha)} - A_{q-j+1}^{(\alpha)} + A_{q-j+1}^{(\alpha)} - A_{q-j}^{(\alpha)}} \\
    &\leq \norm{A_q^{(\alpha)} - A_{q-j+1}^{(\alpha)}} + \norm{A_{q-j+1}^{(\alpha)} - A_{q-j}^{(\alpha)}} \\
    &\leq \norm{A_q^{(\alpha)} - A_{q-j+1}^{(\alpha)}} + \alpha M_{R_0}
\end{align}
Add and subtract the next iterate and apply the triangle inequality $j-1$ more times to obtain
\begin{align}
    \norm{A_q^{(\alpha)} - A_{q-j}^{(\alpha)}} \leq j\alpha M_{R_0}.
\end{align}
Or, more explicitly, we have
\begin{align}
A^{(\alpha)}_{q} - A^{(\alpha)}_{q-j}
&= \sum_{\ell=1}^{j}
   \Big( A^{(\alpha)}_{q-\ell+1} - A^{(\alpha)}_{q-\ell} \Big)
   \label{eq:telescoping-identity}
\end{align}
which gives
\begin{align}
\|A^{(\alpha)}_{q} - A^{(\alpha)}_{q-j}\|
&\le \sum_{\ell=1}^{j}
     \big\|A^{(\alpha)}_{q-\ell+1} - A^{(\alpha)}_{q-\ell}\big\|
     \label{eq:triangle} \\[0.5em]
&\le \sum_{\ell=1}^{j} \alpha M_{R_0}
     \label{eq:one-step-bound} \\[0.5em]
&= j \alpha M_{R_0}. \label{eq:final-telescope}
\end{align}
For $\norm{\lambda B_{q-j}^{(\alpha)} + (1-\lambda)B_{q-j+k}^{(\alpha)}  -B_q^{(\alpha)}}$, we obtain
\begin{align}
    \norm{\lambda B_{q-j}^{(\alpha)} + (1-\lambda)B_{q-j+k}^{(\alpha)}  -B_q^{(\alpha)}} &= \norm{\lambda B_{q-j}^{(\alpha)} + (1-\lambda)B_{q-j+k}^{(\alpha)}  - \lambda B_q^{(\alpha)} - (1-\lambda) B_q^{(\alpha)}} \\
    &\leq \lambda \norm{B_q^{(\alpha)}-B_{q-j}^{(\alpha)}} + (1-\lambda)\norm{B_{q-j+k}^{(\alpha)} - B_q^{(\alpha)}}
\end{align}
The same telescoping argument from (\ref{eq:telescoping-identity}--\ref{eq:final-telescope}) gives
\begin{align}
    \norm{B_q^{(\alpha)}-B_{q-j}^{(\alpha)}} &\leq j\alpha M_{R_0} \\
    \norm{B_{q-j+k}^{(\alpha)} - B_q^{(\alpha)}} &\leq (k-j)\alpha M_{R_0}
\end{align}
Thus, we have
\begin{align}
    \norm{\lambda B_{q-j}^{(\alpha)} + (1-\lambda)B_{q-j+k}^{(\alpha)}  -B_q^{(\alpha)}} &\leq \lambda \norm{B_q^{(\alpha)}-B_{q-j}^{(\alpha)}} + (1-\lambda)\norm{B_{q-j+k}^{(\alpha)} - B_q^{(\alpha)}} \\
    &\leq \lambda j\alpha M_{R_0} + (1-\lambda)(k-j)\alpha M_{R_0} \\
    &= \alpha M_{R_0}\left[ (2\lambda-1)j+(1-\lambda)k \right] \\
    &\leq \alpha M_{R_0}\left[ (2\lambda-1)k+(1-\lambda)k \right] \\
    &= \lambda k \alpha M_{R_0} \\
    &\leq k\alpha M_{R_0}
\end{align}
We have then obtained
\begin{align}
    \norm{A_{q-j}^{(\alpha)} - A_q^{(\alpha)}} &\leq j\alpha M_{R_0} \\
    \norm{\lambda B_{q-j}^{(\alpha)} + (1-\lambda)B_{q-j+k}^{(\alpha)}  -B_q^{(\alpha)}} &\leq k\alpha M_{R_0}
\end{align}
With the above bounds in mind, (\ref{eq normed difference 1}--\ref{eq normed difference 2}) become
\begin{align}
    \norm{(B_1,A_1) - \theta(t_q)} &\leq j\alpha M_{R_0} + \norm{e_q} \\
    \norm{(B_2,A_2) - \theta(t_q)} &\leq k\alpha M_{R_0} + \norm{e_q}
\end{align}
Returning to (\ref{eq sum of normed differences}), we finally obtain the upper bound on the third term in (\ref{eq giant difference of norms}):
\begin{align}
    &\norm{\Bigg( \nabla_B f \left( W_0 + B_1A_1 \right) , \nabla_A f \left( W_0 + B_2A_2 \right) \Bigg) - \nabla g \Big( \theta(t_q) \Big)} \\
    \leq &L_{R_0} \Bigg( \norm{(B_1,A_1) - \theta(t_q)} +  \norm{(B_2,A_2) - \theta(t_q)} \Bigg) \\
    \leq &L_{R_0} \Bigg( j\alpha M_{R_0} + \norm{e_q} + k\alpha M_{R_0} + \norm{e_q} \Bigg) \\
    &= L_{R_0} \Bigg( 2\norm{e_q} + (j+k)\alpha M_{R_0} \Bigg)
\end{align}
The upper bound on $\norm{e_{q+1}}$ finally becomes
\begin{align}
    \norm{e_{q+1}} &\leq \norm{e_q} + \norm{\tau_q} + \alpha L_{R_0} \Bigg( 2\norm{e_q} + (j+k)\alpha M_{R_0} \Bigg) \\
    &= \norm{e_q} + \norm{\tau_q}+2\alpha L_{R_0}\norm{e_q} + (j+k)M_{R_0}L_{R_0}\alpha^2 \\
    &= (1+2\alpha L_{R_0})\norm{e_q} + \norm{\tau_q} + (j+k)M_{R_0}L_{R_0}\alpha^2
\end{align}
Denote $L' :=  2L_{R_0}$ (note $L' > 0$) and recall the bound on $\norm{\tau_q}$ from (\ref{eq bound on tau_q}). We have
\begin{align}
    \norm{e_{q+1}} &\leq \left( 1+\alpha L' \right)\norm{e_q} + L_{R_T}M_T \alpha^2 + (j+k)M_{R_0}L_{R_0}\alpha^2 \\
    &= \left( 1+\alpha L' \right)\norm{e_q} + \Big( L_{R_T}M_T+(j+k)M_{R_0}L_{R_0} \Big)\alpha^2
\end{align}
Denote $C := L_{R_T}M_T+(j+k)M_{R_0}L_{R_0}$, and we arrive at
\begin{align}
    \norm{e_{q+1}} &\leq \left( 1+\alpha L' \right)\norm{e_q} + C\alpha^2
\end{align}
We can then recursively find an upper bound on $\norm{e_q}$ in terms of $\norm{e_0}$:
\begin{align}
    \norm{e_{q}} &\leq \left( 1+\alpha L' \right)\norm{e_{q-1}} + C\alpha^2 \\
    &\leq \left( 1+\alpha L' \right)\Big( \left( 1+\alpha L' \right)\norm{e_{q-2}} + C\alpha^2 \Big) + C\alpha^2 \\
    &\leq \left( 1+\alpha L' \right)\Bigg( \left( 1+\alpha L' \right)\Big( \left( 1+\alpha L' \right)\norm{e_{q-3}} + C\alpha^2 \Big) + C\alpha^2 \Bigg) + C\alpha^2
\end{align}
Applying the recursion $q$ times and collecting the coefficients of each \(C\alpha^{2}\) term produces
\begin{align}
  \norm{e_{q}}  &\leq  (1+\alpha L')^q\norm{e_0} + C\alpha^2 \sum\limits_{p = 0}^{q-1}(1+\alpha L')^{q-1-p}
\end{align}
Recall from (\ref{eq theta(0)}) that the initial conditions for our iterates and the ODE solution are the same, yielding $\norm{e_0} = 0$. We then have
\begin{align}
    \norm{e_{q}}  &\leq  C\alpha^2 \sum\limits_{p = 0}^{q-1}(1+\alpha L')^{q-1-p}
\end{align}
which we can re-index
\begin{align}
    \norm{e_{q}}  &\leq  C\alpha^2 \sum\limits_{p = 0}^{q-1}(1+\alpha L')^{p} \label{eq norm e_q finite geometric series}
\end{align}
Sum the finite geometric series above to find
\begin{align}
    \sum\limits_{p = 0}^{q-1}(1+\alpha L')^{p} &= \frac{1-(1+\alpha L')^q}{1-(1+\alpha L')} = \frac{(1+\alpha L')^q-1}{\alpha L'} \leq \frac{e^{q\alpha L'}-1}{\alpha L'}
\end{align}
where we obtain the inequality by treating $(1+\alpha L')$ as the truncated Taylor series for $e^{\alpha L'}$. Then (\ref{eq norm e_q finite geometric series}) becomes
\begin{align}
    \norm{e_{q}}  &\leq C \frac{e^{q\alpha L'}-1}{L'}\alpha
\end{align}
Denote $G := C \frac{e^{q\alpha L'}-1}{L'}$, and we have
\begin{align}
    \norm{\theta(t_q) - \left( B_q^{(\alpha)} \, , \, A_q^{(\alpha)} \right)} \leq G \alpha
\end{align}
for any $q = ik+j$ satisfying $0 \le q+1 \le \frac{T}{\alpha}$. 

Divide $[0,T]$ into the sub-intervals $[t_q, t_{q+1}]$. For any $t \in [t_q, t_{q+1}]$, we have
\begin{align}
    \norm{\theta_\alpha (t) - \theta (t)} &= \norm{\theta_\alpha (t) - \Big( B_q^{(\alpha)} , A_q^{(\alpha)} \Big) + \Big( B_q^{(\alpha)} , A_q^{(\alpha)} \Big) - \theta(t_q) + \theta(t_q) -\theta(t)} \\
    &\leq \norm{\theta_\alpha (t) - \Big( B_q^{(\alpha)} , A_q^{(\alpha)} \Big)} + \norm{\Big( B_q^{(\alpha)} , A_q^{(\alpha)} \Big) - \theta(t_q)} + \norm{\theta(t_q) -\theta(t)} \\
    &\leq \norm{\theta_\alpha (t) - \Big( B_q^{(\alpha)} , A_q^{(\alpha)} \Big)} + G\alpha + \norm{\theta(t_q) -\theta(t)}
\end{align}
Bound the first term by noting that $\theta_\alpha(t)$ is the affine interpolation between $\Big( B_q^{(\alpha)} , A_q^{(\alpha)} \Big)$ and $\Big( B_{q+1}^{(\alpha)} , A_{q+1}^{(\alpha)} \Big)$. Thus, $\left\lfloor \frac{t}{\alpha} \right\rfloor\alpha = q\alpha = t_q$. From (\ref{eq Y_alpha}--\ref{eq X_alpha}), we have
\begin{align}
    \theta_\alpha(t) &= \Big( B_q^{(\alpha)} , A_q^{(\alpha)} \Big) - \left( t - t_q \right) \Bigg(  \nabla_B f\left( W_0 + B_{q}^{(\alpha)} A_{q-j}^{(\alpha)} \right) \, , \, \nabla_A f\left( W_0 + \left[ \lambda B_{q-j}^{(\alpha)} + (1-\lambda)B_{q-j+k}^{(\alpha)} \right] A_{q}^{(\alpha)} \right) \Bigg) \label{eq first term}
\end{align}
Note from Assumption \ref{ass boundedness of iterates} that $\theta_a(t)$ and $\Big( B_q^{(\alpha)} , A_q^{(\alpha)} \Big)$ are elements of $\mathcal{D}_{R'}$. By the boundedness of the gradient assumption in \ref{ass boundedness of gradient}, we will have
\begin{align}
    &\norm{\Bigg(  \nabla_B f\left( W_0 + B_{q}^{(\alpha)} A_{q-j}^{(\alpha)} \right) \, , \, \nabla_A f\left( W_0 + \left[ \lambda B_{q-j}^{(\alpha)} + (1-\lambda)B_{q-j+k}^{(\alpha)} \right] A_{q}^{(\alpha)} \right) \Bigg)} \\
    = &\sqrt{\norm{\nabla_B f\left( W_0 + B_{q}^{(\alpha)} A_{q-j}^{(\alpha)} \right)}^2 + \norm{\nabla_A f\left( W_0 + \left[ \lambda B_{q-j}^{(\alpha)} + (1-\lambda)B_{q-j+k}^{(\alpha)} \right] A_{q}^{(\alpha)} \right)}^2} \\
    \leq & \sqrt{M_{R'}^2+M_{R'}^2} \\
    = & \sqrt{2}M_{R'}
\end{align}
Since $t \in [t_q , t_{q+1}]$, we will also have $t - t_q \leq \alpha$. Thus, from \ref{eq first term}, we get
\begin{align}
    \norm{\theta_\alpha(t)-\Big( B_q^{(\alpha)} , A_q^{(\alpha)} \Big)} \leq \sqrt{2}M_{R'}\alpha \label{eq first term bound}
\end{align}
To bound the third term, use the fact that
\begin{align}
    \frac{d\theta(t)}{dt} &= -\nabla g \Big( \theta(t) \Big)
\end{align}
From Remark \ref{rem uniform boundedness of gradient on time}, we have that both sides of the ODE above are bounded by $M_T$. We then get
\begin{align}
    \norm{\theta(t_q) -\theta(t)} = \norm{\theta(t) -\theta(t_q)} = \norm{\int\limits_{t_q}^t \frac{d\theta(s)}{ds} \, ds} \leq {\int\limits_{t_q}^t \norm{\frac{d\theta(s)}{ds}} \, ds} \leq M_T \int\limits_{t_q}^t  \, ds = M_T(t-t_q) \leq M_T\alpha \label{eq third term bound}
\end{align}
From (\ref{eq first term bound}) and (\ref{eq third term bound}), we arrive at
\begin{align}
    \norm{\theta_\alpha (t) - \theta (t)} & \leq \norm{\theta_\alpha (t) - \Big( B_q^{(\alpha)} , A_q^{(\alpha)} \Big)} + G\alpha + \norm{\theta(t_q) -\theta(t)} \\
    &\leq \sqrt{2}M_{R'}\alpha + G\alpha + M_T \alpha
\end{align}
Thus, we have
\begin{align}
    0 \leq \norm{\theta_\alpha (t) - \theta (t)} \leq \left( \sqrt{2}M_{R'} + G + M_T  \right)\alpha \label{eq final inequality}
\end{align}
For all $t \in [t_q , t_{q+1}]$, where $[t_q , t_{q+1}] \subseteq [0,T]$. Since the intervals $[t_q,t_{q+1}]$ cover $[0,T]$, the inequality holds 
for all $t \in [0,T]$. Taking the limit of (\ref{eq final inequality}) as $\alpha \to 0$, we finally have
\begin{align}
    \lim_{\alpha \to 0} \norm{\theta_\alpha (t) - \theta (t)} = 0
\end{align}
for all $t \in [0, T]$, where $T > 0$ is any finite number. We have therefore demonstrated that the continuous affine interpolation connecting the iterates of Algorithm \ref{alg: lora-gradient-descent} follow the behavior described by the LoRA gradient flow ODEs given in (\ref{eq dY/dt = -f}--\ref{eq X(0) = A_0}) in the limit that $\alpha$ approaches zero.


\onecolumn

\section{Proof of Remark \ref{rem product norm}} \label{app product norm definition}

Define 
\[
\Theta := \mathbb{R}^{n\times r} \times \mathbb{R}^{r\times m},
\]
with elements $\theta \in \Theta$ given as ordered pairs of matrices:
\begin{align}
    \theta &= (B,A), \\
    B &\in \mathbb{R}^{n \times r} \\
    A &\in \mathbb{R}^{r \times m}
\end{align}
Define the product norm for elements $\theta \in \Theta$ via
\[
\|\theta\| := \big( \|B\|^2 + \|A\|^2 \big)^{1/2}.
\]
We wish to show that for any $\theta \in \Theta$, $\norm{\theta}$ is a norm on $\Theta$. That is,
\vspace{-6pt}
\begin{enumerate}
    \item $\|\theta\| \ge 0$
    \item $\|\theta\| = 0$ iff $\theta=(\mathbf{0}_{n \times r},\mathbf{0}_{r \times m})$
    \item Define scalar multiplication on $\Theta$ as $\eta \theta = \left( \eta B , \eta A \right)$. Then $\|\eta \theta\| = |\eta|\,\|\theta\|$ for all $\eta \in \mathbb{R}$
    \item For any two elements $\theta_1 = (B_1,A_1)$ and $\theta_2 = (B_2, A_2)$ of $\Theta$, define the binary addition operation $\Theta \times \Theta \to \Theta$ as
    \begin{align}
        \theta_1 + \theta_2 &= (B_1 + B_2, \, A_1 + A_2)
    \end{align}
    Then $\|\theta_1 + \theta_2\| \le \|\theta_1\| + \|\theta_2\|$
\end{enumerate}
\textbf{Proof}: \\
Let $\theta = (A, B)$ be given. Then
\begin{align}
    \|\theta\| = \big( \|B\|^2 + \|A\|^2 \big)^{1/2}
\end{align}
By properties of matrix norms, we have $\| B \| \geq 0$ and $\| A \| \geq 0$. Then $\| B \|^2 \geq 0$ and $\| A \|^2 \geq 0$. This gives
\begin{align}
     &\|B\|^2 + \|A\|^2 \geq 0 \\
   &\big( \|B\|^2 + \|A\|^2 \big)^{1/2} \geq 0
\end{align}
and we have $\| \theta \| \geq 0$. This concludes the proof of 1.

Now suppose $\norm{\theta} = 0$. Then
\begin{align}
    \big( \|B\|^2 + \|A\|^2 \big)^{1/2} &= 0 \\
    \|B\|^2 + \|A\|^2 &= 0 \\
    \norm{B}^2 &= -\norm{A}^2
\end{align}
Now denote $c = \norm{B}^2 = -\norm{A}^2$. Since $\norm{B}^2 \geq 0$, we have $c \geq 0$. On the other hand, $-\norm{A}^2 \leq 0$ gives us $c \leq 0$. Thus, $c = \norm{B}^2 = -\norm{A}^2 = 0$. By properties of matrix norms, we then have
\begin{align}
    B &= \mathbf{0}_{n \times r} \\
    A &= \mathbf{0}_{r \times m}
\end{align}
and $\theta = ( \mathbf{0}_{n \times r} , \mathbf{0}_{r \times m} )$.

Suppose now that $\theta = ( \mathbf{0}_{n \times r} , \mathbf{0}_{r \times m} )$. Then
\begin{align}
    \norm{\theta} &= \left( \norm{\mathbf{0}_{n \times r}}^2 + \norm{\mathbf{0}_{r \times m}}^2 \right)^{1/2} \\
    &= \left( 0^2 + 0^2 \right)^{1/2} \\
    &= 0
\end{align}
and we have proven 2.

For any real number $\eta \in \mathbb{R}$, we will have
\begin{align}
    \norm{\eta \theta} &= \left( \norm{\eta B}^2 + \norm{\eta A}^2 \right)^{1/2} \\
    &= \left( \eta^2\norm{B}^2 + \eta^2\norm{A}^2 \right)^{1/2} \\
    &= \eta \left( \norm{B}^2 + \norm{A}^2 \right)^{1/2} \\
    &= \eta \norm{\theta}
\end{align}
and we have 3.

Finally, for any $\theta_1, \theta_2 \in \Theta$, denote
\begin{align}
    \theta_1 &= (B_1, A_1) \\
    \theta_2 &= (B_2, A_2)
\end{align}
We have
\begin{align}
    \norm{\theta_1+\theta_2} &= \norm{(B_1 + B_2, \, A_1 + A_2)} \\
    &= \left( \norm{B_1+B_2}^2 + \norm{A_1+A_2}^2 \right)^{1/2}
\end{align}
Construct $u, v \in \mathbb{R}^2$, where
\begin{align}
    u &= \Big( \norm{B_1} , \norm{A_1} \Big) \\
    v &= \Big( \norm{B_2} , \norm{A_2} \Big)
\end{align}
and note that
\begin{align}
    \norm{u}_2 &= \left({\norm{B_1}^2 + \norm{A_1}^2}\right)^{1/2} \\
    \norm{v}_2 &= \left({\norm{B_2}^2 + \norm{A_2}^2}\right)^{1/2} \\
    \norm{u+v}_2 &= \left( \left(\norm{B_1} + \norm{B_2}\right)^2 + \left(\norm{A_1} + \norm{A_2}\right)^2 \right)^{1/2}
\end{align}
where $\norm{\cdot}_2$ denotes the vector Euclidean norm. The triangle inequality gives
\begin{align}
    \norm{u+v}_2 \leq \norm{u}_2 + \norm{v}_2
\end{align}
Putting this all together, we have
\begin{align}
    \norm{\theta_1+\theta_2} &= \norm{(B_1 + B_2, \, A_1 + A_2)} \\
    &= \left( \norm{B_1+B_2}^2 + \norm{A_1+A_2}^2 \right)^{1/2} \\
    &\leq \left( \norm{B_1}^2 + 2\norm{B_1}\norm{B_2} + \norm{B_2}^2 + \norm{A_1}^2 + 2\norm{A_1}\norm{A_2}+\norm{A_2}^2 \right)^{1/2} \\
    &= \left( \left(\norm{B_1} + \norm{B_2}\right)^2 + \left(\norm{A_1} + \norm{A_2}\right)^2 \right)^{1/2} \\
    &\leq \left({\norm{B_1}^2 + \norm{A_1}^2}\right)^{1/2} + \left({\norm{B_2}^2 + \norm{A_2}^2}\right)^{1/2} \\
    &= \norm{\theta_1} + \norm{\theta_2}
\end{align}
and the proof of Remark \ref{rem product norm} is complete.



\onecolumn

\section{Learning Dynamics for Trace-Squared Loss (Low-rank)} \label{app learning dynamics for trace squared loss}

Let $W_0 \in \mathbb{R}^{n \times n}$ be a matrix of frozen pretraining weights. We wish to analyze the learning dynamics of the low-rank optimizer $BA \in \mathbb{R}^{n \times n}$ produced by applying LoRA to the finetuning problem
\begin{align}
    \underset{\substack{B \in \mathbb{R}^{n \times r} \\ A \in \mathbb{R}^{r \times n}}}{\min} \, \frac{1}{2}\Tr^2(W_0-BA) \label{eq trace squared objective}
\end{align}
where $r << n$. The partial gradients for our objective $g(B,A) = \dfrac{1}{2}\Tr^2(W_0-BA)$ are given by
\begin{align}
    \nabla_A g(B,A) &= -\Tr(W_0-BA)B^T \label{eq partial g A} \\
    \nabla_B g(B,A) &= -\Tr(W_0-BA)A^T \label{eq partial g B}
\end{align}
\begin{assumption}[Bounded Domain] \label{ass boundedness of iterates for Tr^2}
    We optimize (\ref{eq trace squared objective}) over the subspace $\mathcal{D}_{R'} \subseteq \Theta$, where $R' > 0$ is some finite number. In other words, we assume that there exists some $R' > 0$ such that the norms on both $B$ and $A$ remain bounded above by $R'$ during training.
\end{assumption}
In practical settings, Assumption 1.1 is automatically enforced by computational memory constraints (e.g., finite-precision arithmetic and fixed-parameter storage). The assumption above ensures that the objective gradient remains bounded during training:
\begin{remark}[Boundedness of Gradient for Trace Squared Loss] \label{rem boundedness of gradient Tr^2}
    The objective gradient $\nabla g(B,A)$ remains bounded above during training. That is, for all $(B,A) \in \mathcal{D}_{R'}$, we have
    \begin{align}
        \norm{\nabla g(B,A)}^2 &= \norm{\nabla_B g(B,A)}^2 + \norm{\nabla_A g(B,A)}^2 \\
        &= \norm{-\Tr(W_0-BA)A^T}^2 + \norm{-\Tr(W_0-BA)B^T}^2 \\
        &\leq \Tr^2(W_0-BA) \Big( \norm{A^T}^2 + \norm{B^T}^2  \Big) \\
        &\leq 2R'^2\Tr^2(W_0-BA) \\
        &\leq 2R'^2 \norm{W_0 - BA}^2 \\
        &\leq 2R'^2 \Big( \norm{W_0} + \norm{BA}  \Big)^2 \\
        &\leq 2R'^2 \Big( \norm{W_0} + \norm{B}\norm{A}  \Big)^2 \\
        &\leq 2R'^2 \Big( \norm{W_0} + R'^2  \Big)^2
    \end{align}
    So $\norm{\nabla g(B,A)} \leq \sqrt{2}R' \Big( \norm{W_0} + R'^2  \Big)$ throughout training.
\end{remark}
We then have Lipschitz smoothness during training:
\begin{lemma}[Lipschitz Smoothness for Trace Squared Objective] \label{lem Lipschitz smoothness Tr^2}
    Our objective gradient $\nabla g$ is Lipschitz smooth in our training domain. Namely, there exists $L_{R'} > 0$ such that, for any $(B_1,A_1) , (B_2,A_2) \in \mathcal{D}_{R'}$, we have
    \begin{align}
        \norm{\nabla g(B_1,A_1) - \nabla g(B_2, A_2)} \leq L_{R'} \norm{ \Big( B_1, A_1 \Big) -  \Big( B_2, A_2 \Big) }
    \end{align}
    Proof of this lemma can be found in Appendix \ref{app Lipschitz smoothness Tr^2}.
\end{lemma}

Having shown that the trace squared objective in (\ref{eq trace squared objective}) satisfies Assumptions \ref{ass boundedness of iterates}, \ref{ass boundedness of gradient}, and \ref{ass Lipschitz continuity} from Appendix \ref{app LoRA GF Proof}, the ODEs describing the learning dynamics of $g(B,A)$ under LoRA are given by (see appendix \ref{app LoRA GF Proof})
\begin{align}
    \frac{dY(t)}{dt} &= \Tr(W_0-YX)X^T \label{eq Tr^2 GF Y(t)} \\
    \frac{dX(t)}{dt} &= \Tr(W_0-YX)Y^T \label{eq Tr^2 GF X(t)}
\end{align}
for any $t \in [0,T]$, where $T > 0$ is arbitrary.
Denote the initial conditions for this problem as
\begin{align}
    Y(0) &= Y_0 \\
    X(0) &= X_0
\end{align}
and note by application of Remark \ref{rem existence of ODE solutions} that a solution to the ODE above exists for all $t \in [0,T]$.
Evaluating (\ref{eq Tr^2 GF Y(t)}-\ref{eq Tr^2 GF X(t)}) at $t = 0$ yields the initial derivatives
\begin{align}
    \frac{dY(t)}{dt}\Big|_{t = 0} &= \Tr(W_0-Y_0X_0)X_0^T \\
    \frac{dX(t)}{dt}\Big|_{t = 0} &= \Tr(W_0-Y_0X_0)Y_0^T \label{eq dx/dt initialization}
\end{align}
Following the initialization scheme commonly used in LoRA literature \cite{Hu2022LoRA, xu2025understanding}, we initialize $Y_0 = \mathbf{0}_{n \times r}$, where $\mathbf{0}_{n \times r}$ denotes the $n \times r$ matrix of all zeroes. Thus, all calculations going forward assume $Y_0 = \mathbf{0}_{n \times r}$. We initialize the elements of $X_0$ from a random normal distribution centered at zero with variance $\sigma^2$. In other words,
\begin{align}
    x_{ij}(0) \sim \mathcal{N}(0, \sigma^2) \label{eq distribution appendix}
\end{align}
where $x_{ij}$ are the individual elements of $X$.
\begin{assumption}[Nonzero $X_0$]
    Assume $x_{ij}(0) \neq 0$ for at least one $x_{ij}(0) \in X_0$, which holds almost surely for the Gaussian initialization in \eqref{eq distribution appendix}. \label{ass X nonzero}
\end{assumption}
\begin{assumption}[Nonzero $W_0$]
    The trace of the matrix of prefrozen weights in nonzero. In other words, $\Tr(W_0) \neq 0$. \label{ass trace W_0 nonzero}
\end{assumption}
Violation of either Assumption~$\ref{ass X nonzero}$ or \ref{ass trace W_0 nonzero} causes the LoRA dynamics to stall at the saddle point $(\mathbf{0}_{n \times r}, \mathbf{0}_{r \times n})$, yielding $\mathbf{0}_{n \times n}$ as the optimizing matrix in either case. Our gradient flow analysis then proves trivial. To converge to the rank $r$ optimizer, our weight matrix and initial conditions must satisfy Assumptions \ref{ass X nonzero} and \ref{ass trace W_0 nonzero}.

With proper initialization, we proceed with our solution to the problem described in (\ref{eq Tr^2 GF Y(t)}-\ref{eq dx/dt initialization}) by first solving for the dynamics of $\Tr(W_0-YX)$ for $t \in [0,T]$. Right multiply both sides of (\ref{eq Tr^2 GF Y(t)}) by $X$, and left multiply both sides of (\ref{eq Tr^2 GF X(t)}) by $Y$ to get
\begin{align}
    \frac{dY}{dt}X &= \Tr(W_0-YX)X^TX \\
    Y\frac{dX}{dt} &= \Tr(W_0-YX)YY^T
\end{align}
Add together the two equations above to get
\begin{align}
    \frac{dY}{dt}X + Y\frac{dX}{dt} &= \Tr(W_0-YX) \left[ X^TX+YY^T \right] \\
    \frac{d(YX)}{dt} &= \Tr(W_0-YX) \left[ X^TX+YY^T \right] \\
    -\frac{d(YX)}{dt} &= -\Tr(W_0-YX) \left[ X^TX+YY^T \right]
\end{align}
Since $\dfrac{dW_0}{dt} = 0$ (pretrained weights are frozen), this is equivalent to
\begin{align}
    \frac{d}{dt}(W_0-YX) &= -\Tr(W_0-YX) \left[ X^TX+YY^T \right] \label{eq W_0-YX}
\end{align}
Take the trace of the both sides of (\ref{eq W_0-YX}) to arrive at
\begin{align}
    \frac{d}{dt}\Tr(W_0-YX) &= -\Tr(W_0-YX)\Tr\left( X^TX + YY^T \right) \label{eq sick of naming equations}
\end{align}
Apply the $\frac{d}{d t}$ operator to both sides of the equation above to arrive at
\begin{align}
    \frac{d^2}{dt^2}\Tr(W_0-YX) &= -\frac{d}{dt}\left[ \Tr(W_0-YX) \right]\Tr\left( X^TX + YY^T \right) - \Tr(W_0-YX) \frac{d}{dt}\left[ \Tr\left( X^TX + YY^T \right) \right] \label{eq almost at a(t)}
\end{align}

To derive an ODE for $\Tr(W_0-YX)$, it remains to find expressions for $\frac{d}{dt}\Tr\left( X^TX + YY^T \right)$ and $\Tr\left( X^TX + YY^T \right)$ in terms of $\Tr(W_0-YX)$. Return to (\ref{eq Tr^2 GF Y(t)}). Right multiply both sides of our ODE by $Y^T$ to find
\begin{align}
    \frac{dY}{dt}Y^T &= \Tr(W_0-YX)X^TY^T \label{eq hmmmm}
\end{align}
Take the transpose of both sides to get
\begin{align}
    Y\frac{dY^T}{dt} &= \Tr(W_0-YX)YX \label{eq hmmmmmmmmmmm}
\end{align}
Add (\ref{eq hmmmm}) and (\ref{eq hmmmmmmmmmmm}) together to find
\begin{align}
    \frac{dY}{dt}Y^T + Y\frac{dY^T}{dt} &= \Tr(W_0-YX)\left(X^TY^T + YX \right) \\
    \frac{d\left( YY^T  \right)}{dt} &= \Tr(W_0-YX)\left(X^TY^T + YX \right)
\end{align}
Taking the trace of both sides, we have
\begin{align}
    \frac{d}{dt}\Tr\left( YY^T \right) &= \Tr(W_0-YX)\Tr\left(X^TY^T + YX \right) \\
    &= 2\Tr(W_0-YX)\Tr\left( YX \right) \label{eq amm}
\end{align}
The same calculation for (\ref{eq Tr^2 GF X(t)}) immediately gives
\begin{align}
    \frac{d}{dt} \Tr\left( X^TX \right) &= 2\Tr(W_0-YX)\Tr\left( YX \right) \label{eq ammmmmmmmm}
\end{align}
Add together (\ref{eq amm}) and (\ref{eq ammmmmmmmm}) to get
\begin{align}
    \frac{d}{dt} \Tr\left( YY^T+X^TX  \right) &= 4\Tr(W_0-YX)\Tr(YX) \label{eq d/dt Tr(XX+YY)}
\end{align}
We now find an expression for $\Tr\left( X^TX + YY^T \right)$ in terms of $\Tr(W_0-YX)$. Return to (\ref{eq sick of naming equations}) and give the following assumption:
\begin{remark}[$\Tr(W_0-YX)$ Nonzero during Training]\label{rem Tr^2 nonzero during training}
    We have that $\Tr(W_0-YX)$ is nonzero during training. That is, for all $t$ prior to convergence, we have
    \begin{align}
        \Tr\Big(W_0-Y(t)X(t)\Big) \neq 0 \label{eq trace nonzero during training}
    \end{align}
    Note from the partial gradients of $g$ in (\ref{eq partial g A}--\ref{eq partial g B}) that $\Tr(W_0-XY) = 0$ only at critical points of $g$. Therefore, (\ref{eq trace nonzero during training}) must hold during training, or, in other words, wherever
    \begin{align}
        \norm{ \frac{dY(t)}{dt} }  + \norm{ \frac{dX(t)}{dt}  } \neq 0
    \end{align}
\end{remark}

With this remark in mind, we can divide both sides of (\ref{eq sick of naming equations}) by $-\Tr(W_0-YX)$ to find
\begin{align}
    \Tr\left( YY^T+X^TX \right) &= -\frac{\frac{d}{dt}\Tr(W_0-YX)}{\Tr(W_0-YX)} \label{eq Tr(XX+YY)}
\end{align}
Finally, we substitute (\ref{eq d/dt Tr(XX+YY)}) and (\ref{eq Tr(XX+YY)}) into (\ref{eq almost at a(t)}) to arrive at
\begin{align}
    \frac{d^2}{dt^2}\Tr(W_0-YX) &= -\frac{d}{dt}\left[ \Tr(W_0-YX) \right]\cdot -\frac{\frac{d}{dt}\Tr(W_0-YX)}{\Tr(W_0-YX)} - \Tr(W_0-YX) \cdot 4\Tr(W_0-YX)\Tr(YX) \\
    \frac{d^2}{dt^2}\Tr(W_0-YX) &= \frac{\Big( \frac{d}{dt} \Tr(W_0-YX) \Big)^2}{\Tr(W_0-YX)} - 4\Tr^2(W_0-YX)\Tr(YX) \\
    \frac{d^2}{dt^2}\Tr(W_0-YX) &=  \frac{\Big( \frac{d}{dt} \Tr(W_0-YX) \Big)^2}{\Tr(W_0-YX)} + 4\Tr^2(W_0-YX)\Tr(-YX) \\
    \frac{d^2}{dt^2}\Tr(W_0-YX) &=  \frac{\Big( \frac{d}{dt} \Tr(W_0-YX) \Big)^2}{\Tr(W_0-YX)} + 4\Tr^2(W_0-YX)\Big(\Tr(W_0-YX) - \Tr(W_0)\Big) \\
    \frac{d^2}{dt^2}\Tr(W_0-YX) &=  \frac{\Big( \frac{d}{dt} \Tr(W_0-YX) \Big)^2}{\Tr(W_0-YX)} + 4\Tr^3(W_0-YX)-4\Tr(W_0)\Tr^2(W_0-YX)
\end{align}
Denote
\begin{align}
    a(t) &:= \Tr\Big(W_0 - Y(t)X(t)\Big) \\
    c &:= \Tr(W_0)
\end{align}
We then have the initial-value problem
\begin{align}
    a''(t) &=  \frac{\Big( a'(t) \Big)^2}{a(t)} + 4a^3(t)-4ca^2(t) \label{eq a(t) ODE 1}
\end{align}
\begin{align}
    a(0) &= \Tr \Big( W_0-Y_0X_0 \Big) = \Tr \left(W_0 \right) \\
    a'(0) &= -\Tr \Big( W_0-Y_0X_0 \Big)\Tr\left(Y_0Y_0^T + X_0^TX_0 \right) \quad \Big( \text{ from  (\ref{eq sick of naming equations})}\Big) \\
    &= -\Tr \Big( W_0 \Big)\norm{X_0}^2 \label{eq a(t) ODE 2}
\end{align}
In Appendix (\ref{app solution to a(t) ODE}), we show that the closed-form solution to the IVP above is given on $[0,T]$ by
\begin{align}
    a(t) &= \sgn (c)\frac{\norm{X_0}^4+4c^2}{2\norm{X_0}^2\sinh\left( \sqrt{\kappa_1}(t+\kappa_2) \right)+4|c|} \label{eq a(t) final}
\end{align}
where the constants $\kappa_1$ and $\kappa_2$ are given by
\begin{align}
    \kappa_1 &= \norm{X_0}^4 +4c^2 \\
    \kappa_2 &= \frac{1}{\sqrt{\norm{X_0}^4+4c^2}}\arsinh\left(\frac{\norm{X_0}^2}{2|c|}\right)
\end{align}
\begin{remark}[Nonzero Denominator] \label{rem nonzero denominator}
    Note that the denominator in (\ref{eq a(t) final}) is strictly positive at initialization and increases monotonically in $t$. Consequently, $a(t)$ is well-defined for all $t \in [0,T]$.
\end{remark}
Returning to (\ref{eq Tr^2 GF Y(t)}--\ref{eq Tr^2 GF X(t)}), our ODEs decouple to give
\begin{align}
    \frac{dY(t)}{dt} &= a(t)X^T \label{eq a(t)X^T} \\
    \frac{dX(t)}{dt} &= a(t)Y^T \label{eq a(t)Y^T}
\end{align}
Apply the $\dfrac{d}{dt}$ operator to both sides of (\ref{eq a(t)X^T}). We have
\begin{align}
    \frac{d^2Y(t)}{dt^2} &= a'(t)X^T + a(t)\frac{dX^T}{dt} \label{eq a'(t)X^T + a(t)}
\end{align}
Take the transpose of (\ref{eq a(t)Y^T}) and plug it into (\ref{eq a'(t)X^T + a(t)}) above to find
\begin{align}
    \frac{d^2Y(t)}{dt^2} &= a'(t)X^T + a(t)\frac{dX^T}{dt} \\
    \frac{d^2Y(t)}{dt^2} &= a'(t)X^T + a^2(t)Y \label{eq asdflkj}
\end{align}
Dividing both sides of ($\ref{eq a(t)X^T}$) by $a(t)$ yields
\begin{align}
    X^T &= \frac{1}{a(t)}\frac{dY}{dt}
\end{align}
Then (\ref{eq asdflkj}) becomes
\begin{align}
    \frac{d^2Y(t)}{dt^2} &= \frac{a'(t)}{a(t)}\frac{dY}{dt} + a^2(t)Y
\end{align}
We are thus left with the linear, decoupled ODE for $Y(t)$:
\begin{align}
    \frac{d^2Y}{dt^2} - \frac{a'(t)}{a(t)}\frac{dY}{dt} - a^2(t)Y &= 0
\end{align}
where $a(t)$ is given in (\ref{eq a(t) final}). Taking a derivative of $a(t)$ yields
\begin{align}
    a'(t) &= \Big( \norm{X_0}^4+4c^2 \Big) \cdot -\sgn(c)\frac{2\norm{X_0}^2\sqrt{\kappa_1}\cosh{\left( \sqrt{\kappa_1}(t+\kappa_2) \right)}}{\Big( 2\norm{X_0}^2\sinh\left( \sqrt{\kappa_1}(t+\kappa_2) \right)+4|c| \Big)^2} \\
    &=  -\sgn(c)\frac{2\norm{X_0}^2\Big( \norm{X_0}^4+4c^2 \Big)^{3/2}\cosh{\left( \sqrt{\kappa_1}(t+\kappa_2) \right)}}{\Big( 2\norm{X_0}^2\sinh\left( \sqrt{\kappa_1}(t+\kappa_2) \right)+4|c| \Big)^2}
\end{align}
Then we can calculate
\begin{align}
    \frac{a'(t)}{a(t)} &= -\sgn(c)\frac{2\norm{X_0}^2\Big( \norm{X_0}^4+4c^2 \Big)^{3/2}\cosh{\left( \sqrt{\kappa_1}(t+\kappa_2) \right)}}{\Big( 2\norm{X_0}^2\sinh\left( \sqrt{\kappa_1}(t+\kappa_2) \right)+4|c| \Big)^2} \cdot \sgn(c)\frac{2\norm{X_0}^2\sinh\left( \sqrt{\kappa_1}(t+\kappa_2) \right)+4|c|}{\norm{X_0}^4+4c^2} \\
    &= -\frac{\norm{X_0}^2\Big( \norm{X_0}^4+4c^2 \Big)^{1/2}\cosh{\left( \sqrt{\kappa_1}(t+\kappa_2) \right)}}{\norm{X_0}^2\sinh\left( \sqrt{\kappa_1}(t+\kappa_2) \right)+2|c|} \\
    &= -\frac{\norm{X_0}^2\sqrt{\kappa_1}\cosh{\left( \sqrt{\kappa_1}(t+\kappa_2) \right)}}{\norm{X_0}^2\sinh\left( \sqrt{\kappa_1}(t+\kappa_2) \right)+2|c|}
\end{align}
We also have
\begin{align}
    a^2(t) &= \left( \frac{\norm{X_0}^4+4c^2}{2\norm{X_0}^2\sinh\left( \sqrt{\kappa_1}(t+\kappa_2) \right)+4|c|} \right)^2 \\
    &= \frac{\kappa_1^2}{\Big(2\norm{X_0}^2\sinh\left( \sqrt{\kappa_1}(t+\kappa_2) \right)+4|c|\Big)^2}
\end{align}
Our next task is to solve the following ODE for $Y(t)$:
\begin{align}
    \frac{d^2Y(t)}{dt^2} - \frac{a'(t)}{a(t)}\frac{dY(t)}{dt} - a^2(t)Y(t) &= 0 .\label{eq big ODE to solve}
\end{align}
subject to the initial conditions
\begin{align}
    Y(0) = Y_0, \qquad 
    \frac{dY(t)}{dt}\Big|_{t = 0} = \Tr(W_0-Y_0X_0)X_0^T .
\end{align}
Although we adopt the initialization $Y_0 = \mathbf{0}_{n \times r}$, we keep $Y_0$ general in the calculation below to obtain a unified expression for the corresponding solution $X(t)$.

Let 
\begin{align}
    s(t) &= 2\norm{X_0}^2\sinh\left( \sqrt{\kappa_1}(t+\kappa_2) \right)+4|c| \\
    s'(t) &= 2\sqrt{\kappa_1}\norm{X_0}^2\cosh\left( \sqrt{\kappa_1}(t+\kappa_2) \right) \\
    s''(t) &= 2\kappa_1\norm{X_0}^2\sinh\left( \sqrt{\kappa_1}(t+\kappa_2) \right)
\end{align}
Then (\ref{eq big ODE to solve}) becomes
\begin{align}
    \frac{d^2Y}{dt^2} + \frac{s'(t)}{s(t)}\frac{dY}{dt}-\frac{\kappa_1^2}{s^2(t)} Y &= 0
\end{align}
Multiply both sides of the equation above by $s(t)$ to find
\begin{align}
    s(t)\frac{d^2Y}{dt^2} + s'(t)\frac{dY}{dt} - \frac{\kappa_1^2}{s(t)}Y &= 0 \label{eq big ODE number 2}
\end{align}
Note that $s(t) > 0$ on $[0,T]$ and make the substitution
\begin{align}
    Y(t) &= \frac{U(t)}{\sqrt{s(t)}} = s^{-1/2}(t)U(t) \\
    Y'(t) &= -\frac{1}{2}s^{-3/2}(t)s'(t)U(t) + s^{-1/2}(t)U'(t) \\
    Y''(t) &= \frac{3}{4}s^{-5/2}(t)s'^2(t)U(t)-\frac{1}{2}s^{-3/2}(t)s''(t)U(t)-\frac{1}{2}s^{-3/2}(t)s'(t)U'(t)-\frac{1}{2}s^{-3/2}(t)s'(t)U'(t)+s^{-1/2}(t)U''(t)
\end{align}
Then we have
\begin{align}
    s(t)Y''(t) &= \frac{3}{4}s^{-3/2}(t)s'^2(t)U(t)-\frac{1}{2}s^{-1/2}(t)s''(t)U(t)-\frac{1}{2}s^{-1/2}(t)s'(t)U'(t)-\frac{1}{2}s^{-1/2}(t)s'(t)U'(t)+s^{1/2}(t)U''(t)
\end{align}
as well as
\begin{align}
    s'(t)Y'(t) &= -\frac{1}{2}s^{-3/2}(t)s'^2(t)U(t) + s^{-1/2}(t)s'(t)U'(t)
\end{align}
and
\begin{align}
    \frac{\kappa_1^2}{s(t)}Y &= \kappa_1^2s^{-3/2}(t)U(t)
\end{align}
Then (\ref{eq big ODE number 2}) becomes
\begin{align}
&\frac{3}{4}\, s^{-3/2}(t)\, s'^2(t)\, U(t)
-\frac{1}{2}\, s^{-1/2}(t)\, s''(t)\, U(t)
-\frac{1}{2}\, s^{-1/2}(t)\, s'(t)\, U'(t)
-\frac{1}{2}\, s^{-1/2}(t)\, s'(t)\, U'(t) \\
&\quad
+ s^{1/2}(t)\, U''(t)
-\frac{1}{2}\, s^{-3/2}(t)\, s'^2(t)\, U(t)
+ s^{-1/2}(t)\, s'(t)\, U'(t)
- \kappa_1^{2}\, s^{-3/2}(t)\, U(t)
= 0 .
\end{align}
or, after combining like terms,
\begin{align}
    \frac{1}{4}s^{-3/2}(t)\, s'^2(t)\, U(t)-\frac{1}{2}\, s^{-1/2}(t)\, s''(t)\, U(t)+ s^{1/2}(t)\, U''(t) - \kappa_1^{2}\, s^{-3/2}(t)\, U(t) = 0
\end{align}
Multiply both sides by $s^{-1/2}(t)$ and collect all the $U(t)$ terms to arrive at
\begin{align}
    U''(t) + \left(\frac{1}{4}s^{-2}(t)\, s'^2(t) -\frac{1}{2}\, s^{-1}(t)\, s''(t) - \kappa_1^{2}\, s^{-2}(t) \right)U(t) &= 0 \\
    U''(t) + \left(  \frac{s'^2(t)-2s(t)s''(t)-4\kappa_1^2}{4s^2(t)} \right)U(t) &= 0 \label{eq U ode}
\end{align}
The coefficient on our $U(t)$ term simplifies:
{\small
\[
\begin{aligned}
\frac{s'(t)^2 - 2s(t)s''(t) - 4\kappa_1^2}{4s(t)^2}
&=
\frac{
4\kappa_1\|X_0\|^4 \cosh^2\!\big(\sqrt{\kappa_1}(t+\kappa_2)\big)
- 4\kappa_1\|X_0\|^2 \sinh\!\big(\sqrt{\kappa_1}(t+\kappa_2)\big)
\Big(
2\|X_0\|^2 \sinh\!\big(\sqrt{\kappa_1}(t+\kappa_2)\big)
+ 4|c|
\Big)
- 4\kappa_1^2
}{
4\Big(
2\|X_0\|^2 \sinh\!\big(\sqrt{\kappa_1}(t+\kappa_2)\big)
+ 4|c|
\Big)^2
}
\\[6pt]
\end{aligned}
\]
}
{\small
\[
\begin{aligned}
&=
\frac{\kappa_1}{4} \cdot 
\frac{
4\|X_0\|^4 \cosh^2\!\big(\sqrt{\kappa_1}(t+\kappa_2)\big)
- 8\|X_0\|^4 \sinh^2\!\big(\sqrt{\kappa_1}(t+\kappa_2)\big)
- 16|c|\|X_0\|^2 \sinh\!\big(\sqrt{\kappa_1}(t+\kappa_2)\big)
- 4\kappa_1
}{
4\|X_0\|^4 \sinh^2\!\big(\sqrt{\kappa_1}(t+\kappa_2)\big)
+ 16|c|\|X_0\|^2 \sinh\!\big(\sqrt{\kappa_1}(t+\kappa_2)\big)
+ 16c^2
}
\\[6pt]
&=
\frac{\kappa_1}{4} \cdot 
\frac{
4\|X_0\|^4 \cosh^2\!\big(\sqrt{\kappa_1}(t+\kappa_2)\big)
- 4\|X_0\|^4 \sinh^2\!\big(\sqrt{\kappa_1}(t+\kappa_2)\big)
- 4\|X_0\|^4 \sinh^2\!\big(\sqrt{\kappa_1}(t+\kappa_2)\big)
- 16|c|\|X_0\|^2 \sinh\!\big(\sqrt{\kappa_1}(t+\kappa_2)\big)
- 4\kappa_1
}{
4\|X_0\|^4 \sinh^2\!\big(\sqrt{\kappa_1}(t+\kappa_2)\big)
+ 16|c|\|X_0\|^2 \sinh\!\big(\sqrt{\kappa_1}(t+\kappa_2)\big)
+ 16c^2
}
\\[6pt]
&=
\frac{\kappa_1}{4} \cdot 
\frac{
4\|X_0\|^4
- 4\|X_0\|^4 \sinh^2\!\big(\sqrt{\kappa_1}(t+\kappa_2)\big)
- 16|c|\|X_0\|^2 \sinh\!\big(\sqrt{\kappa_1}(t+\kappa_2)\big)
- 4\kappa_1
}{
4\|X_0\|^4 \sinh^2\!\big(\sqrt{\kappa_1}(t+\kappa_2)\big)
+ 16|c|\|X_0\|^2 \sinh\!\big(\sqrt{\kappa_1}(t+\kappa_2)\big)
+ 16c^2
}
\\[6pt]
&=
\frac{\kappa_1}{4} \cdot 
\frac{
\|X_0\|^4
- \|X_0\|^4 \sinh^2\!\big(\sqrt{\kappa_1}(t+\kappa_2)\big)
- 4|c|\|X_0\|^2 \sinh\!\big(\sqrt{\kappa_1}(t+\kappa_2)\big)
- \norm{X_0}^4-4c^2
}{
\|X_0\|^4 \sinh^2\!\big(\sqrt{\kappa_1}(t+\kappa_2)\big)
+ 4|c|\|X_0\|^2 \sinh\!\big(\sqrt{\kappa_1}(t+\kappa_2)\big)
+ 4c^2
}
\\[6pt]
&=
\frac{\kappa_1}{4} \cdot 
\frac{
-4c^2
- \|X_0\|^4 \sinh^2\!\big(\sqrt{\kappa_1}(t+\kappa_2)\big)
- 4|c|\|X_0\|^2 \sinh\!\big(\sqrt{\kappa_1}(t+\kappa_2)\big)
}{
4c^2
+ \|X_0\|^4 \sinh^2\!\big(\sqrt{\kappa_1}(t+\kappa_2)\big)
+ 4|c|\|X_0\|^2 \sinh\!\big(\sqrt{\kappa_1}(t+\kappa_2)\big)
}
\\[4pt]
&= -\frac{\kappa_1}{4}.
\end{aligned}
\]
}
Then (\ref{eq U ode}) becomes
\begin{align}
    U''(t) -\frac{\kappa_1}{4}U(t) = 0
\end{align}
which is solved by
\begin{align}
    U(t) &= A\sinh\left( \frac{\sqrt{\kappa_1}}{2} t \right) + B\cosh\left( \frac{\sqrt{\kappa_1}}{2} t \right)
\end{align}
We then have
\begin{align}
    Y(t)\sqrt{s(t)} &= A\sinh\left( \frac{\sqrt{\kappa_1}}{2} t \right) + B\cosh\left( \frac{\sqrt{\kappa_1}}{2} t \right) \\
    Y(t) &= \frac{A\sinh\left( \frac{\sqrt{\kappa_1}}{2} t \right) + B\cosh\left( \frac{\sqrt{\kappa_1}}{2} t \right)}{\sqrt{s(t)}} \\
   Y(t) &= \frac{A\sinh\left( \frac{\sqrt{\kappa_1}}{2} t \right) + B\cosh\left( \frac{\sqrt{\kappa_1}}{2} t \right)}{\sqrt{2\norm{X_0}^2\sinh\left( \sqrt{\kappa_1}(t+\kappa_2) \right)+4|c|}} \label{eq Y(t) general solution}
\end{align}
Differentiating (\ref{eq Y(t) general solution}) with respect to $t$ using the product and chain rules yields
\begin{align} 
    Y'(t) &= \frac{\sqrt{C_{1}} \left[A\norm{X_{0}}^{2} \sinh\left(\frac{\sqrt{C_{1}} \left(t + 2C_{2}\right)}{2}\right) - B\norm{X_{0}}^{2} \cosh\left(\frac{\sqrt{C_{1}} \left(t + 2C_{2}\right)}{2}\right) + 2B \left|c\right| \sinh\left(\frac{\sqrt{C_{1}} \, t}{2}\right) + 2A \left|c\right| \cosh\left(\frac{\sqrt{C_{1}} \, t}{2}\right)\right]}{2^{\frac{3}{2}} \left(\norm{X_{0}}^{2} \sinh\left(\sqrt{C_{1}} \left(t + C_{2}\right)\right) + 2 \left|c\right|\right)^{\frac{3}{2}}}
\end{align}
Now use our initial conditions to solve for $A$ and $B$. We have
\begin{align}
    Y(0) &= \frac{B}{\sqrt{2\norm{X_0}^2\sinh\left( \sqrt{\kappa_1}\kappa_2 \right)+4|c|}} \\
    &= \frac{B}{\sqrt{2\norm{X_0}^2 \cdot \frac{\norm{X_0}^2}{2|c|}+4|c|}} \\
    &= \frac{B\sqrt{|c|}}{\sqrt{\norm{X_0}^4+4c^2}} \\
    &= Y_0
\end{align}
and
\begin{align}
    Y'(0) &= \frac{\sqrt{C_{1}} \left[A\norm{X_{0}}^{2} \sinh\left(\kappa_2\sqrt{\kappa_1}\right) - B\norm{X_{0}}^{2} \cosh\left(\kappa_2\sqrt{\kappa_1}\right) + 2A \left|c\right| \right]}{2^{\frac{3}{2}} \left(\norm{X_{0}}^{2} \sinh\left(\kappa_2\sqrt{C_{1}}\right) + 2 \left|c\right|\right)^{\frac{3}{2}}} \\
    &= \frac{\sqrt{C_{1}} \left[A\norm{X_{0}}^{2} \cdot \frac{\norm{X_0}^2}{2|c|} - B\norm{X_{0}}^{2} \cdot \frac{\sqrt{\norm{X_0}^4+4c^2}}{2|c|} + 2A \left|c\right| \right]}{2^{\frac{3}{2}} \left(\norm{X_{0}}^{2} \cdot \frac{\norm{X_0}^2}{2|c|} + 2 \left|c\right|\right)^{\frac{3}{2}}} \\
    &= \frac{\sqrt{C_{1}} \left[A\norm{X_{0}}^{4} - B\norm{X_{0}}^{2} \sqrt{\norm{X_0}^4+4c^2} + 4A \left|c\right|^2 \right]}{2|c|\left(\frac{\norm{X_0}^4}{|c|} + 4 \left|c\right|\right)^{\frac{3}{2}}} \\
    &= \frac{\sqrt{C_{1}|c|} \left[A\norm{X_{0}}^{4} - B\norm{X_{0}}^{2} \sqrt{\norm{X_0}^4+4c^2} + 4A \left|c\right|^2 \right]}{2|c|^{3/2}\left(\frac{\norm{X_0}^4}{|c|} + 4 \left|c\right|\right)^{\frac{3}{2}}} \\
    &= \frac{\sqrt{C_{1}|c|} \left[A\norm{X_{0}}^{4} - B\norm{X_{0}}^{2} \sqrt{\norm{X_0}^4+4c^2} + 4A \left|c\right|^2 \right]}{2\left(\norm{X_{0}}^{4} + 4 c^2\right)^{\frac{3}{2}}} \\
    &= \frac{\sqrt{|c|}\sqrt{\norm{X_0}^4+4c^2} \left[A\left(\norm{X_{0}}^{4}+4c^2\right) - B\norm{X_{0}}^{2} \sqrt{\norm{X_0}^4+4c^2} \right]}{2\left(\norm{X_{0}}^{4} + 4 c^2\right)^{\frac{3}{2}}} \\
    &= \frac{\sqrt{|c|} \left[A\left(\norm{X_{0}}^{4}+4c^2\right) - B\norm{X_{0}}^{2} \sqrt{\norm{X_0}^4+4c^2} \right]}{2\left(\norm{X_{0}}^{4} + 4 c^2\right)} \\
    &= \Tr(W_0-Y_0X_0)X_0^T
\end{align}
So to determine $A$ and $B$, we solve
\begin{align}
    Y_0 &= \frac{B\sqrt{|c|}}{\sqrt{\norm{X_0}^4+4c^2}} \\
    \Tr(W_0-Y_0X_0)X_0^T &= \frac{\sqrt{|c|} \left[A\left(\norm{X_{0}}^{4}+4c^2\right) - B\norm{X_{0}}^{2} \sqrt{\norm{X_0}^4+4c^2} \right]}{2\left(\norm{X_{0}}^{4} + 4 c^2\right)}
\end{align}
which gives
\begin{align}
    A &= \frac{\norm{X_0}^2}{\sqrt{|c|}}Y_0 + \frac{2\Tr(W_0-Y_0X_0)X_0^T}{\sqrt{|c|}} \\
    B &= \frac{\sqrt{\norm{X_0}^4+4c^2}}{\sqrt{|c|}}Y_0
\end{align}
We thus finally arrive at our particular solution for $Y(t)$:
{
\begin{align}
    Y(t) = &\left(\frac{\norm{X_0}^2}{\sqrt{|c|}}Y_0 + \frac{2\Tr(W_0-Y_0X_0)X_0^T}{\sqrt{|c|}} \right)\frac{\sinh\left( \frac{\sqrt{\kappa_1}}{2} t \right)}{\sqrt{2\norm{X_0}^2\sinh\left( \sqrt{\kappa_1}(t+\kappa_2) \right)+4|c|}} \nonumber \\
    &+\frac{\sqrt{\norm{X_0}^4+4c^2}}{\sqrt{|c|}} \cdot \frac{\cosh\left( \frac{\sqrt{\kappa_1}}{2} t \right)}{\sqrt{2\norm{X_0}^2\sinh\left( \sqrt{\kappa_1}(t+\kappa_2) \right)+4|c|}}Y_0
\end{align}
}
Or, gathering the $Y_0$ and $\Tr\left( W_0-Y_0X_0 \right)X_0^T$ terms, we have
\begin{align}
Y(t)
=
&\frac{1}{\sqrt{2\|X_0\|^2\sinh\!\left(\sqrt{\kappa_1}(t+\kappa_2)\right)+4|c|}}
\Bigg[
\frac{\|X_0\|^2}{\sqrt{|c|}}
\sinh\!\left(\frac{\sqrt{\kappa_1}}{2}t\right)
+
\frac{\sqrt{\|X_0\|^4+4c^2}}{\sqrt{|c|}}
\cosh\!\left(\frac{\sqrt{\kappa_1}}{2}t\right)\Bigg]Y_0
\nonumber\\
&
+
\frac{1}{\sqrt{2\|X_0\|^2\sinh\!\left(\sqrt{\kappa_1}(t+\kappa_2)\right)+4|c|}}\cdot
\frac{2}
{\sqrt{|c|}}
\,
\sinh\!\left(\frac{\sqrt{\kappa_1}}{2}t\right)
\,
\Tr\!\left(W_0-Y_0X_0\right)X_0^T
\end{align}
Now define
\begin{align}
    p(t) &:= \frac{1}{\sqrt{2\|X_0\|^2\sinh\!\left(\sqrt{\kappa_1}(t+\kappa_2)\right)+4|c|}}
\Bigg[
\frac{\|X_0\|^2}{\sqrt{|c|}}
\sinh\!\left(\frac{\sqrt{\kappa_1}}{2}t\right)
+
\frac{\sqrt{\|X_0\|^4+4c^2}}{\sqrt{|c|}}
\cosh\!\left(\frac{\sqrt{\kappa_1}}{2}t\right)\Bigg] \\
q(t) &:= \frac{1}{\sqrt{2\|X_0\|^2\sinh\!\left(\sqrt{\kappa_1}(t+\kappa_2)\right)+4|c|}}\cdot
\frac{2}
{\sqrt{|c|}}
\,
\sinh\!\left(\frac{\sqrt{\kappa_1}}{2}t\right)
\end{align}
We then have the closed-form expression for $Y(t)$:
\begin{align}
    Y(t) &= p(t)Y_0 + q(t)\Tr\!\left(W_0-Y_0X_0\right)X_0^T
\end{align}
An identical calculation for $X(t)$, noting the symmetry of (\ref{eq a(t)X^T}--\ref{eq a(t)Y^T}), yields
\begin{align}
    X(t) &= p(t)X_0 + q(t)\Tr\!\left(W_0-Y_0X_0\right)Y_0^T
\end{align}
Using general $Y_0 \in \mathbb{R}^{n \times r}$ enabled us to take advantage of the symmetry in (\ref{eq a(t)X^T}--\ref{eq a(t)Y^T}) to immediately find a closed-form expression for $X(t)$. Now that we have solutions for both $X(t)$ and $Y(t)$, we can enforce $Y_0 = \mathbf{0}_{n \times r}$ to arrive at
\begin{align}
    Y(t) &= q(t)\Tr\!\left(W_0\right)X_0^T \\
    X(t) &= p(t)X_0
\end{align}
which solve (\ref{eq Tr^2 GF Y(t)}--\ref{eq Tr^2 GF X(t)}) for our specified initial conditions. Thus, the rank-$r$ minimizer for \eqref{eq trace squared objective} at time $t \in [0,T]$ is given by
\begin{align}
    Y(t)X(t) &=  p(t)q(t)\Tr\left( W_0 \right)X_0^TX_0 \label{eq rank r minimizer}
\end{align}

It remains to investigate how the rank-$r$ minimizer in \eqref{eq rank r minimizer} behaves for large training times. Since our closed-form expressions for $X(t)$ and $Y(t)$ solve (\ref{eq Tr^2 GF Y(t)}--\ref{eq Tr^2 GF X(t)}) on $[0,T]$ for any $T>0$, they define a single solution to the same ODE system for all $t \in [0,\infty)$. Thus, we investigate the convergence of the rank-$r$ minimizer by examining $Y(t)X(t)$ in the limit that $t \to \infty$. Recall that, as $t \to \infty$,
\begin{align}
\sinh\!\left(\frac{\sqrt{\kappa_1}}{2}t\right) &\sim \frac{1}{2} e^{\frac{\sqrt{\kappa_1}}{2}t}, \\
\cosh\!\left(\frac{\sqrt{\kappa_1}}{2}t\right) &\sim \frac{1}{2} e^{\frac{\sqrt{\kappa_1}}{2}t},
\end{align}
and
\begin{align}
\sinh\!\left(\sqrt{\kappa_1}(t+\kappa_2)\right) &\sim \frac{1}{2} e^{\sqrt{\kappa_1}(t+\kappa_2)}
\end{align}

This gives us
\begin{align}
     \lim_{t \to \infty}\frac{\sinh\!\left(\frac{\sqrt{\kappa_1}}{2}t\right)}{\sqrt{2\|X_0\|^2\sinh\!\left(\sqrt{\kappa_1}(t+\kappa_2)\right)+4|c|}} &=  \lim_{t \to \infty}\frac{\cosh\!\left(\frac{\sqrt{\kappa_1}}{2}t\right)}{\sqrt{2\|X_0\|^2\sinh\!\left(\sqrt{\kappa_1}(t+\kappa_2)\right)+4|c|}} \\
    &=  \lim_{t \to \infty}\frac{\frac{1}{2}e^{\frac{\sqrt{\kappa_1}}{2}t}}{\sqrt{2\|X_0\|^2\cdot \frac{e^{\sqrt{\kappa_1}(t+\kappa_2)}}{2}+4|c|}} \\
    &= \lim_{t \to \infty}\frac{e^{\frac{\sqrt{\kappa_1}}{2}t}}{2\sqrt{\|X_0\|^2\cdot {e^{\sqrt{\kappa_1}(t+\kappa_2)}}+4|c|}} \cdot \frac{e^{-\frac{\sqrt{\kappa_1}}{2}t}}{e^{-\frac{\sqrt{\kappa_1}}{2}t}} \\
    &= \lim_{t \to \infty}\frac{1}{2\sqrt{\|X_0\|^2\cdot {e^{\kappa_2\sqrt{\kappa_1}}}+4|c|e^{-\sqrt{\kappa_1}t}}} \\
    &= \frac{1}{2\norm{X_0}e^{\frac{\kappa_2\sqrt{\kappa_1}}{2}}} \label{eq limit result}
\end{align}
Note that
\begin{align}
    e^{\kappa_2\sqrt{\kappa_1}} &= e^{\arsinh\left( \frac{\norm{X_0}^2}{2|c|}\right)} \\
    &= e^{\ln \left(  \frac{\norm{X_0}^2}{2|c|}+\sqrt{1+\left( \frac{\norm{X_0}^2}{2|c|}\right)^2} \right)} \\
    &= \frac{\norm{X_0}^2}{2|c|}+\sqrt{1+\left( \frac{\norm{X_0}^2}{2|c|}\right)^2} \\
    &= \frac{\|X_0\|^2+\sqrt{\|X_0\|^4+4c^2}}{2|c|}
\end{align}
So \eqref{eq limit result} becomes
\begin{align}
    \frac{1}{2\norm{X_0}} \cdot \sqrt{\frac{2|c|}{\|X_0\|^2+\sqrt{\|X_0\|^4+4c^2}}} &= \frac{\sqrt{|c|}}{\sqrt{2}\norm{X_0}\sqrt{\|X_0\|^2+\sqrt{\|X_0\|^4+4c^2}}} 
\end{align}
So we have
\begin{align}
    \lim_{t \to \infty}\frac{\sinh\!\left(\frac{\sqrt{\kappa_1}}{2}t\right)}{\sqrt{2\|X_0\|^2\sinh\!\left(\sqrt{\kappa_1}(t+\kappa_2)\right)+4|c|}} &=  \lim_{t \to \infty}\frac{\cosh\!\left(\frac{\sqrt{\kappa_1}}{2}t\right)}{\sqrt{2\|X_0\|^2\sinh\!\left(\sqrt{\kappa_1}(t+\kappa_2)\right)+4|c|}} \\
    &= \frac{\sqrt{|c|}}{\sqrt{2}\norm{X_0}\sqrt{\|X_0\|^2+\sqrt{\|X_0\|^4+4c^2}}}
\end{align}
We can then quickly calculate
\begin{align}
    \lim_{t \to \infty} p(t) &= \lim_{t \to \infty}\frac{1}{\sqrt{2\|X_0\|^2\sinh\!\left(\sqrt{\kappa_1}(t+\kappa_2)\right)+4|c|}}
\Bigg[
\frac{\|X_0\|^2}{\sqrt{|c|}}
\sinh\!\left(\frac{\sqrt{\kappa_1}}{2}t\right)
+
\frac{\sqrt{\|X_0\|^4+4c^2}}{\sqrt{|c|}}
\cosh\!\left(\frac{\sqrt{\kappa_1}}{2}t\right)\Bigg] \\
&= \lim_{t \to \infty} \frac{\|X_0\|^2}{\sqrt{|c|}}\cdot\frac{\sinh\!\left(\frac{\sqrt{\kappa_1}}{2}t\right)}{\sqrt{2\|X_0\|^2\sinh\!\left(\sqrt{\kappa_1}(t+\kappa_2)\right)+4|c|}} + \frac{\sqrt{\|X_0\|^4+4c^2}}{\sqrt{|c|}}\cdot\frac{\cosh\!\left(\frac{\sqrt{\kappa_1}}{2}t\right)}{\sqrt{2\|X_0\|^2\sinh\!\left(\sqrt{\kappa_1}(t+\kappa_2)\right)+4|c|}} \\
&= \frac{\|X_0\|^2}{\sqrt{|c|}}\cdot\frac{\sqrt{|c|}}{\sqrt{2}\norm{X_0}\sqrt{\|X_0\|^2+\sqrt{\|X_0\|^4+4c^2}}} + \frac{\sqrt{\|X_0\|^4+4c^2}}{\sqrt{|c|}}\cdot\frac{\sqrt{|c|}}{\sqrt{2}\norm{X_0}\sqrt{\|X_0\|^2+\sqrt{\|X_0\|^4+4c^2}}} \\
&= \frac{\norm{X_0}^2+\sqrt{\|X_0\|^4+4c^2}}{\sqrt{2}\norm{X_0}\sqrt{\|X_0\|^2+\sqrt{\|X_0\|^4+4c^2}}} \\
&= \frac{\sqrt{\norm{X_0}^2+\sqrt{\|X_0\|^4+4c^2}}}{\sqrt{2}\norm{X_0}}
\end{align}
We also have
\begin{align}
    \lim_{t \to \infty} q(t) &= \lim_{t \to \infty}\frac{2}
{\sqrt{|c|}} \cdot  \frac{\sinh\!\left(\frac{\sqrt{\kappa_1}}{2}t\right)}{\sqrt{2\|X_0\|^2\sinh\!\left(\sqrt{\kappa_1}(t+\kappa_2)\right)+4|c|}} \\
&= \frac{2}
{\sqrt{|c|}} \cdot \frac{\sqrt{|c|}}{\sqrt{2}\norm{X_0}\sqrt{\|X_0\|^2+\sqrt{\|X_0\|^4+4c^2}}} \\
&=\frac{2}{\sqrt{2}\norm{X_0}\sqrt{\|X_0\|^2+\sqrt{\|X_0\|^4+4c^2}}}
\end{align}
Now calculate the post-training rank-$r$ minimizer for \eqref{eq trace squared objective}, which converges to
\begin{align}
    \lim_{t \to \infty}Y(t)X(t) &= \lim_{t \to \infty} p(t)q(t)\Tr\left( W_0 \right)X_0^TX_0 \\
    &= \frac{\sqrt{\norm{X_0}^2+\sqrt{\|X_0\|^4+4c^2}}}{\sqrt{2}\norm{X_0}} \cdot \frac{2}{\sqrt{2}\norm{X_0}\sqrt{\|X_0\|^2+\sqrt{\|X_0\|^4+4c^2}}} \cdot \Tr(W_0)X_0^TX_0 \\
&= \frac{\Tr(W_0)}{\norm{X_0}^2}X_0^TX_0
\end{align}
For a quick sanity check, calculate the trace of our final rank-$r$ minimizer:
\begin{align}
    \Tr\left( \frac{\Tr(W_0)}{\norm{X_0}^2}X_0^TX_0 \right) &= \frac{\Tr(W_0)}{\norm{X_0}^2}\Tr\left( X_0^TX_0 \right) = \frac{\Tr(W_0)}{\norm{X_0}^2} \norm{X_0}^2 = \Tr(W_0)
\end{align}
which gives $\lim\limits_{t \to \infty} \dfrac{1}{2}\Tr^2 \Big( W_0-Y(t)X(t) \Big) = 0$ as expected.


\onecolumn

\section{Proof of Lemma \ref{lem Lipschitz smoothness Tr^2}} \label{app Lipschitz smoothness Tr^2}

\begin{lemma}[Lipschitz Smoothness for Trace Squared Objective]
    Define $g: \mathbb{R}^{n \times r} \times \mathbb{R}^{r \times n} \to \mathbb{R}$ via
    \begin{align}
        g(B,A) &:= \frac{1}{2}\Tr^2(W_0-BA)
    \end{align}
    for constant $W_0 \in \mathbb{R}^{n \times n}$. Then the objective gradient $\nabla g$ is Lipschitz smooth in our training domain. Namely, there exists $L_{R'} > 0$ such that, for any $(B_1,A_1) , (B_2,A_2) \in \mathcal{D}_{R'}$, we have
    \begin{align}
        \norm{\nabla g(B_1,A_1) - \nabla g(B_2, A_2)} \leq L_{R'} \norm{ \Big( B_1, A_1 \Big) -  \Big( B_2, A_2 \Big) }
    \end{align}
    \textbf{Proof:} \\
    
    Let $(B_1,A_1) , (B_2,A_2) \in \mathcal{D}_{R'}$ be given. Refer to the partial gradients for $g(B,A)$ in (\ref{eq partial g A}--\ref{eq partial g B}) of Appendix \ref{app learning dynamics for trace squared loss}. We have
    \begin{align}
        \norm{\nabla g(B_1,A_1) - \nabla g(B_2,A_2)} &= \norm{\Bigg( \nabla_B g(B_1,A_1) - \nabla_B g(B_2, A_2)  \, , \, \nabla_A g(B_1,A_1) - \nabla_A g(B_2, A_2) \Bigg)}
    \end{align}
    The partial gradients with respect to $B$ simplify to
    \begin{align}
        \nabla_B g(B_1,A_1) - \nabla_B g(B_2, A_2) &= -\Tr(W_0-B_1A_1)A_1^T +\Tr(W_0-B_2A_2)A_2^T \\
        &= \Tr(W_0-B_2A_2)A_2^T-\Tr(W_0-B_1A_1)A_2^T+\Tr(W_0-B_1A_1)A_2^T-\Tr(W_0-B_1A_1)A_1^T \\
        &= \Big[\Tr(W_0-B_2A_2)-\Tr(W_0-B_1A_1) \Big]A_2^T + \Tr(W_0-B_1A_1)\left( A_2^T - A_1^T \right) \\
        &= \Big[\Tr(B_1A_1-B_2A_2)\Big]A_2^T + \Tr(W_0-B_1A_1)\left( A_2^T - A_1^T \right)
    \end{align}
    Recall that $|\Tr(BA)| \leq \norm{B}\norm{A}$ (Cauchy-Schwarz for the Frobenius inner product). Taking the norm, we find
    \begin{align}
        \norm{\nabla_B g(B_1,A_1) - \nabla_B g(B_2, A_2)} &= \norm{\Big[\Tr(B_1A_1-B_2A_2)\Big]A_2^T + \Tr(W_0-B_1A_1)\left( A_2^T - A_1^T \right)} \\
        &\leq \norm{\Big[\Tr(B_1A_1-B_2A_2)\Big]A_2^T}+\norm{\Tr(W_0-B_1A_1)\left( A_2^T - A_1^T \right)} \\
        &= |\Tr(B_1A_1-B_2A_2)|\norm{A_2^T}+ |\Tr(W_0-B_1A_1)|\norm{A_2^T-A_1^T} \\
        &= |\Tr(B_1A_1-B_1A_2+B_1A_2-B_2A_2)|\norm{A_2^T}+ |\Tr(W_0-B_1A_1)|\norm{A_2^T-A_1^T} \\
        &= \left|\Tr\Big( B_1(A_1-A_2)+ (B_1-B_2)A_2 \Big)\right|\norm{A_2^T}+ |\Tr(W_0-B_1A_1)|\norm{A_2^T-A_1^T} \\
        &= \left|\Tr\Big( B_1(A_1-A_2)\Big)+ \Tr \Big((B_1-B_2)A_2 \Big)\right|\norm{A_2^T}+ |\Tr(W_0-B_1A_1)|\norm{A_2^T-A_1^T} \\
        &\leq \Big( \norm{B_1}\norm{A_1-A_2}+\norm{B_1-B_2}\norm{A_2} \Big)\norm{A_2^T}+ |\Tr(W_0-B_1A_1)|\norm{A_2^T-A_1^T} \\
        &\leq \Big( R'\norm{A_1-A_2}+\norm{B_1-B_2}R' \Big)R'+ |\Tr(W_0-B_1A_1)|\norm{A_2^T-A_1^T} \\
        &\leq \Big( \norm{A_1-A_2}+\norm{B_1-B_2}\Big)R'^2+\Big(|\Tr(W_0)|+|\Tr(B_1A_1)|\Big)\norm{A_2^T-A_1^T} \\
        &\leq \Big( \norm{A_1-A_2}+\norm{B_1-B_2}\Big)R'^2+\Big(|\Tr(W_0)|+\norm{B_1}\norm{A_1}\Big)\norm{A_2^T-A_1^T} \\
        &\leq \Big( \norm{A_1-A_2}+\norm{B_1-B_2}\Big)R'^2+\Big(|\Tr(W_0)|+R'^2\Big)\norm{A_2^T-A_1^T} \\
        &= \norm{B_1-B_2}R'^2+\Big(|\Tr(W_0)|+2R'^2\Big)\norm{A_2-A_1} \\
        &\leq \norm{B_1-B_2}\Big(|\Tr(W_0)| + 2R'^2 \Big) + \Big(|\Tr(W_0)|+2R'^2\Big)\norm{A_2-A_1} \\
        &= \Big(|\Tr(W_0)| + 2R'^2 \Big)\Big(\norm{B_1-B_2}+\norm{A_1-A_2}  \Big) \\
        &\leq \sqrt{2}\Big(|\Tr(W_0)| + 2R'^2 \Big)\Big(\norm{B_1-B_2}^2+\norm{A_1-A_2}^2  \Big)^{1/2}
    \end{align}
    where the last line holds by Cauchy-Schwartz on $\mathbb{R}^2$. So we have
    \begin{align}
        \norm{\nabla_B g(B_1,A_1) - \nabla_B g(B_2, A_2)} &\leq \sqrt{2}\Big(|\Tr(W_0)|+2R'^2\Big)\bigg( \norm{B_1-B_2}^2+\norm{A_1-A_2}^2  \bigg)^{1/2} \label{eq giant inequality asdf}
    \end{align}
    Denote ${\dfrac{1}{\sqrt{2}}L_{R'}} := \sqrt{2}\Big(|\Tr(W_0)|+2R'^2\Big)$. Squaring \eqref{eq giant inequality asdf}, we get
    \begin{align}
        \norm{\nabla_B g(B_1,A_1) - \nabla_B g(B_2, A_2)}^2 &\leq {\frac{1}{2}L_{R'}^2}\bigg( \norm{B_1-B_2}^2+\norm{A_1-A_2}^2  \bigg) \label{eq big inequality nabla_B g}
    \end{align}
    An identical calculation for $\norm{\nabla_A g(B_1,A_1) - \nabla_A g(B_2, A_2)}$ gives
    \begin{align}
        \norm{\nabla_A g(B_1,A_1) - \nabla_A g(B_2, A_2)}^2 &\leq {\frac{1}{2}L_{R'}^2}\bigg( \norm{B_1-B_2}^2+\norm{A_1-A_2}^2  \bigg) \label{eq big inequality nabla_A g}
    \end{align}
    Together, (\ref{eq big inequality nabla_B g}--\ref{eq big inequality nabla_A g}) give 
    \begin{align}
        \norm{\nabla g(B_1,A_1) - \nabla g(B_2, A_2)}^2 &= \norm{\nabla_B g(B_1,A_1) - \nabla_B g(B_2, A_2)}^2 + \norm{\nabla_A g(B_1,A_1) - \nabla_A g(B_2, A_2)}^2 \\
        &\leq {L_{R'}^2}\bigg( \norm{B_1-B_2}^2+\norm{A_1-A_2}^2  \bigg) \\
        &= L_{R'}^2 \norm{ \Big( B_1, A_1 \Big) -  \Big( B_2, A_2 \Big) }^2
    \end{align}
    or simply
    \begin{align}
        \norm{\nabla g(B_1,A_1) - \nabla g(B_2, A_2)} \leq L_{R'} \norm{ \Big( B_1, A_1 \Big) -  \Big( B_2, A_2 \Big) }
    \end{align}
    for all $(B_1,A_1),(B_2,A_2) \in \mathcal{D}_{R'}$ Thus, $g$ is $L_{R'}$ Lipschitz smooth on $\mathcal{D}_{R'}$.
\end{lemma}


\onecolumn

\section{Solution to the Initial-Value Problem in (\ref{eq a(t) ODE 1}--\ref{eq a(t) ODE 2})} \label{app solution to a(t) ODE}

We wish to solve the initial-value problem
\begin{align}
    a''(t) &=  \frac{\Big( a'(t) \Big)^2}{a(t)} + 4a^3(t)-4ca^2(t) \label{eq main problem}
\end{align}
\begin{align}
    a(0) &= \Tr \Big( W_0 \Big) := c \\
    a'(0) &= -\Tr \Big( W_0 \Big)\norm{X_0}^2 := -c\norm{X_0}^2
\end{align}
for $t \in [0,T]$, where $T > 0$ is given in Appendix \ref{app learning dynamics for trace squared loss}.

Begin by letting $u(t) = a'(t)$. This gives $u'(t) = a''(t)$. By the Chain Rule, we find
\begin{align}
    a''(t) = \frac{d^2a}{dt^2} &= \frac{du}{da}\frac{da}{dt} = \frac{du}{da} \cdot u
\end{align}
Substituting for $a'(t)$ and $a''(t)$, (\ref{eq main problem}) becomes
\begin{align}
    u\frac{du}{da} &= \frac{u^2}{a} + 4a^3 - 4ca^2 \label{eq subbing u in}
\end{align}
Now let $y = u^2$. Differentiate both sides with respect to $a$. By the Chain Rule, we have
\begin{align}
    \frac{dy}{da} &= 2u\frac{du}{da} \\
    \frac{1}{2} \frac{dy}{da} &= u\frac{du}{da}
\end{align}
Then (\ref{eq subbing u in}) becomes
\begin{align}
    \frac{1}{2}\frac{dy}{da} &= \frac{y}{a}+4a^3-4ca^2 \\
    \frac{dy}{da} &= \frac{2y}{a}+8a^3-8ca^2 \\
    \frac{dy}{da} - \frac{2y}{a} &= 8a^3-8ca^2 \label{eq frac}
\end{align}
By Remark (\ref{rem Tr^2 nonzero during training}) in Appendix \ref{app learning dynamics for trace squared loss}, we have that $a(t) \neq 0$ during training. Since we are concerned only with the trajectory of $a(t)$ prior to convergence, we can divide both sides of (\ref{eq frac}) by $a^2$ to get
\begin{align}
    \frac{1}{a^2}\frac{dy}{da} - \frac{2}{a^3}y &= 8a-8c \\
    \frac{d}{da} \left[ \frac{y}{a^2} \right]  &= 8a-8c \\
    \int \frac{d}{da} \left[ \frac{y}{a^2} \right] \, da  &= \int 8a-8c \, da \\
    \frac{y}{a^2} &= 4a^2-8ca+\kappa_1 \\
    y &= a^2(4a^2 - 8ca + \kappa_1) \\
    u^2 &= a^2(4a^2 - 8ca + \kappa_1) \\
    u &= \pm a \sqrt{4a^2 - 8ca + \kappa_1} \\
    a' &= \pm a \sqrt{4a^2 - 8ca + \kappa_1} \label{eq opposite signed}
\end{align}
Note that the quantity on the right-hand side remains real:
\begin{remark}[Radicand is Non-negative] \label{rem radical is non negative}
    From (\ref{eq opposite signed}), we have
    \begin{align}
        a' &= \pm a \sqrt{4a^2 - 8ca + \kappa_1} \\
        \pm\frac{a'}{a} &= \sqrt{4a^2 - 8ca + \kappa_1} \\
        \left(\frac{a'}{a}\right)^2 &= 4a^2 - 8ca + \kappa_1
    \end{align}
    Thus, $4a^2 - 8ca + \kappa_1 \geq 0$.
\end{remark}
From (\ref{eq sick of naming equations}) in Appendix \ref{app learning dynamics for trace squared loss}, we have
\begin{align}
    a'(t) &= -a(t)\left( \norm{X}^2 + \norm{Y}^2 \right) \label{eq a' = -a ...}
\end{align}
Since $\left( \norm{X}^2 + \norm{Y}^2 \right) \geq 0$, we have that $a'(t)$ and $a(t)$ are opposite signed. Thus, (\ref{eq opposite signed}) becomes
\begin{align}
    a' &= -a \sqrt{4a^2 - 8ca + \kappa_1} \label{eq post opposite signed}
\end{align}
We utilize the initial conditions at this point to solve for $\kappa_1$:
\begin{align}
    a'(0) &= -a(0) \sqrt{4a^2(0) - 8ca(0) + \kappa_1} \\
    -\Tr \Big( W_0 \Big)\norm{X_0}^2 &= -\Tr \Big( W_0 \Big) \sqrt{4\Tr^2\Big( W_0 \Big) - 8c\Tr \Big( W_0 \Big) + \kappa_1} \\
    \norm{X_0}^2 &=  \sqrt{4c^2 - 8c^2 + \kappa_1} \\
    \norm{X_0}^4 &=  \kappa_1-4c^2 \\
    \kappa_1 &= \norm{X_0}^4+4c^2
\end{align}
So we get $\kappa_1 = \norm{X_0}^4+4c^2$. Note that $\kappa_1 > 0$.

Return to (\ref{eq post opposite signed}). We have
\begin{align}
    \frac{da}{dt} &= -a \sqrt{4a^2 - 8ca + \kappa_1} \\
    -\int \frac{da}{a\sqrt{4a^2 - 8ca + \kappa_1}} &= \int dt \\
    -\int \frac{da}{a\sqrt{4a^2 - 8ca + \kappa_1}} &= t+\kappa_2 \label{eq t+C_2}
\end{align}
Calculate the integral on left-hand side above. Begin by completing the square in the denominator:
\begin{align}
 \int \frac{da}{a\sqrt{4a^2 - 8ca + \kappa_1}} &= \int \frac{da}{a\sqrt{4(a-c)^2+\kappa_1-4c^2}} \label{eq original integral} \\
 &= \frac{1}{2} \int \frac{da}{a\sqrt{(a-c)^2+\frac{\kappa_1}{4}-c^2}} \\
 &= \frac{1}{2} \int \frac{-a \, da}{-a^2\sqrt{(a-c)^2+\frac{\kappa_1}{4}-c^2}}
\end{align}
Making the substitution:
\begin{align}
    v &= \frac{1}{a} \\
    dv &= -\frac{1}{a^2} da,
\end{align}
we find
\begin{align}
    \frac{1}{2} \int \frac{-a \, da}{-a^2\sqrt{(a-c)^2+\frac{\kappa_1}{4}-c^2}} &= -\frac{1}{2} \int \frac{dv}{v\sqrt{\left(\frac{1}{v}-c\right)^2+\frac{\kappa_1}{4}-c^2}} \\
    &= -\frac{1}{2} \int \pm \frac{dv}{\sqrt{\left(1-cv\right)^2+v^2\left(\frac{\kappa_1}{4}-c^2\right)}}
\end{align}
We insert the $\pm$ symbol into our integrand above to avoid presuming on the sign of $v$. To pull $v$ into the square root in the denominator above, we have to write $v = \pm \sqrt{v^2}$ since we have
\begin{align}
    v = +\sqrt{v^2} \qquad v \geq 0 \\
    v = -\sqrt{v^2} \qquad v < 0
\end{align}
We can then continue our calculation:
\begin{align}
  -\frac{1}{2} \int \pm \frac{dv}{\sqrt{\left(1-cv\right)^2+v^2\left(\frac{\kappa_1}{4}-c^2\right)}}  &= -\frac{1}{2} \int \pm \frac{dv}{\sqrt{1-2cv+c^2v^2-c^2v^2+\frac{v^2\kappa_1}{4}}} \\
    &= -\int \pm \frac{dv}{\sqrt{4-8cv+v^2\kappa_1}}
\end{align}
Complete the square in the denominator once again to find
\begin{align}
    -\int \pm \frac{dv}{\sqrt{4-8cv+v^2\kappa_1}} &= -\int \pm\frac{dv}{\sqrt{\kappa_1 \left( v-\frac{4c}{\kappa_1}\right)^2 + 4-\frac{16c^2}{\kappa_1}}} \\
    &= -\int \pm\frac{dv}{\sqrt{\left( \sqrt{\kappa_1}v-\frac{4c}{\sqrt{\kappa_1}}\right)^2 + 4-\frac{16c^2}{\kappa_1}}} \\
    &= -\int \pm\frac{dv}{\sqrt{\left(\frac{\kappa_1v-4c}{\sqrt{\kappa_1}}\right)^2 + 4-\frac{16c^2}{\kappa_1}}}
\end{align}
Now make the substitution
\begin{align}
    w &= \frac{\kappa_1v-4c}{\sqrt{\kappa_1}\sqrt{4-\frac{16c^2}{\kappa_1}}} \\
    dw &= \sqrt{\frac{\kappa_1}{4-\frac{16c^2}{\kappa_1}}}dv \\
    dv &= \frac{\sqrt{4-\frac{16c^2}{\kappa_1}}}{\sqrt{\kappa_1}}dw
\end{align}
\begin{remark}[Another Non-negative Radicand]\label{rem Another Non-negative Radicand}
    Note that
    \begin{align}
        4-\frac{16c^2}{\kappa_1} &= 4\left(1 - \frac{4c^2}{\kappa_1} \right) = 4\left( \frac{\kappa_1-4c^2}{\kappa_1}\right) = \frac{4}{\kappa_1}(\kappa_1-4c^2)
    \end{align}
    Since $\kappa_1 > 0$ for our initialization scheme, the left factor is positive. On the right,
    \begin{align}
        \kappa_1 - 4c^2 &= \norm{X_0}^4 + 4c^2 - 4c^2 = \norm{X_0}^4 > 0
    \end{align}
    Therefore, $4-\frac{16c^2}{\kappa_1} > 0$, and $w$ is real.
\end{remark}
Changing variables to $w$ and carefully tracking the sign of the integrand based on the sign of $v$ , we then have
\begin{align}
    -\int \pm \frac{dv}{\sqrt{\left(\frac{\kappa_1v-4c}{\sqrt{\kappa_1}}\right)^2 + 4-\frac{16c^2}{\kappa_1}}} &= -\int \pm \frac{dw}{\sqrt{\left(w\sqrt{4-\frac{16c^2}{\kappa_1}}\right)^2 + 4-\frac{16c^2}{\kappa_1}}} \cdot \frac{\sqrt{4-\frac{16c^2}{\kappa_1}}}{\sqrt{\kappa_1}} \\
    &= -\int \pm \frac{dw}{\sqrt{w^2\left({4-\frac{16c^2}{\kappa_1}}\right) + 4-\frac{16c^2}{\kappa_1}}} \cdot \frac{\sqrt{4-\frac{16c^2}{\kappa_1}}}{\sqrt{\kappa_1}} \\
    &= -\frac{1}{\sqrt{\kappa_1}}\int \pm \frac{dw}{\sqrt{w^2 + 1}} \\
\end{align}
Thus, we have arrived at
\begin{align}
    -\int \pm \frac{dv}{\sqrt{\left(\frac{\kappa_1v-4c}{\sqrt{\kappa_1}}\right)^2 + 4-\frac{16c^2}{\kappa_1}}} &= \begin{cases}
        -\dfrac{1}{\sqrt{\kappa_1}} \displaystyle \int + \dfrac{dw}{\sqrt{w^2 + 1}} & v \geq 0 \\
        -\dfrac{1}{\sqrt{\kappa_1}} \displaystyle \int - \dfrac{dw}{\sqrt{w^2 + 1}} & v < 0
    \end{cases}
\end{align}
Taking the anti-derivatives on the right, this becomes
\begin{align}
    -\int \pm \frac{dv}{\sqrt{\left(\frac{\kappa_1v-4c}{\sqrt{\kappa_1}}\right)^2 + 4-\frac{16c^2}{\kappa_1}}} &= \begin{cases}
        -\dfrac{1}{\sqrt{\kappa_1}}\arsinh(w) & v \geq 0 \\
        -\dfrac{1}{\sqrt{\kappa_1}} \arsinh(-w) & v < 0
    \end{cases}
\end{align}
Back-substitution then yields
\begin{align}
    -\int \pm \frac{dv}{\sqrt{\left(\frac{\kappa_1v-4c}{\sqrt{\kappa_1}}\right)^2 + 4-\frac{16c^2}{\kappa_1}}} &= \begin{cases}
        -\dfrac{1}{\sqrt{\kappa_1}}\arsinh\left(\dfrac{\kappa_1v-4c}{\sqrt{\kappa_1}\sqrt{4-\frac{16c^2}{\kappa_1}}}\right) & v \geq 0 \\
        -\dfrac{1}{\sqrt{\kappa_1}} \arsinh \left(-\dfrac{\kappa_1v-4c}{\sqrt{\kappa_1}\sqrt{4-\frac{16c^2}{\kappa_1}}} \right) & v < 0
    \end{cases}
\end{align}
Note that since $v = \dfrac{1}{a}$, the signs of $a$ and $v$ will be the same. Recalling the original integral in (\ref{eq original integral}), we finally have
\begin{align}
  \int \frac{da}{a\sqrt{4a^2 - 8ca + \kappa_1}}  &= \begin{cases}
        -\dfrac{1}{\sqrt{\kappa_1}}\arsinh\left(\dfrac{\frac{\kappa_1}{a}-4c}{\sqrt{\kappa_1}\sqrt{4-\frac{16c^2}{\kappa_1}}}\right) & a \geq 0 \\
        -\dfrac{1}{\sqrt{\kappa_1}} \arsinh \left(-\dfrac{\frac{\kappa_1}{a}-4c}{\sqrt{\kappa_1}\sqrt{4-\frac{16c^2}{\kappa_1}}} \right) & a < 0
    \end{cases}
\end{align}
Eq. (\ref{eq t+C_2}) then yields
\begin{align}
    t+\kappa_2 &= \begin{cases}
        \dfrac{1}{\sqrt{\kappa_1}}\arsinh\left(\dfrac{\frac{\kappa_1}{a}-4c}{\sqrt{\kappa_1}\sqrt{4-\frac{16c^2}{\kappa_1}}}\right) & a \geq 0 \\
        \dfrac{1}{\sqrt{\kappa_1}} \arsinh \left(-\dfrac{\frac{\kappa_1}{a}-4c}{\sqrt{\kappa_1}\sqrt{4-\frac{16c^2}{\kappa_1}}} \right) & a < 0
    \end{cases} \label{eq almost solved for a(t)}
\end{align}
or simply
\begin{align}
    t+\kappa_2 &= \dfrac{1}{\sqrt{\kappa_1}}\arsinh\left(\sgn (a)\dfrac{\frac{\kappa_1}{a}-4c}{\sqrt{\kappa_1}\sqrt{4-\frac{16c^2}{\kappa_1}}}\right)
\end{align}
Note that since $a(t)$ is continuous and nonzero and that $a(0) = c$, we will have $\sgn(a) = \sgn(c)$ on $[0,T]$. Thus we have
\begin{align}
    t+\kappa_2 &= \dfrac{1}{\sqrt{\kappa_1}}\arsinh\left(\sgn (c)\dfrac{\frac{\kappa_1}{a}-4c}{\sqrt{\kappa_1}\sqrt{4-\frac{16c^2}{\kappa_1}}}\right)
\end{align}
At this point, we solve for $\kappa_2$:
\begin{align}
    \kappa_2 &= \frac{1}{\sqrt{\kappa_1}}\arsinh\left(\sgn (c)\frac{\frac{\kappa_1}{a(0)}-4c}{\sqrt{\kappa_1}\sqrt{4-\frac{16c^2}{\kappa_1}}}\right) \\
    &= \frac{1}{\sqrt{\kappa_1}}\arsinh\left(\sgn (c)\frac{\frac{\kappa_1}{c}-4c}{\sqrt{\kappa_1}\sqrt{4-\frac{16c^2}{\kappa_1}}}\right) \\
    &= \frac{1}{\sqrt{\kappa_1}}\arsinh\left(\sgn (c)\frac{\frac{\kappa_1}{c}-4c}{\sqrt{4\kappa_1-16c^2}}\right) \\
    &= \frac{1}{\sqrt{\kappa_1}}\arsinh\left(\sgn (c)\frac{\frac{\kappa_1}{c}-4c}{\sqrt{4\norm{X_0}^4+16c^2-16c^2}}\right) \\
    &= \frac{1}{\sqrt{\kappa_1}}\arsinh\left(\sgn (c)\frac{\frac{\kappa_1}{c}-4c}{2\norm{X_0}^2}\right) \\
    &= \frac{1}{\sqrt{\kappa_1}}\arsinh\left(\sgn (c)\frac{\kappa_1-4c^2}{2c\norm{X_0}^2}\right) \\
    &= \frac{1}{\sqrt{\kappa_1}}\arsinh\left(\sgn (c)\frac{\norm{X_0}^4+4c^2-4c^2}{2c\norm{X_0}^2}\right) \\
    &= \frac{1}{\sqrt{\norm{X_0}^4+4c^2}}\arsinh\left(\sgn (c)\frac{\norm{X_0}^2}{2c}\right) \\
    \kappa_2 &= \frac{1}{\sqrt{\norm{X_0}^4+4c^2}}\arsinh\left(\frac{\norm{X_0}^2}{2|c|}\right)
\end{align}
With $\kappa_1$ and $\kappa_2$ defined, we can find a closed-form solution for $a(t)$. Start from (\ref{eq almost solved for a(t)}):
\begin{align}
    t+\kappa_2 &= \frac{1}{\sqrt{\kappa_1}}\arsinh\left(\sgn (c)\frac{\frac{\kappa_1}{a}-4c}{\sqrt{\kappa_1}\sqrt{4-\frac{16c^2}{\kappa_1}}}\right) \\
    \sqrt{\kappa_1}(t+\kappa_2) &= \arsinh\left(\sgn (c)\frac{\frac{\kappa_1}{a}-4c}{\sqrt{\kappa_1}\sqrt{4-\frac{16c^2}{\kappa_1}}}\right)  \\
    \sinh\left( \sqrt{\kappa_1}(t+\kappa_2) \right) &= \sgn (c)\frac{\frac{\kappa_1}{a}-4c}{\sqrt{\kappa_1}\sqrt{4-\frac{16c^2}{\kappa_1}}} \\
    \sinh\left( \sqrt{\kappa_1}(t+\kappa_2) \right) &= \sgn (c)\frac{\frac{\kappa_1}{a}-4c}{\sqrt{4\kappa_1-16c^2}} \\
    \sgn (c){\sqrt{4\kappa_1-16c^2}}\sinh\left( \sqrt{\kappa_1}(t+\kappa_2) \right) &= {\frac{\kappa_1}{a}-4c} \\
    \sgn (c){\sqrt{4\kappa_1-16c^2}}\sinh\left( \sqrt{\kappa_1}(t+\kappa_2) \right)+4c &= \frac{\kappa_1}{a} \\
    \sgn (c){\sqrt{4\norm{X_0}^4+16c^2-16c^2}}\sinh\left( \sqrt{\kappa_1}(t+\kappa_2) \right)+4c &= \frac{\kappa_1}{a} \\
    2\sgn (c)\norm{X_0}^2\sinh\left( \sqrt{\kappa_1}(t+\kappa_2) \right)+4c &= \frac{\kappa_1}{a}
\end{align}
and finally,
\begin{align}
    a(t) &= \frac{\norm{X_0}^4+4c^2}{2\sgn (c)\norm{X_0}^2\sinh\left( \sqrt{\kappa_1}(t+\kappa_2) \right)+4c} \\
   a(t) &= \sgn (c)\frac{\norm{X_0}^4+4c^2}{2\norm{X_0}^2\sinh\left( \sqrt{\kappa_1}(t+\kappa_2) \right)+4|c|}
\end{align}
where
\begin{align}
    c:&= \Tr(W_0)
\end{align}


\onecolumn

\section{Learning Dynamics for Trace-Squared Loss (Full-rank)} \label{app full rank learning dynamics for trace squared loss}

Let $W_0 \in \mathbb{R}^{n \times n}$ be a matrix of frozen pretraining weights. We wish to analyze the learning dynamics of the full-rank optimizer $W \in \mathbb{R}^{n \times n}$ produced by applying full-rank GD to the problem
\begin{align}
    \underset{\substack{W \in \mathbb{R}^{n \times n}}}{\min} \, \frac{1}{2}\Tr^2(W_0-W) \label{eq trace squared objective full rank}
\end{align}
The gradient for the loss function $f(W)$ above is given by
\begin{align}
   \nabla f(W) = -\Tr\left( W_0-W \right) I_n \label{eq gradient flow ODE trace squared full rank}
\end{align}
where $I_n \in \mathbb{R}^{n \times n}$ denotes the identity matrix. The gradient flow ODE describing the learning dynamics of the trace-squared objective in
\eqref{eq trace squared objective full rank} under full-rank gradient descent is given by
\begin{align}
    \frac{dU(t)}{dt} &= -\nabla f\big(U(t)\big)
    = \Tr\big(W_0-U(t)\big)\,I_n, \label{eq gradient flow ODE for full rank trace squared} \\
    U(0) &= Y_0X_0, \label{eq IC trace squared full rank}
\end{align}
where the solution $U:\mathbb{R}_0^+\to\mathbb{R}^{n\times n}$ denotes the continuous-time analogue of the full-rank GD iterates.
Throughout the calculation below, we allow a general initialization $Y_0\in\mathbb{R}^{n\times r}$ for completeness. However,
to facilitate direct comparison with the low-rank learning dynamics for the same objective
(see Appendix~\ref{app learning dynamics for trace squared loss}), we specialize to the choice
$Y_0=\mathbf{0}_{n\times r}$ at the conclusion of our analysis.

Taking the trace of both sides of (\ref{eq gradient flow ODE for full rank trace squared}--\ref{eq IC trace squared full rank}), we have
\begin{align}
\frac{d \Tr\Big( U(t) \Big)}{dt} &= \Tr\Big( W_0-U(t) \Big)\Tr\left( I_n \right) \\
    \frac{d \Tr\Big( U(t) \Big)}{dt} &= n \Tr\Big( W_0-U(t) \Big) \\
  \frac{d \Tr\Big( U(t) \Big)}{dt} &= n\Tr\left( W_0 \right) - n \Tr\Big( U(t) \Big) \\
  \Tr \Big( U(0) \Big) &= \Tr(Y_0X_0)
\end{align}
which is solved by
\begin{align}
    \Tr \Big( U(t) \Big) &= \Tr(W_0) - e^{-nt}\Tr\left( W_0-Y_0X_0 \right)
\end{align}
Plugging our closed-form expression for $\Tr\Big( U(t) \Big)$ into \eqref{eq gradient flow ODE for full rank trace squared}, we arrive at
\begin{align}
    \frac{dU(t)}{dt} &= \Tr\Big( W_0-U(t) \Big) I_n \\
    \frac{dU(t)}{dt} &= \left[ \Tr\left( W_0 \right) -\Tr\Big( U(t) \Big)\right] I_n \\
    \frac{dU(t)}{dt} &= \left[ \Tr\left( W_0 \right) -\Tr(W_0) + e^{-nt}\Tr\left( W_0-Y_0X_0 \right)\right] I_n \\
   \frac{dU(t)}{dt} &= e^{-nt}\Tr\left( W_0-Y_0X_0 \right)I_n
\end{align}
which, for our initial condition in \eqref{eq IC trace squared full rank}, is solved by
\begin{align}
     U(t) &= \frac{1-e^{-nt}}{n}\Tr \left( W_0-Y_0X_0 \right)I_n+Y_0X_0
\end{align}
For a quick sanity check, note that
\begin{align}
    U(0) &= Y_0X_0 \\
    \lim\limits_{t \to \infty} \Tr \Big(U(t)\Big) &= \Tr\left(\frac{1}{n}\Tr \left( W_0-Y_0X_0 \right)I_n+Y_0X_0 \right) \\
    &= \frac{1}{n}\Tr \left( W_0-Y_0X_0 \right)\Tr\left( I_n \right) + \Tr(Y_0X_0) \\
    &= \Tr \left( W_0-Y_0X_0 \right) + \Tr(Y_0X_0) \\
    &= \Tr(W_0)
\end{align}
which gives the result we would expect for the minimization problem in \eqref{eq trace squared objective full rank}.

Taking the limit of $U(t)$ as $t\to\infty$, the gradient flow in (\ref{eq gradient flow ODE for full rank trace squared}--\ref{eq IC trace squared full rank})
converges to the following global minimizer of \eqref{eq trace squared objective full rank}:
\begin{align}
    \lim_{t\to\infty} U(t)
    = \frac{1}{n}\Tr\!\left( W_0 - Y_0X_0 \right)I_n+Y_0X_0 \label{eq final full rank minimizer}
\end{align}
In particular, for the same initialization $Y_0=\mathbf{0}_{n\times r}$ used in our low-rank analysis
(Appendix~\ref{app learning dynamics for trace squared loss}),
the full-rank flow converges to
\begin{align}
    \lim_{t\to\infty} U(t)
    = \frac{1}{n}\Tr(W_0)\,I_n.
\end{align}

To demonstrate uniqueness of our solution $U(t)$ to the IVP in (\ref{eq gradient flow ODE for full rank trace squared}--\ref{eq IC trace squared full rank}), suppose there exists an additional solution $U^*(t)$ which satisfies (\ref{eq gradient flow ODE for full rank trace squared}--\ref{eq IC trace squared full rank}). We then have
\begin{align}
    \frac{dU^*(t)}{dt} &= \Tr\Big( W_0-U^*(t) \Big) I_n \\
    U^*(0) &= Y_0X_0
\end{align}
We can calculate
\begin{align}
    \frac{dU(t)}{dt}-\frac{dU^*(t)}{dt} = \frac{d}{dt}\Big( U(t) - U^*(t) \Big) &= \Tr\Big( W_0-U(t) \Big) I_n - \Tr\Big( W_0-U^*(t) \Big) I_n \\
    &= \Tr \Big( U^*(t)-U(t) \Big)I_n \\
    &= -\Tr \Big( U(t)-U^*(t) \Big)I_n
\end{align}
or simply
\begin{align}
    \frac{d}{dt}\Big( U(t) - U^*(t) \Big) &= -\Tr \Big( U(t)-U^*(t) \Big)I_n \label{eq U-U^*}
\end{align}
Note that $U(0)-U^*(0) = Y_0X_0-Y_0X_0 = \mathbf{0}_{n \times n}$.
Now define $S: \mathbb{R}_0^+ \to \mathbb{R}^{n \times n}$ via
\begin{align}
    S(t) &= U(t)-U^*(t) \\
    S(0) &= Y_0X_0-Y_0X_0 = \mathbf{0}_{n \times n}
\end{align}
Then \eqref{eq U-U^*} gives
\begin{align}
    \frac{d}{dt}\Big( S(t) \Big) &= -\Tr \Big( S(t) \Big)I_n \label{eq S ode}
\end{align}
Take the trace of both sides of \eqref{eq S ode} to find
\begin{align}
    \frac{d}{dt}\Tr\Big( S(t) \Big) &= -\Tr \Big( S(t) \Big)\Tr\left(I_n\right) \\
    \frac{d}{dt}\Tr\Big( S(t) \Big) &= -n\Tr \Big( S(t) \Big)
\end{align}
which is solved with the initial condition $\Tr\Big( S(0) \Big) = 0$ by
\begin{align}
    \Tr\Big( S(t) \Big) &= 0
\end{align}
Since the trace of $S(t)$ has only the trivial solution, \eqref{eq S ode} becomes
\begin{align}
    \frac{d}{dt}\Big( S(t) \Big) &= \mathbf{0}_{n \times n}
\end{align}
This implies that $S(t)$ is a constant equal to its initial condition. Thus,
\begin{align}
    S(t) &= \mathbf{0}_{n \times n}
\end{align}
for all $t \in [0,\infty)$. This implies that
\begin{align}
    U(t)=U^*(t)
\end{align}
and our solution $U(t)$ to the IVP in (\ref{eq gradient flow ODE for full rank trace squared}--\ref{eq IC trace squared full rank}) must be unique.


\onecolumn

\section{Low-Rank Approximation Error for the Trace-Squared Objective Relative to the Full-Rank Minimizer} \label{app approximation error}

In this work, we analyze both the full-rank and LoRA parameterizations of the trace-squared fine-tuning problem. 
Using gradient flow dynamics, we characterize the convergence behavior of classical full-rank gradient descent 
as well as low-rank–adapted gradient descent toward minimizers of their respective objectives:
\begin{align}
    \underset{W \in \mathbb{R}^{n \times n}}{\min} \; &\frac{1}{2}\Tr^2(W_0-W) \label{eq full rank minimization problem} \\
    \underset{\substack{B \in \mathbb{R}^{n \times r} \\ A \in \mathbb{R}^{r \times n}}}{\min} \; &\frac{1}{2}\Tr^2(W_0-BA). \label{eq low rank minimization problem}
\end{align}
We find that application of classical GD to the full-rank problem converges to the following minimizer of \eqref{eq full rank minimization problem}:
\begin{align}
   \lim_{t \to \infty}U(t) &= \frac{1}{n}\Tr(W_0)\,I_n \in \underset{W \in \mathbb{R}^{n \times n}}{\arg\min} \; \frac{1}{2}\Tr^2(W_0-W)
\end{align}
while application of low-rank adapted GD to the low-rank problem converges to the following minimizer of \eqref{eq low rank minimization problem}:
\begin{align}
    \lim_{t \to \infty}Y(t)X(t) &= \frac{\Tr(W_0)}{\norm{X_0}^2}X_0^TX_0 \in \underset{\substack{B \in \mathbb{R}^{n \times r} \\ A \in \mathbb{R}^{r \times n}}}{\arg\min} \; \frac{1}{2}\Tr^2(W_0-BA) 
\end{align}
Recall that, for both the low-rank and full-rank variants of our problem, we adopt the standard LoRA initialization scheme~\cite{Hu2022LoRA}, namely
$U(0)=Y(0)X(0)=Y_0X_0$, where $Y_0\in\mathbb{R}^{n\times r}$ is initialized as the zero matrix and
$X_0\in\mathbb{R}^{r\times n}$ has entries drawn i.i.d.\ from a centered Gaussian distribution with variance $\sigma^2$. Namely,
\begin{align}
    x_{ij}(0) \sim \mathcal{N}(0,\sigma^2). \label{eq random init}
\end{align}
where $x_{ij}$ are the individual elements of $X$.

Recall from Appendices~\ref{app learning dynamics for trace squared loss}
and~\ref{app full rank learning dynamics for trace squared loss} that both
the full-rank and low-rank gradient flows converge to global minimizers of the
trace-squared objective, achieving zero final loss. We thus examine only the approximation error between the resulting full-rank and low-rank minimizers. In particular, we calculate the final relative error:
\begin{align}
    \lim_{t \to \infty} \frac{\norm{Y(t)X(t)-U(t)}^2}{\norm{U(t)}^2}
\end{align}
The norm in the denominator is simply
\begin{align}
    \lim_{t \to \infty} \norm{U(t)}^2 &= \norm{\frac{1}{n}\Tr(W_0)\,I_n}^2 \\
    &= \frac{|\Tr(W_0)|^2}{n^2}\norm{I_n}^2 \\
    &= \frac{|\Tr(W_0)|^2}{n^2}\cdot n \\
    &= \frac{|\Tr(W_0)|^2}{n}
\end{align}
For the numerator, we find
\begin{align}
    \lim_{t \to \infty} \norm{Y(t)X(t)-U(t)}^2 &= \norm{\frac{\Tr(W_0)}{\norm{X_0}^2}X_0^TX_0-\frac{1}{n}\Tr(W_0)\,I_n}^2 \\
    &= \left|\Tr(W_0)\right|^2\norm{\frac{X_0^TX_0}{\norm{X_0}^2}-\frac{1}{n}\,I_n}^2
\end{align}
Calculate the Frobenius norm above:
\begin{align}
    \norm{\frac{X_0^TX_0}{\norm{X_0}^2}-\frac{1}{n}\,I_n}^2 &= \Tr \left( \left[ \frac{X_0^TX_0}{\norm{X_0}^2}-\frac{1}{n}\,I_n \right]^T \left[ \frac{X_0^TX_0}{\norm{X_0}^2}-\frac{1}{n}\,I_n \right]  \right) \\
    &= \Tr \left( \left[ \frac{X_0^TX_0}{\norm{X_0}^2}-\frac{1}{n}\,I_n \right] \left[ \frac{X_0^TX_0}{\norm{X_0}^2}-\frac{1}{n}\,I_n \right]  \right) \\
    &= \Tr \left( \frac{X_0^TX_0X_0^TX_0}{\norm{X_0}^4} - \frac{2}{n}\frac{X_0^TX_0}{\norm{X_0}^2} +\frac{1}{n^2}I_n \right) \\
    &= \Tr\left( \frac{X_0^TX_0X_0^TX_0}{\norm{X_0}^4} \right) - \frac{2}{n}\Tr \left( \frac{X_0^TX_0}{\norm{X_0}^2} \right) + \frac{1}{n^2}\Tr\left( I_n \right) \\
    &= \frac{1}{\norm{X_0}^4}\Tr\left( {X_0^TX_0X_0^TX_0} \right) - \frac{2}{n\norm{X_0}^2}\Tr \left( {X_0^TX_0} \right) + \frac{1}{n^2} \cdot n \\
    &= \frac{\norm{X_0^TX_0}^2}{\norm{X_0}^4} - \frac{2\norm{X_0}^2}{n\norm{X_0}^2} + \frac{1}{n} \\
    &= \frac{\norm{X_0^TX_0}^2}{\norm{X_0}^4} - \frac{1}{n}
\end{align}
We then have
\begin{align}
    \lim_{t \to \infty} \norm{Y(t)X(t)-U(t)}^2 &= \left|\Tr(W_0)\right|^2 \left( \frac{\norm{X_0^TX_0}^2}{\norm{X_0}^4} - \frac{1}{n} \right)
\end{align}
The square of the relative error between the low-rank and full-rank minimizers for arbitrary $X_0 \neq \mathbf{0}_{r \times n}$ is then
\begin{align}
    \lim_{t \to \infty} \frac{\norm{Y(t)X(t)-U(t)}^2}{\norm{U(t)}^2} &= \left|\Tr(W_0)\right|^2 \left( \frac{\norm{X_0^TX_0}^2}{\norm{X_0}^4} - \frac{1}{n} \right) \cdot \frac{n}{|\Tr(W_0)|^2} \\
    &= \frac{n\norm{X_0^TX_0}^2}{\norm{X_0}^4} - 1
\end{align}
Note that the expression above is non-negative since $\sqrt{n}\norm{X_0^TX_0} \geq \norm{X_0}^2$ by Cauchy-Schwarz applied to the eigenvalues of $X_0^TX_0$.
We then have the final relative error
\begin{align}
    \lim_{t \to \infty} \frac{\norm{Y(t)X(t)-U(t)}}{\norm{U(t)}} &= \sqrt{\frac{n\norm{X_0^TX_0}^2}{\norm{X_0}^4} - 1} \\
    &= \frac{\sqrt{n\norm{X_0^TX_0}^2-\norm{X_0}^4}}{\norm{X_0}^2}
\end{align}
To reiterate, the expression above for the relative error between the low-rank and full-rank gradient flow minimizers is valid for arbitrary $X_0 \neq \mathbf{0}_{r \times n}$. However, we can study the $\textit{expected}$ approximation error as a function of $r$ and $n$ in the context of the standard LoRA initialization scheme in \eqref{eq random init}. 

Initializing $X_0 \in \mathbbm{R}^{n \times r}$ as in \eqref{eq random init}, begin by calculating the expected value of the ratio
\begin{align}
    \frac{\norm{X_0^TX_0}^2}{\norm{X_0}^4} \label{eq expectation ratio}
\end{align}
Write
\begin{align}
    X_0 = rZ
\end{align}
where $r = \norm{X_0}$ and $Z = \dfrac{X_0}{\norm{X_0}}$. Then \eqref{eq expectation ratio} becomes
\begin{align}
    \frac{\norm{X_0^TX_0}^2}{\norm{X_0}^4} &= \frac{\norm{r^2Z^TZ}^2}{\norm{r}^4} = \frac{r^4\norm{Z^TZ}^2}{r^4} = \norm{Z^TZ}^2 = \sum_{a=1}^{r} \sum_{b=1}^{r} \sum_{i=1}^{n} \sum_{j=1}^{n}
z_{ia}\, z_{ib}\, z_{ja}\, z_{jb} \label{eq lots of z}
\end{align}
To calculate the expectation of the expression above, we then need the expectation of the individual scalar entries of $Z$.
Let $\mathbf{z} \in \mathbb{R}^{nr}$ denote the vector of stacked entries of $Z$, and let $\mathbf{x_0} \in \mathbbm{R}^{nr}$ denote the vector of stacked entries of $X_0$. Use $\norm{\cdot}_2$ to represent the vector Euclidean norm. Then
\begin{align}
    \mathbf{z} &= \frac{\mathbf{x_0}}{\norm{\mathbf{x_0}}_2}
\end{align}
We have that
\begin{align}
    \mathbf{x_0} \sim \mathcal{N}\left( 0, \sigma^2I_{nr} \right)
\end{align}
By the properties of any scalar random variable $\xi$ that
\begin{align}
    \Var(\sigma\xi) &= \sigma^2\Var(\xi) \\
    \mathbb{E}\left[ \sigma \xi  \right] &= \sigma\mathbb{E}\left[ \xi  \right]
\end{align}
we can write
\begin{align}
    \mathbf{x_0} &= \sigma \mathbf{g}
\end{align}
for $\mathbf{g} \in \mathbb{R}^{nr}$, where
\begin{align}
    \mathbf{g} \sim \mathcal{N}\left( 0, I_{nr} \right)
\end{align}
It is known that $\dfrac{\mathbf{g}}{\norm{\mathbf{g}}_2}$ is uniformly distributed on the unit sphere in $\mathbb{R}^{nr}$~\cite{vershynin2018highdim}. We then immediately have
\begin{align}
    \mathbf{z} &= \frac{\mathbf{x_0}}{\norm{\mathbf{x_0}}_2} = \frac{\sigma \mathbf{g}}{\sigma\norm{\mathbf{g}}_2} = \frac{\mathbf{g}}{\norm{\mathbf{g}}_2}
\end{align}
and therefore $\mathbf{z}$ is uniformly distributed on the unit sphere. Now denote $\rho := \norm{\mathbf{g}}_2$. We know that $\mathbf{z}$ and $\rho$ are independent~\cite{vershynin2018highdim}. Choose arbitrary indices $i,j \in \{1,\dots,nr\}$. For the corresponding scalar components of $\mathbf{z}$ and $\mathbf{g}$, we have
\begin{align}
    z_i^2 &= \frac{g_i^2}{\rho^2} \qquad \text{ and } \qquad z_i^2z_j^2 = \frac{g_i^2g_j^2}{\rho^4}
\end{align}
or
\begin{align}
    \rho^2 z_i^2 &= g_i^2 \qquad \text{ and } \qquad \rho^4 z_i^2z_j^2 = g_i^2g_j^2
\end{align}
Taking advantage of the independence of $\mathbf{z}$ and $\rho$, we find
\begin{align}
    \mathbb{E}\left[\rho^2\right] \mathbb{E}\left[z_i^2\right] &= \mathbb{E}\left[g_i^2\right] \qquad \text{ and } \qquad \mathbb{E}\left[\rho^4\right] \mathbb{E}\left[z_i^2z_j^2\right] = \mathbb{E}\left[g_i^2g_j^2\right]
\end{align}
We can thus calculate
\begin{align}
    \mathbb{E}\left[z_i^2\right] &= \frac{ \mathbb{E}\left[g_i^2\right]}{\mathbb{E}\left[\rho^2\right]} \qquad \text{ and } \qquad \mathbb{E}\left[z_i^2z_j^2\right] = \frac{\mathbb{E}\left[g_i^2g_j^2\right]}{\mathbb{E}\left[\rho^4\right]} \label{eq lots of expectations}
\end{align}
By standard results for Gaussian normal distributions, we have
\begin{align}
    \mathbb{E}\left[g_i^2\right] &= 1 \\
    \mathbb{E}\left[g_i^2g_j^2\right] &= \begin{cases}
        1 & i \neq j \\
        3 & i = j
    \end{cases}
\end{align}
Since $\mathbf{g} \in \mathbb{R}^{nr}$ and $\rho = \norm{\mathbf{g}}_2$, we know that
\begin{align}
    \rho^2 &= \norm{\mathbf{g}}_2^2 = \sum\limits_{i = 1}^{nr} g_i^2
\end{align}
Therefore, $\rho^2$ follows a chi-squared distribution with $nr$ degrees of freedom, or $\rho^2 \sim \chi^2_{nr}$.
By standard results for first and second moments of chi-squared distributions, we have
\begin{align}
    \mathbb{E}\left[ \rho^2 \right] &= nr \\
    \mathbb{E}\left[ \rho^4 \right] &= nr(nr+2)
\end{align}
Putting these results together, \eqref{eq lots of expectations} becomes
\begin{align}
    \mathbb{E}\left[z_i^2\right] &= \frac{1}{nr} \qquad \text{ and } \qquad \mathbb{E}\left[z_i^2z_j^2\right] = \begin{cases}
        \dfrac{1}{nr(nr+2)} & i \neq j \\
        \dfrac{3}{nr(nr+2)} & i = j
    \end{cases} \label{eq second and fourth moments of z}
\end{align}
Thus we have the second and fourth moments of the scalar components of $\mathbf{z}$, or, equivalently, the scalar compenents of $Z$.

We are now ready to compute the expectation of $\dfrac{\norm{X_0^TX_0}^2}{\norm{X_0}^4}$. Taking the expectation of \eqref{eq lots of z} gives
\begin{align}
    \mathbb{E}\left[ \frac{\norm{X_0^TX_0}^2}{\norm{X_0}^4} \right] &= \sum_{a=1}^{r} \sum_{b=1}^{r} \sum_{i=1}^{n} \sum_{j=1}^{n}\mathbb{E}\left[ z_{ia}\, z_{ib}\, z_{ja}\, z_{jb} \right] \label{eq big sum}
\end{align}
The expectation
\[
\mathbb{E}\!\left[ z_{ia}\, z_{ib}\, z_{ja}\, z_{jb} \right]
\]
depends only on whether the row indices $a,b$ and the column indices $i,j$ coincide, and hence takes one of four possible values corresponding to the cases where $i=j$ or $i\neq j$ and $a=b$ or $a\neq b$:
\begin{enumerate}[label=\Alph*.]
    \item $a=b$ and $i=j$. In this case, all four factors coincide, and by
    \eqref{eq second and fourth moments of z},
    \[
    \mathbb{E}\!\left[ z_{ia}\, z_{ib}\, z_{ja}\, z_{jb} \right]
    =
    \mathbb{E}\!\left[ z_{ia}^4 \right]
    =
    \frac{3}{nr(nr+2)}.
    \]

    \item $a=b$ and $i\neq j$. In this case,
    \[
    \mathbb{E}\!\left[ z_{ia}\, z_{ib}\, z_{ja}\, z_{jb} \right]
    =
    \mathbb{E}\!\left[ z_{ia}^2 z_{ja}^2 \right]
    =\frac{1}{nr(nr+2)},
    \]
    which is again given by \eqref{eq second and fourth moments of z}.

    \item $a\neq b$ and $i=j$. In this case,
    \[
    \mathbb{E}\!\left[ z_{ia}\, z_{ib}\, z_{ja}\, z_{jb} \right]
    =
    \mathbb{E}\!\left[ z_{ia}^2 z_{ib}^2 \right]
    = \frac{1}{nr(nr+2)}.
    \]

    \item $a\neq b$ and $i\neq j$. In this case, the four indices are distinct, and, by symmetry on the unit sphere,
    \[
    \mathbb{E}\!\left[ z_{ia}\, z_{ib}\, z_{ja}\, z_{jb} \right] = 0.
    \]
    See Appendix~\ref{app vanishing expectation} for an explanation of this case.
\end{enumerate}

Note that, for the sum in \eqref{eq big sum}, we will have $nr$ terms fall under case A ($r$ choices for $a$ and $n$ choices for $i$). We will have $nr(n-1)$ terms satisfy case B ($r$ choices for $a$ and $n(n-1)$ choices for $i$). For case C, we will have $nr(r-1)$ terms ($n$ choices for $i$ and $r(r-1)$ choices for $a$). Terms for case D contribute nothing to the expectation value and do not need to be counted. We finally calculate our sum:
\begin{align}
    \sum_{a=1}^{r} \sum_{b=1}^{r} \sum_{i=1}^{n} \sum_{j=1}^{n}\mathbb{E}\left[ z_{ia}\, z_{ib}\, z_{ja}\, z_{jb} \right] &= \frac{3nr}{nr(nr+2)} + \frac{nr(n-1)}{nr(nr+2)} + \frac{nr(r-1)}{nr(nr+2)} \\
    &= \frac{3}{(nr+2)} + \frac{(n-1)}{(nr+2)} + \frac{(r-1)}{(nr+2)} \\
    &= \frac{n+r+1}{nr+2}
\end{align}
So we arrive at
\begin{align}
    \mathbb{E}\left[ \frac{\norm{X_0^TX_0}^2}{\norm{X_0}^4} \right] &= \frac{n+r+1}{nr+2}
\end{align}
The expectation of the \textit{square} of the final relative error between the rank-$r$ and full-rank solutions is then
\begin{align}
    \mathbb{E}\left[\lim_{t \to \infty} \frac{\norm{Y(t)X(t)-U(t)}^2}{\norm{U(t)}^2}\right] &= \mathbb{E}\left[\frac{n\norm{X_0^TX_0}^2}{\norm{X_0}^4} - 1\right] \\
    &= n\mathbb{E}\left[ \frac{\norm{X_0^TX_0}^2}{\norm{X_0}^4} \right] - 1 \\
    &= n \cdot \frac{n+r+1}{nr+2} - 1 \\
    &= \frac{n^2+n-2}{nr+2} \label{eq expectation of relative error squared}
\end{align}
From our expectation for the square of the final relative error, we can use Jensen's inequality to find an upper bound on the expectation of the final relative error itself. Define
\begin{align}
    R^2 := \lim_{t \to \infty} \frac{\norm{Y(t)X(t)-U(t)}^2}{\norm{U(t)}^2}.
\end{align}
Note that \(R^2 \ge 0\) and $\mathbb{E}[R^2]$ is finite by \eqref{eq expectation of relative error squared}. Recall that the square root function \(\phi(x) = \sqrt{x}\) is concave on \([0,\infty)\). By Jensen's inequality~\cite{dekking2005modern}, we have
\begin{align}
    \mathbb{E}\left[ \sqrt{R^2} \right] \leq \sqrt{\mathbb{E}\left[ R^2 \right]}
\end{align}
Combining this with \eqref{eq expectation of relative error squared}, we arrive at
\begin{align}
    \mathbb{E}\left[\lim_{t \to \infty} \frac{\norm{Y(t)X(t)-U(t)}}{\norm{U(t)}}\right]
    \leq \sqrt{\frac{n^2+n-2}{nr+2}}
\end{align}


\onecolumn

\section{Vanishing Fourth-Order Moment in Appendix \ref{app approximation error}} \label{app vanishing expectation}

Consider the random vector $\mathbf{z} \in \mathbb{R}^{nr}$ which is uniformly distributed on the unit sphere in $\mathbb{R}^{nr}$. Let
\[
S^{nr-1} = \{ \mathbf{v} \in \mathbb{R}^{nr} : \|\mathbf{v}\|_2 = 1 \}
\]
denote the unit sphere, and let $|S^{nr-1}|$ denote its surface area. Throughout this section, integrals over $S^{nr-1}$ are taken with respect to its surface measure $dS$.

The uniform distribution on $S^{nr-1}$ is then given by the probability measure
\[
d\mathbb{P}(\mathbf{z}) = \frac{1}{|S^{nr-1}|}\, dS(\mathbf{z}),
\]
which is normalized since
\[
\frac{1}{|S^{nr-1}|} \int_{S^{nr-1}} dS(\mathbf{z}) = 1.
\]

Let $z_i, z_j, z_k, z_l$ denote any four distinct scalar entries of the random vector $\mathbf{z}$. We wish to show that
\begin{align}
    \mathbb{E} \left[ z_i z_j z_k z_l \right] &= 0 \label{eq vanishing 4th order moment}
\end{align}
We can calculate \eqref{eq vanishing 4th order moment} directly by integrating over $S^{nr-1}$. We have
\begin{align}
    \mathbb{E} \left[ z_i z_j z_k z_l \right] &= \int_{S^{nr-1}} z_i z_j z_k z_l \, d\mathbb{P}(\mathbf{z}) \\
    &= \frac{1}{|S^{nr-1}|} \int_{S^{nr-1}} z_i z_j z_k z_l \, dS(\mathbf{z}) \label{eq expectation integral}
\end{align}
Note that the unit sphere $S^{nr-1}$ is invariant under coordinate sign flips. In particular, let $T:\mathbb{R}^{nr}\to\mathbb{R}^{nr}$ denote the orthogonal transformation that flips the sign of the $i$th coordinate and leaves all others unchanged. Then $T(S^{nr-1})=S^{nr-1}$, and the surface measure $dS(\mathbf{z})$ is invariant under $T$. Under this transformation,
\[
z_i z_j z_k z_l \;\mapsto\; (-z_i) z_j z_k z_l = -\, z_i z_j z_k z_l.
\]
Therefore,
\[
\int_{S^{nr-1}} z_i z_j z_k z_l \, dS(\mathbf{z})
=
-\int_{S^{nr-1}} z_i z_j z_k z_l \, dS(\mathbf{z}),
\]
which implies that the integral vanishes.


\onecolumn

\section{Learning Dynamics for Low-rank Matrix Approximation} \label{app learning dynamics for SSE}

Let $W_0 \in \mathbb{R}^{n \times m}$ be a matrix of frozen pretraining weights. We wish to analyze the learning dynamics of the low-rank optimizer $BA \in \mathbb{R}^{n \times m}$ produced by applying LoRA to the finetuning problem
\begin{align}
    \underset{\substack{B \in \mathbb{R}^{n \times r} \\ A \in \mathbb{R}^{r \times m}}}{\min} \, \frac{1}{2}\norm{W_0-BA}^2 \label{eq sse objective}
\end{align}
where $r << \min(n,m)$. The partial gradients for our objective $h(B,A) = \frac{1}{2}\norm{W_0-BA}^2$ are given by
\begin{align}
    \nabla_A h(B,A) &= -B^T(W_0-BA) \label{eq partial g A sse} \\
    \nabla_B h(B,A) &= -(W_0-BA)A^T \label{eq partial g B sse}
\end{align}
\begin{assumption}[Bounded Domain] \label{ass boundedness of iterates for Tr^2 sse}
    We optimize (\ref{eq sse objective}) over the subspace $\mathcal{D}_{R'} \subseteq \Theta$, where $R' > 0$ is some finite number. In other words, we assume that there exists some $R' > 0$ such that the norms on both $B$ and $A$ remain bounded above by $R'$ during training.
\end{assumption}
In practical settings, Assumption 1.1 is automatically enforced by computational memory constraints (e.g., finite-precision arithmetic and fixed-parameter storage). The assumption above ensures that the objective gradient remains bounded during training:
\begin{remark}[Boundedness of Gradient for Squared Frobenius Loss] \label{rem boundedness of gradient SSE}
    The objective gradient $\nabla h(B,A)$ remains bounded above during training. That is, for all $B,A \in \mathcal{D}_{R'}$, we have
    \begin{align}
        \norm{\nabla h(B,A)}^2 &= \norm{\nabla_B h(B,A)}^2 + \norm{\nabla_A h(B,A)}^2 \\
        &= \norm{(W_0-BA)A^T}^2 + \norm{B^T(W_0-BA)}^2 \\
        &\leq \norm{W_0-BA}^2\norm{A}^2 + \norm{W_0-BA}^2\norm{B}^2 \\
        &= \norm{W_0-BA}^2 \Big( \norm{A}^2+\norm{B}^2 \Big) \\
        &\leq 2R'^{2}\norm{W_0-BA}^2 \\
        &\leq 2R'^2\left( \norm{W_0}+\norm{BA} \right)^2 \\
        &\leq 2R'^2\left( \norm{W_0}+\norm{B}\norm{A} \right)^2 \\
        &\leq 2R'^2\left( \norm{W_0}+R'^2 \right)^2
    \end{align}
    So $\norm{\nabla h(B,A)} \leq \sqrt{2}R'\left( \norm{W_0} + R'^2 \right)$ throughout training.
\end{remark}
We then have Lipschitz smoothness during training:
\begin{lemma}[Lipschitz Smoothness for Squared Frobenius Objective] \label{lem Lipschitz smoothness SSE}
    Our objective gradient $\nabla h$ is Lipschitz smooth in our training domain. Namely, there exists $L_{R'} > 0$ such that, for any $(B_1,A_1) , (B_2,A_2) \in \mathcal{D}_{R'}$, we have
    \begin{align}
        \norm{\nabla h(B_1,A_1) - \nabla h(B_2, A_2)} \leq L_{R'} \norm{ \Big( B_1, A_1 \Big) -  \Big( B_2, A_2 \Big) }
    \end{align}
    Proof of this lemma can be found in Appendix \ref{app Lipschitz smoothness SSE}.
\end{lemma}

Having shown that the squared Frobenius objective in (\ref{eq sse objective}) satisfies Assumptions \ref{ass boundedness of iterates}, \ref{ass boundedness of gradient}, and \ref{ass Lipschitz continuity} from Appendix \ref{app LoRA GF Proof}, the ODEs describing the learning dynamics of $h(B,A)$ under LoRA are given by (see appendix \ref{app LoRA GF Proof})
\begin{align}
    \frac{dY(t)}{dt} &= \left( W_0-YX \right)X^T \label{eq LoRA GF ODE Y SSE} \\
    \frac{dX(t)}{dt} &= Y^T\left( W_0-YX \right) \label{eq LoRA GF ODE X SSE}
\end{align}
for any $t \in [0,T]$, where $T > 0$ is arbitrary.

Unlike the trace-squared objective in Appendix~\ref{app learning dynamics for trace squared loss}, we are not aware of a closed-form solution to the ODE above under conventional element-wise initialization, i.e., when all $rm$ entries of $X_0$ are drawn independently from a normal distribution. We will show, however, that the solution to (\ref{eq LoRA GF ODE Y SSE}--\ref{eq LoRA GF ODE X SSE}) converges to the theoretical rank-$r$ minimizer~\cite{EckartYoung1936} of \eqref{eq sse objective} when utilizing the spectral initialization schema used in previous work~\cite{xu2025understanding}.

Recall that $W_0$ is a known matrix of prefrozen weights. We begin by assuming that the pretrained weight matrix $W_0$ has rank strictly greater than $r$:
\begin{assumption}[Rank of $W_0$] \label{ass rank of W_0}
Let $k = \rank(W_0)$. We assume that
\begin{align}
    r < k \le \min(n,m).
\end{align}
\end{assumption}
Note that violation of Assumption~\ref{ass rank of W_0} renders \eqref{eq sse objective} outside the scope of LoRA. 

Denote the singular value decomposition (SVD) of $W_0$ by
\begin{align}
    W_0 = U \Sigma_0 V^T, \label{eq SVD of W_0}
\end{align}
where $\Sigma_0 \in \mathbb{R}^{n \times m}$ is the rectangular diagonal matrix whose main diagonal entries are the singular values of $W_0$ (in non-increasing order). Here, $U \in \mathbb{R}^{n \times n}$ and $V \in \mathbb{R}^{m \times m}$ are the left and right singular vector matrices of $W_0$, respectively. Recall that $U$ and $V$ are orthogonal.

For spectral initialization of $X_0$ and $Y_0$, define the rotated variables $\tilde{Y} \in \mathbb{R}^{n \times r}$ and $\tilde{X} \in \mathbb{R}^{r \times m}$ by
\begin{align}
    \tilde{Y} &= U^T Y, \\
    \tilde{X} &= X V,
\end{align}
so that
\begin{align}
    Y &= U\tilde{Y}, \label{eq Y sse} \\
    X &= \tilde{X}V^T. \label{eq X sse}
\end{align}
To ensure $Y_0X_0=\mathbf{0}_{n \times m}$, we initialize $\tilde{Y}_0=\mathbf{0}_{n \times r}$ and draw the $r$ diagonal entries of $\tilde{X}_0$ from a centered Gaussian distribution:
\begin{align}
    \tilde{x}_{ii}(0) \sim \mathcal{N}(0,\sigma^2), \qquad i=1,\dots,r,
\end{align}
where $\tilde{x}_{ii}$ denotes the $(i,i)$ entry of $\tilde{X}$. The off-diagonal elements of $\tilde{X}$ are initialized at zero.

With the transformations in \eqref{eq SVD of W_0} and (\ref{eq Y sse}--\ref{eq X sse}) in mind, our ODEs in (\ref{eq LoRA GF ODE Y SSE}--\ref{eq LoRA GF ODE X SSE}) become
\begin{align}
    U\frac{d\tilde{Y}(t)}{dt} &= \left( U \Sigma_0 V^T-U\tilde{Y}\tilde{X}V^T \right)V\tilde{X}^T \label{eq almost to Y spectral init ODE SSE} \\
    \frac{d\tilde{X}(t)}{dt}V^T &= \tilde{Y}^TU^T\left( U \Sigma_0 V^T-U\tilde{Y}\tilde{X}V^T \right) \label{eq almost to X spectral init ODE SSE}
\end{align}
Note that $U^TU$ and $V^TV$ both give the identity matrix since $U$ and $V$ are orthogonal. Left-multiplying $U^T$ on both sides of \eqref{eq almost to Y spectral init ODE SSE} and right-multiplying $V$ on both sides of \eqref{eq almost to X spectral init ODE SSE}, our ODEs become
\begin{align}
    \frac{d\tilde{Y}(t)}{dt} &= \left(\Sigma_0 - \tilde{Y}\tilde{X}\right)\tilde{X}^T \label{eq tilde Y SSE} \\
    \frac{d\tilde{X}(t)}{dt} &=\tilde{Y}^T \left(\Sigma_0 - \tilde{Y}\tilde{X}\right) \label{eq tilde X SSE}
\end{align}
As a consequence of our initialization scheme for $\tilde{X}_0$ and $\tilde{Y}_0$, all matrices in our transformed ODEs above remain diagonal for all $t \in [0,T]$, implying that the scalar dynamics for the $i$th diagonal element of $\tilde{Y}$ and $\tilde{X}$ decouple to
\begin{align}
    \frac{d\tilde{y}_{ii}}{dt} &= \left( s_{0,i}-\tilde{y}_{ii}\tilde{x}_{ii} \right)\tilde{x}_{ii} \label{eq decoupled scalar ODE y ii SSE the first one} \\
    \frac{d\tilde{x}_{ii}}{dt} &= \left( s_{0,i}-\tilde{y}_{ii}\tilde{x}_{ii} \right)\tilde{y}_{ii} \label{eq decoupled scalar ODE x ii SSE the first one} \\
    \tilde{y}_{ii}(0) &= 0 \label{eq y_ii initial SSE} \\
    \tilde{x}_{ii}(0) &\sim \mathcal{N}(0,\sigma^2) \label{eq x_ii initial SSE}
\end{align}
where $s_{0,i}$ denotes the $i$th largest singular value of $W_0$. Since $\tilde{Y}$ and $\tilde{X}$ each contain $r$ diagonal elements, (\ref{eq decoupled scalar ODE y ii SSE}--\ref{eq x_ii initial SSE}) give the learning dynamics for the first $r$ diagonals of $\tilde{Y}(t)\tilde{X}(t)$. Note that the dimensions of $\tilde{Y}$ and $\tilde{X}$ ensure that the remaining diagonal values of $\tilde{Y}(t)\tilde{X}(t)$ remain zero for all $t$.

\begin{assumption}[Nonzero Initialization for $\tilde{x}_{ii}$]
    We assume $\tilde{x}_{ii}(0) \neq 0$ for all $i \in \{1,\ldots,r\}$, which holds almost surely under the Gaussian initialization in \eqref{eq x_ii initial SSE}. \label{ass nonzero spectral initialization}
\end{assumption}
Note that violation of Assumption \ref{ass nonzero spectral initialization} above initializes $\Big(\tilde{x}_{ii}(t), \tilde{y}_{ii}(t)\Big)$ at the saddle point $(0,0)$ of the dynamical system in (\ref{eq decoupled scalar ODE y ii SSE}--\ref{eq x_ii initial SSE}). In this case, the $i$th diagonal of $\tilde{Y}(t)\tilde{X}(t)$ will fail to converge to $s_{i,0}$, and our LoRA gradient flow will not converge to the theoretical rank-$r$ minimizer of \eqref{eq sse objective}.

Since (\ref{eq decoupled scalar ODE y ii SSE}--\ref{eq x_ii initial SSE}) hold for all $i \in \{ 1, \ldots, r \}$, it suffices to solve the system for a single generic index $i$. For convenience, drop the indices in our notation and denote
\begin{align}
    y &:= \tilde{y}_{ii} \\
    x &:= \tilde{x}_{ii} \\
    s_0 &:= s_{0,i} 
\end{align}
We then wish to solve the IVP
\begin{align}
    \frac{d{y}}{dt} &= \left( s_0-{y}{x} \right){x} \label{eq decoupled scalar ODE y ii SSE} \\
    \frac{d{x}}{dt} &= \left( s_0-{y}{x} \right){y} \label{eq decoupled scalar ODE x ii SSE} \\
    {y}(0) &= 0 \\
    {x}(0) &\sim \mathcal{N}(0,\sigma^2) \label{eq scalar x init sse}
\end{align}
Conveniently, we already implicitly have the solution to the IVP above in Appendix~\ref{app learning dynamics for trace squared loss}. There we show that the solution to the matrix-valued IVP:
\begin{align}
    \frac{dY(t)}{dt} &= \Tr(W_0-YX)X^T \label{eq matrix valued version Y SSE} \\
    \frac{dX(t)}{dt} &= \Tr(W_0-YX)Y^T \\
    Y(0) &= \mathbf{0}_{n \times r} \\
    X(0) &= X_0 \label{eq matrix valued init X SSE}
\end{align}
is given by
\begin{align}
    Y(t) &= q(t)\Tr\!\left(W_0\right)X_0^T \\
    X(t) &= p(t)X_0,
\end{align}
where $p(t)$ and $q(t)$ are given in Appendix~\ref{app learning dynamics for trace squared loss}.

Recall that the above matrix-valued IVP is solved in Appendix~\ref{app learning dynamics for trace squared loss} for any $X_0 \in \mathbb{R}^{m \times r}$ such that $\norm{X_0} \neq 0$, as well as any $W_0$ such that $\Tr(W_0) \neq 0$. While in practice LoRA uses $r << \min(n,m)$ for parameter efficiency, we make no assumption in Appendix~\ref{app learning dynamics for trace squared loss} on the dimensions of $Y$, $X$, and $W_0$ when finding the solution of (\ref{eq matrix valued version Y SSE}--\ref{eq matrix valued init X SSE}). Observe that when $n=m=r=1$, the trace operator is trivial, and $\Tr(W_0 - YX) = w_0 - yx$. Importantly, Appendix~\ref{app learning dynamics for trace squared loss} then presents the resulting \textit{scalar} IVP
\begin{align}
    \frac{dy(t)}{dt} &= (w_0-yx)x \\
    \frac{dx(t)}{dt} &= (w_0-yx)y \\
    y(0) &= 0 \\
    x(0) &= x_0
\end{align}
along with the solution
\begin{align}
    y(t) &= q_{s}(t)w_0x_0 \\
    x(t) &= p_{s}(t)x_0
\end{align}
for arbitrary nonzero $w_0$ and $x_0$. Since the first $r$ singular values of $W_0$ are implicitly nonzero based on Assumption~\ref{ass rank of W_0}, we can set $w_0 = s_0$ to arrive at the solution of (\ref{eq decoupled scalar ODE y ii SSE}--\ref{eq scalar x init sse}):
\begin{align}
    y(t) &= q_{s}(t)s_0x_0 \\
    x(t) &= p_{s}(t)x_0
\end{align}
Note that, for the IVP in (\ref{eq decoupled scalar ODE y ii SSE}--\ref{eq scalar x init sse}) under Assumption~\ref{ass nonzero spectral initialization}, $x_0$ is a nonzero random variable drawn from $\mathcal{N}(0,\sigma^2)$. The scalar equivalents of $q(t)$ and $p(t)$, which we denote $q_s(t)$ and $p_s(t)$, are given by
\begin{align}
    p_s(t) &:= \frac{1}{\sqrt{2x_0^2\sinh\!\left(\sqrt{\xi_1}(t+\xi_2)\right)+4s_0}}
\Bigg[
\frac{x_0^2}{\sqrt{s_0}}
\sinh\!\left(\frac{\sqrt{\xi_1}}{2}t\right)
+
\frac{\sqrt{x_0^4+4s_0^2}}{\sqrt{s_0}}
\cosh\!\left(\frac{\sqrt{\xi_1}}{2}t\right)\Bigg] \\
q_s(t) &:= \frac{1}{\sqrt{2x_0^2\sinh\!\left(\sqrt{\xi_1}(t+\xi_2)\right)+4s_0}}\cdot
\frac{2}
{\sqrt{s_0}}
\,
\sinh\!\left(\frac{\sqrt{\xi_1}}{2}t\right)
\end{align}
where $\xi_1$ and $\xi_2$ are equal to
\begin{align}
    \xi_1 &= {x_0}^4 +4s_0^2 \\
    \xi_2 &= \frac{1}{\sqrt{{x_0}^4+4s_0^2}}\arsinh\left(\frac{{x_0}^2}{2s_0}\right)
\end{align}
Readers of Appendix~\ref{app learning dynamics for trace squared loss} should note we drop any absolute value bars above on $s_0$ since $W_0$'s singular values are nonnegative. We thus have that the $i$th diagonal element of $\tilde{Y}(t)\tilde{X}(t)$ at time $t \in [0,T]$ is given by
\begin{align}
    y(t)x(t) &= p_s(t)q_s(t)s_0x_0^2 \label{eq y(t)x(t) sse}
\end{align}
Since our closed-form expressions for $y(t)$ and $x(t)$ solve (\ref{eq decoupled scalar ODE y ii SSE the first one}--\ref{eq x_ii initial SSE}) on $[0,T]$ for any $T>0$, they define a single solution to the same ODE system for all $t \in [0,\infty)$. Thus, we can study the convergence of $y(t)x(t)$ by taking their limit as $t \to \infty$. Recall that, as $t \to \infty$,
\begin{align}
\sinh\!\left(\frac{\sqrt{\xi_1}}{2}t\right) &\sim \frac{1}{2} e^{\frac{\sqrt{\xi_1}}{2}t}, \\
\cosh\!\left(\frac{\sqrt{\xi_1}}{2}t\right) &\sim \frac{1}{2} e^{\frac{\sqrt{\xi_1}}{2}t},
\end{align}
and
\begin{align}
\sinh\!\left(\sqrt{\xi_1}(t+\xi_2)\right) &\sim \frac{1}{2} e^{\sqrt{\xi_1}(t+\xi_2)}
\end{align}
This gives us
\begin{align}
     \lim_{t \to \infty}\frac{\sinh\!\left(\frac{\sqrt{\xi_1}}{2}t\right)}{\sqrt{2x_0^2\sinh\!\left(\sqrt{\xi_1}(t+\xi_2)\right)+4s_0}} &=  \lim_{t \to \infty}\frac{\cosh\!\left(\frac{\sqrt{\xi_1}}{2}t\right)}{\sqrt{2x_0^2\sinh\!\left(\sqrt{\xi_1}(t+\xi_2)\right)+4s_0}} \\
    &=  \lim_{t \to \infty}\frac{\frac{1}{2}e^{\frac{\sqrt{\xi_1}}{2}t}}{\sqrt{2x_0^2\cdot \frac{e^{\sqrt{\xi_1}(t+\xi_2)}}{2}+4s_0}} \\
    &= \lim_{t \to \infty}\frac{e^{\frac{\sqrt{\xi_1}}{2}t}}{2\sqrt{x_0^2\cdot {e^{\sqrt{\xi_1}(t+\xi_2)}}+4s_0}} \cdot \frac{e^{-\frac{\sqrt{\xi_1}}{2}t}}{e^{-\frac{\sqrt{\xi_1}}{2}t}} \\
    &= \lim_{t \to \infty}\frac{1}{2\sqrt{x_0^2\cdot {e^{\xi_2\sqrt{\xi_1}}}+4s_0e^{-\sqrt{\xi_1}t}}} \\
    &= \frac{1}{2|x_0|e^{\frac{\xi_2\sqrt{\xi_1}}{2}}} \label{eq exponential SSE}
\end{align}
Note that
\begin{align}
    e^{\xi_2\sqrt{\xi_1}} &= e^{\arsinh\left( \frac{x_0^2}{2s_0}\right)} \\
    &= e^{\ln \left(  \frac{x_0^2}{2s_0}+\sqrt{1+\left( \frac{x_0^2}{2s_0}\right)^2} \right)} \\
    &= \frac{x_0^2}{2s_0}+\sqrt{1+\left( \frac{x_0^2}{2s_0}\right)^2} \\
    &= \frac{x_0^2+\sqrt{x_0^4+4s_0^2}}{2s_0}
\end{align}
So \eqref{eq exponential SSE} becomes
\begin{align}
    \frac{1}{2|x_0|} \cdot \sqrt{\frac{2s_0}{x_0^2+\sqrt{x_0^4+4s_0^2}}} &= \frac{\sqrt{s_0}}{\sqrt{2}|x_0|\sqrt{x_0^2+\sqrt{x_0^4+4s_0^2}}} 
\end{align}
We have
\begin{align}
    \lim_{t \to \infty}\frac{\sinh\!\left(\frac{\sqrt{\xi_1}}{2}t\right)}{\sqrt{2x_0^2\sinh\!\left(\sqrt{\xi_1}(t+\xi_2)\right)+4s_0}} &=  \lim_{t \to \infty}\frac{\cosh\!\left(\frac{\sqrt{\xi_1}}{2}t\right)}{\sqrt{2x_0^2\sinh\!\left(\sqrt{\xi_1}(t+\xi_2)\right)+4s_0}} \\
    &= \frac{\sqrt{s_0}}{\sqrt{2}|x_0|\sqrt{x_0^2+\sqrt{x_0^4+4s_0^2}}}
\end{align}
We can then quickly calculate
\begin{align}
    \lim_{t \to \infty} p_s(t) &= \lim_{t \to \infty}\frac{1}{\sqrt{2x_0^2\sinh\!\left(\sqrt{\xi_1}(t+\xi_2)\right)+4s_0}}
\Bigg[
\frac{x_0^2}{\sqrt{s_0}}
\sinh\!\left(\frac{\sqrt{\xi_1}}{2}t\right)
+
\frac{\sqrt{x_0^4+4s_0^2}}{\sqrt{s_0}}
\cosh\!\left(\frac{\sqrt{\xi_1}}{2}t\right)\Bigg] \\
&= \lim_{t \to \infty} \frac{x_0^2}{\sqrt{s_0}}\cdot\frac{\sinh\!\left(\frac{\sqrt{\xi_1}}{2}t\right)}{\sqrt{2x_0^2\sinh\!\left(\sqrt{\xi_1}(t+\xi_2)\right)+4s_0}} + \frac{\sqrt{x_0^4+4s_0^2}}{\sqrt{s_0}}\cdot\frac{\cosh\!\left(\frac{\sqrt{\xi_1}}{2}t\right)}{\sqrt{2x_0^2\sinh\!\left(\sqrt{\xi_1}(t+\xi_2)\right)+4s_0}} \\
&= \frac{x_0^2}{\sqrt{s_0}}\cdot\frac{\sqrt{s_0}}{\sqrt{2}|x_0|\sqrt{x_0^2+\sqrt{x_0^4+4s_0^2}}} + \frac{\sqrt{x_0^4+4s_0^2}}{\sqrt{s_0}}\cdot\frac{\sqrt{s_0}}{\sqrt{2}|x_0|\sqrt{x_0^2+\sqrt{x_0^4+4s_0^2}}} \\
&= \frac{x_0^2+\sqrt{x_0^4+4s_0^2}}{\sqrt{2}|x_0|\sqrt{x_0^2+\sqrt{x_0^4+4s_0^2}}} \\
&= \frac{\sqrt{x_0^2+\sqrt{x_0^4+4s_0^2}}}{\sqrt{2}|x_0|}
\end{align}
We also have
\begin{align}
    \lim_{t \to \infty} q_s(t) &= \lim_{t \to \infty}\frac{2}
{\sqrt{s_0}} \cdot  \frac{\sinh\!\left(\frac{\sqrt{\xi_1}}{2}t\right)}{\sqrt{2x_0^2\sinh\!\left(\sqrt{\xi_1}(t+\xi_2)\right)+4s_0}} \\
&= \frac{2}
{\sqrt{s_0}} \cdot \frac{\sqrt{s_0}}{\sqrt{2}|x_0|\sqrt{x_0^2+\sqrt{x_0^4+4s_0^2}}} \\
&=\frac{2}{\sqrt{2}|x_0|\sqrt{x_0^2+\sqrt{x_0^4+4s_0^2}}}
\end{align}
We finally have all we need to arrive at the convergence of the $i$th diagonal of $\tilde{Y}(t)\tilde{X}(t)$. Taking the limit of \eqref{eq y(t)x(t) sse} as $t \to \infty$, we find
\begin{align}
    \lim_{t \to \infty} y(t)x(t) &= \lim_{t \to \infty} p_s(t)q_s(t)s_0x_0^2 \\
    &= s_0x_0^2 \cdot \frac{\sqrt{x_0^2+\sqrt{x_0^4+4s_0^2}}}{\sqrt{2}|x_0|} \cdot \frac{2}{\sqrt{2}|x_0|\sqrt{x_0^2+\sqrt{x_0^4+4s_0^2}}} \\
    &= s_0
\end{align}
Return to indexed notation. We thus find that, for $i \in \{ 1, \ldots , r \}$, the convergence of the $i$th diagonal of $\tilde{Y}\tilde{X}$ is given by
\begin{align}
    \lim_{t \to \infty} \tilde{y}_{ii}(t)\tilde{x}_{ii}(t) &= s_{0,i} = \left( \Sigma_0 \right)_{ii} \label{eq convergence to singular values SSE}
\end{align}
where $s_{0,i}$ denotes the $i$th largest singular value of $W_0$. 

Consequently, under the spectral initialization scheme in (\ref{eq y_ii initial SSE}-\ref{eq x_ii initial SSE}),
the solution to the ODE system in (\ref{eq tilde Y SSE}--\ref{eq tilde X SSE})
satisfies
\begin{align}
    \lim_{t \to \infty} \tilde{Y}(t)\tilde{X}(t)
    &=
    \begin{bmatrix}
        \Sigma_{0,r} & \mathbf{0} \\
        \mathbf{0} & \mathbf{0}
    \end{bmatrix}, \label{eq converged tilde matrix}
\end{align}
where $\Sigma_{0,r} \in \mathbb{R}^{r \times r}$ denotes the diagonal matrix
whose entries are the top $r$ singular values of $W_0$ in non-increasing order. Recover the converged optimizer $\lim\limits_{t \to \infty}Y(t)X(t)$ for \eqref{eq sse objective} by recalling the transformations in (\ref{eq Y sse}--\ref{eq X sse}). Left multiplying \eqref{eq converged tilde matrix} by $U$ and right multiplying by $V^T$, we arrive at
\begin{align}
    \lim_{t \to \infty} U\tilde{Y}(t)\tilde{X}(t)V^T
    &=U
    \begin{bmatrix}
        \Sigma_{0,r} & \mathbf{0} \\
        \mathbf{0} & \mathbf{0}
    \end{bmatrix}V^T
\end{align}
or simply
\begin{align}
Y^*X^* :=
    \lim_{t \to \infty} Y(t)X(t)
    &=U
    \begin{bmatrix}
        \Sigma_{0,r} & \mathbf{0} \\
        \mathbf{0} & \mathbf{0}
    \end{bmatrix}V^T
\end{align}
We thus obtain the final rank-$r$ solution produced by the LoRA gradient flow for the problem in \eqref{eq sse objective}. We can calculate the final loss as
\begin{align}
    \frac{1}{2}\norm{W_0-Y^*X^*}^2 &= \frac{1}{2}\norm{U\Sigma_0V^T-U\begin{bmatrix}
        \Sigma_{0,r} & \mathbf{0} \\
        \mathbf{0} & \mathbf{0}
    \end{bmatrix}V^T}^2 \\
    &= \frac{1}{2}\norm{U\left(\Sigma_0 - \begin{bmatrix}
        \Sigma_{0,r} & \mathbf{0} \\
        \mathbf{0} & \mathbf{0}
    \end{bmatrix} \right)V^T}^2 \\
    &= \frac{1}{2}\norm{\Sigma_0 - \begin{bmatrix}
        \Sigma_{0,r} & \mathbf{0} \\
        \mathbf{0} & \mathbf{0}
    \end{bmatrix}}^2 \\
    &= \frac{1}{2}\sum\limits_{i = r+1}^{\min(n,m)} s_{0,i}^2 \label{eq final loss sse}
\end{align}
Note also that our LoRA rank-$r$ solution $Y^*X^*$ is equivalent to
\begin{align}
    Y^*X^* &= U_r \Sigma_{0,r} V_r^T \label{eq final optimizer sse}
\end{align}
where $U_r$ and $V_r$ (truncations of $U$ and $V$) are the matrices of left and right singular vectors for the top $r$ singular values of $W_0$. Our final loss and optimizer in
(\ref{eq final loss sse}--\ref{eq final optimizer sse})
coincide with the theoretical minimum loss and optimizer for
\eqref{eq sse objective} characterized by the
Eckart--Young--Mirsky theorem~\cite{EckartYoung1936}. We thus find via our gradient flow analysis that LoRA converges to the optimal rank-$r$ solution of \eqref{eq sse objective} under spectral initialization.


\onecolumn

\section{Proof of Lemma \ref{lem Lipschitz smoothness SSE}}\label{app Lipschitz smoothness SSE}

\begin{lemma}[Lipschitz Smoothness Squared Frobenius Objective]
    Define $h: \mathbb{R}^{n \times r} \times \mathbb{R}^{r \times m} \to \mathbb{R}$ via
    \begin{align}
        h(B,A) &:= \frac{1}{2}\norm{W_0-BA}^2
    \end{align}
    for constant $W_0 \in \mathbb{R}^{n \times m}$. Then the objective gradient $\nabla h$ is Lipschitz smooth in our training domain. Namely, there exists $L_{R'} > 0$ such that, for any $(B_1,A_1) , (B_2,A_2) \in \mathcal{D}_{R'}$, we have
    \begin{align}
        \norm{\nabla h(B_1,A_1) - \nabla h(B_2, A_2)} \leq L_{R'} \norm{ \Big( B_1, A_1 \Big) -  \Big( B_2, A_2 \Big) }
    \end{align}
    \textbf{Proof:} \\
    
    Let $(B_1,A_1) , (B_2,A_2) \in \mathcal{D}_{R'}$ be given. Refer to the partial gradients for $h(B,A)$ in (\ref{eq partial g A sse}--\ref{eq partial g B sse}) of Appendix \ref{app learning dynamics for SSE}. We have
    \begin{align}
        \norm{\nabla h(B_1,A_1) - \nabla h(B_2,A_2)} &= \norm{\Bigg( \nabla_B h(B_1,A_1) - \nabla_B h(B_2, A_2)  \, , \, \nabla_A h(B_1,A_1) - \nabla_A h(B_2, A_2) \Bigg)}
    \end{align}
    The partial gradients with respect to $B$ simplify to
    \begin{align}
       \nabla_B h(B_1,A_1) - \nabla_B h(B_2, A_2) &= -(W_0-B_1A_1)A_1^T + (W_0-B_2A_2)A_2^T \\
       &= (W_0-B_2A_2)A_2^T - (W_0-B_1A_1)A_2^T + (W_0-B_1A_1)A_2^T - (W_0-B_1A_1)A_1^T \\
       &= \Big( (W_0-B_2A_2) - (W_0-B_1A_1) \Big) A_2^T + \left(W_0-B_1A_1\right)\left(A_2^T-A_1^T\right) \\
       &= \left( B_1A_1-B_2A_2 \right)A_2^T + \left( W_0 - B_1A_1 \right)\left( A_2^T - A_1^T \right)
    \end{align}
    Taking the norm, we find
    \begin{align}
        \norm{\nabla_B h(B_1,A_1) - \nabla_B h(B_2, A_2)} &= \norm{\left( B_1A_1-B_2A_2 \right)A_2^T + \left( W_0 - B_1A_1 \right)\left( A_2^T - A_1^T \right)} \\
        &\leq \norm{\left( B_1A_1-B_2A_2 \right)A_2^T}+\norm{\left( W_0 - B_1A_1 \right)\left( A_2^T - A_1^T \right)} \\
        &\leq \norm{\left( B_1A_1-B_2A_2 \right)}\norm{A_2^T}+\norm{\left( W_0 - B_1A_1 \right)}\norm{\left( A_2^T - A_1^T \right)} \\
        &\leq \norm{\left( B_1A_1-B_2A_2 \right)}R'+\Big(\norm{ W_0 } + \norm{ B_1A_1 }\Big)\norm{\left( A_2^T - A_1^T \right)} \\
        &\leq \norm{\left( B_1A_1-B_2A_2 \right)}R'+\Big(\norm{ W_0 } + \norm{ B_1}\norm{A_1 }\Big)\norm{\left( A_2^T - A_1^T \right)} \\
        &\leq \norm{\left( B_1A_1-B_2A_2 \right)}R'+\Big(\norm{ W_0 } + R'^2\Big)\norm{\left( A_2^T - A_1^T \right)} \\
        &\leq \norm{\left( B_1A_1-B_1A_2+B_1A_2-B_2A_2 \right)}R'+\Big(\norm{ W_0 } + R'^2\Big)\norm{\left( A_2^T - A_1^T \right)} \\
        &\leq \norm{ B_1(A_1-A_2)+(B_1-B_2)A_2 }R'+\Big(\norm{ W_0 } + R'^2\Big)\norm{\left( A_2^T - A_1^T \right)} \\
        &\leq  \Big(\norm{ B_1(A_1-A_2)}+\norm{(B_1-B_2)A_2 }\Big)R'+\Big(\norm{ W_0 } + R'^2\Big)\norm{\left( A_2^T - A_1^T \right)} \\
        &\leq \Big(\norm{ B_1}\norm{(A_1-A_2)}+\norm{(B_1-B_2)}\norm{A_2 }\Big)R'+\Big(\norm{ W_0 } + R'^2\Big)\norm{\left( A_2^T - A_1^T \right)} \\
        &\leq \Big(R'\norm{(A_1-A_2)}+\norm{(B_1-B_2)}R'\Big)R'+\Big(\norm{ W_0 } + R'^2\Big)\norm{\left( A_2^T - A_1^T \right)} \\
        &= \Big(\norm{(A_1-A_2)}+\norm{(B_1-B_2)}\Big)R'^2+\Big(\norm{ W_0 } + R'^2\Big)\norm{\left( A_1-A_2 \right)} \\
        &= \norm{B_1-B_2}R'^2 + \Big(\norm{ W_0 } + 2R'^2\Big)\norm{\left( A_1-A_2 \right)} \\
        &\leq \norm{B_1-B_2}\Big(\norm{ W_0 } + 2R'^2\Big) + \Big(\norm{ W_0 } + 2R'^2\Big)\norm{ A_1-A_2 } \\
        &= \Big(\norm{ W_0 } + 2R'^2\Big)\Big( \norm{B_1-B_2} + \norm{A_1-A_2} \Big) \\
        &\leq \sqrt{2}\Big(\norm{ W_0 } + 2R'^2\Big)\Big( \norm{B_1-B_2}^2 + \norm{A_1-A_2}^2 \Big)^{1/2}
    \end{align}
    where the last line holds by Cauchy-Schwartz on $\mathbb{R}^2$. So we have
    \begin{align}
        \norm{\nabla_B h(B_1,A_1) - \nabla_B h(B_2, A_2)} &\leq \sqrt{2}\Big(\norm{ W_0 } + 2R'^2\Big)\Big( \norm{B_1-B_2}^2 + \norm{A_1-A_2}^2 \Big)^{1/2}
    \end{align}
    Denote $\dfrac{1}{\sqrt{2}}L_{R'} := \sqrt{2}\Big(\norm{ W_0 } + 2R'^2\Big)$. Squaring both sides of our inequality, we get
    \begin{align}
        \norm{\nabla_B h(B_1,A_1) - \nabla_B h(B_2, A_2)}^2 &\leq \frac{1}{2}L_{R'}^2\Big( \norm{B_1-B_2}^2 + \norm{A_1-A_2}^2 \Big) \label{eq nabla  h_B}
    \end{align}
    An identical calculation for $\norm{\nabla_A h(B_1,A_1) - \nabla_A h(B_2, A_2)}$ gives
    \begin{align}
        \norm{\nabla_A h(B_1,A_1) - \nabla_A h(B_2, A_2)}^2 &\leq \frac{1}{2}L_{R'}^2\Big( \norm{B_1-B_2}^2 + \norm{A_1-A_2}^2 \Big) \label{eq nabla h_A}
    \end{align}
    Add together \eqref{eq nabla  h_B} and \eqref{eq nabla h_A} to find
    \begin{align}
        \norm{\nabla_B h(B_1,A_1) - \nabla_B h(B_2, A_2)}^2+\norm{\nabla_A h(B_1,A_1) - \nabla_A h(B_2, A_2)}^2 &\leq L_{R'}^2\Big( \norm{B_1-B_2}^2 + \norm{A_1-A_2}^2 \Big)
    \end{align}
    or, equivalently,
    \begin{align}
        \norm{\nabla h(B_1,A_1) - \nabla h(B_2,A_2)}^2 &\leq L_{R'}^2\norm{\left(B_1,A_1\right)-\left(B_2,A_2\right)}^2
    \end{align}
    Taking the square root, we finally arrive at
    \begin{align}
        \norm{\nabla h(B_1,A_1) - \nabla h(B_2,A_2)} &\leq L_{R'}\norm{\left(B_1,A_1\right)-\left(B_2,A_2\right)}
    \end{align}
    Thus, $h$ is $L_{R'}$ smooth on $\mathcal{D}_{R'}$.
\end{lemma}


\end{document}